\documentclass[11pt]{article}

\usepackage{fullpage}

\usepackage{latexsym}

\usepackage{dsfont}
\usepackage{graphicx}
\usepackage{amsfonts}
\usepackage{amsmath,amssymb}
\usepackage{amsthm}
\usepackage{subfig}\DeclareCaptionType{copyrightbox}
\usepackage{xcolor}

\usepackage{rotating}

\usepackage{verbatim}
\usepackage{multirow}
\usepackage[ruled,linesnumbered]{algorithm2e}
\usepackage{tabularx}
\usepackage{placeins}
\newtheorem{defi}{Definition}
\newtheorem{sat}{Theorem}
\newtheorem{lem}{Lemma}
\newtheorem{ass}{Assumption}
\newtheorem{exam}{Example}
\newtheorem{expe}{Experiment}

\newcommand{\E}{{\mathrm E}}
\renewcommand{\P}{{\mathrm P}}
\newcommand{\argmin}{\operatornamewithlimits{argmin}}

\DeclareMathAlphabet{\mathbb}{U}{bbold}{m}{n}

\newcommand{\1}{\mathds{1}}
\newcommand{\N}{\mathbb{N}}
\newcommand{\R}{\mathbb{R}}

\newcommand{\MID}{\hspace{1mm}\vert\hspace{1mm}}

\newcommand{\MIDR}{\hspace{1mm}\right\vert\hspace{1mm}}
\newcommand{\ASSIGN}{:=}
\newcommand{\VAR}{\sigma^2}
\newcommand{\COV}{cov}

\newcommand{\Msix}{M^6}

\graphicspath{{./svgs/}}

\begin{document}
\clubpenalty=10000
\widowpenalty = 10000

\title{Explanation of Stagnation at Points that are not Local Optima in Particle Swarm Optimization by Potential Analysis}

\author{
Alexander Rass
\qquad
Manuel Schmitt
\qquad
Rolf Wanka\\[2mm]
Department of Computer Science\\
University of Erlangen-Nuremberg, Germany\\
{\{alexander.rass, manuel.schmitt, rolf.wanka\}@fau.de}
}

\maketitle

\section*{Abstract}
Particle Swarm Optimization (PSO) is a nature-inspired
meta-heuristic for solving continuous optimization problems.
In \cite{SWa:13,SWc:15}, the potential of the particles of a swarm has been
used to show that slightly modified PSO guarantees convergence to local optima.
Here we show that under specific circumstances the unmodified PSO,
even with swarm parameters known (from the literature) to be ``good'',
almost surely does not yield convergence to a local optimum is provided.
This undesirable phenomenon is called stagnation.
For this purpose, the particles' potential in each dimension is analyzed mathematically.
Additionally, some reasonable assumptions on the behavior of the particles' potential are made.
Depending on the objective function and, interestingly, the number of particles,
the potential in some dimensions may decrease much faster than in other dimensions.
Therefore, these dimensions lose relevance,
i.\,e., the contribution of their entries to the decisions about attractor updates
becomes insignificant and, with positive probability, they never regain relevance.
If Brownian Motion is assumed to be an approximation of the
time-dependent drop of potential, practical, i.\,e., large
values for this probability are calculated.
Finally, on chosen multidimensional polynomials of degree two, experiments are provided
showing that the required circumstances occur quite frequently.
Furthermore, experiments are provided showing that
even when the very simple sphere function is processed
the described stagnation phenomenon occurs.
Consequently, unmodified PSO does not converge to any local optimum of the chosen functions for tested parameter settings.


\newpage
\tableofcontents
\newpage
\section{Introduction}

Particle swarm optimization (PSO), introduced by Ken\-ne\-dy and 
Eberhart~\cite{ken_eb_1995,eb_ken_1995}, is a very popular nature-inspired 
meta-heuristic for solving continuous optimization problems.
Fields of very successful application are, among many others,
Bio\-medical Image Processing~\cite{WSZZE:04},
Geosciences~\cite{OD:10},
and Materials Science~\cite{RPPN:09}, where the
continuous objective function on a multi-dimensional domain
is not given in a closed form, but by a ``black box''.
The popularity of the PSO framework 
is due to the fact that on the one hand it can be realized and,
if necessary, adapted to further needs easily,
but on the other hand shows in experiments good performance results with respect
to the quality of the obtained solution and the speed needed to obtain it.
A thorough discussion of PSO can be found in~\cite{swarmhandbook:11}.

To be precise, let an objective function $f:\mathbb{R}^D\rightarrow \mathbb{R}$
on a $D$-dimensional domain be given
that (w.\,l.\,o.\,g.) has to be minimized.
A population of \emph{particles}, each consisting of a position (the candidate for
a solution), a velocity
and a local attractor, moves through the search space $\mathbb{R}^D$. The
local attractor of a particle is the best position with respect to $f$
this particle has encountered so far.
The best of all local attractors is the global attractor.
The movement of a particle is governed 
by so-called movement equations that depend on both the particle's
velocity and its two attractors and on some additional fixed algorithm parameters.
The pseudo code of the PSO approach is visualized in Algorithm \ref{alg:classicalPSO}.
Additionally, Definition \ref{def:classicalPSO} captures the PSO behavior mathematically as a \emph{stochastic process}.
The population in motion is called the \emph{swarm}.

There are guidelines known for the ``good'' choice of the fixed parameters
that control the impact of the current velocity and the attractors
on the updated velocity of a particle (\cite{T:03,JLY:07a}\nocite{JLY:07})
such that the swarm provably converges to a particular point
in the search space (under some reasonable assumptions).
However, the point of convergence is not necessarily a global optimum.
Local optima might also be considered acceptable,
but unfortunately it is possible that the point of convergence
is not even a local optimum.
In the latter case, one says that the swarm \emph{stagnates}.
Examples are presented in Sec.~\ref{sec:experiments}.
E.\,g., with established \emph{good} parameter settings,
stagnation can be observed
if $3$ particles work on the $10$-dimensional,
sphere function.
In~\cite{LW:11}, Lehre and Witt show that there are non-trivial bad parameter
settings and initial configurations such that PSO possibly converges at arbitrary points when processing 
the one-dimensional sphere function. However, their result is presented for 
populations of $1$ and $2$ particles only.
Additionally, their parameter settings considerably deviate from 
those generally considered as good (\cite{Carlisle01,CK:02,T:03}).

In \cite{SWa:13, SWc:15}, the notion of the potential of the
swarm has been introduced.
This potential has been used to prove that in the one-dimensional
search space the swarm almost surely (in the mathematical sense)
finds a local optimum.
If $D\ge2$, the movement equation has been adapted slightly
to avoid that the swarm's potential drops too close to $0$, and
hence avoiding stagnation at a non-optimal position.
Consequently, this version of PSO almost surely finds local optima.

A comprehesive overview on theoretical results concerning PSO can be found in \cite{S15}.

The phenomenon of convergence to a point that is not a local
optimum, has to the best of our knowledge not yet been formally
investigated in a setting which is not generally restricted to a number of particles or specific parameters for the PSO.
In this paper, reasons for the phenomenon
of stagnation are given.
For it, the notion of the potential is modified
such that it is now aware of how the
particles experience the objective function $f$.
A theoretical model is provided
which on the one hand uses reasonable assumptions and
on the other hand provides a basis to mathematically prove that
the swarm almost surely stagnates.
The model captures the observation that during the execution of PSO
on objective functions $f$ the contribution of some dimensions to the potential
decreases exponentially faster than the contribution of
other dimensions.
This does not generally imply that PSO faces the problem of stagnation,
but if the model is applicable,
which is often the case when few particles are used,
then this model supplies indications whether in a specific setting stagnation is present or not.
The assumptions which need to be applied for the model will be justified in the experimental Section \ref{sec:experiments}.

The first model is used to prove the statement, that finally the swarm stagnates indefinitely almost surely.
In that model we can prove that the conditional probability that stagnation remains indefinitely if stagnation emerges at some time $T$ is positive.
If the time-dependent drop of potential is additionally approximated
by a Brownian Motion with drift, then the described probability can be calculated by an explicit formula.
The values obtained from that formula coincide very well with empirically measured probabilities in experiments.
If stagnation remains indefinitely mainly poor solutions will be returned by PSO.

Experimentally it is shown that the described
separation of potential indeed occurs when PSO is run with
certain numbers of particles on some popular 
functions from the CEC benchmark set~\cite{benchmarkset} and an additional function.

This paper is organized as follows: 
After the formal definition of the PSO process in
Section~\ref{sec:definitions},
in Section~\ref{sec:theory}
the new notion of potential, results of experiments and the proof that
under realistic assumptions PSO almost surely stagnates, are presented.
In Section~\ref{sec:experiments},
experimental analysis
of three benchmark functions from the CEC benchmark set~\cite{benchmarkset} and an additional function is presented.

\newpage
\section{Definitions}
\label{sec:definitions}
First the model which is going to be used for our analysis of the PSO algorithm is presented.
Algorithm \ref{alg:classicalPSO} represents the pseudo code of the classical PSO algorithm.
No bound handling strategies are investigated, because they have almost no influence on the convergence if the swarm is converging
to a point not on the boundaries.
Similar to \cite{SWa:13, SWc:15}, the model describes the positions of the particles,
the velocities and the global and local attractors also as real-valued stochastic processes. 
Basic mathematical tools from probability theory,
which are needed for this analysis can be found in,
e.\,g.,~\cite{durrett2010probability}.
\begin{defi}[Classical PSO process]\label{def:classicalPSO}
A \emph{swarm} of $N$
 particles moves through the $D$-dimensional search space
$\mathbb{R}^D$. Let $f:\mathbb{R}^D\rightarrow \mathbb{R}$ be the objective function.
At each time $t\in\N$, each particle $n$ has a position $X_t^n\in\R^D$,
a velocity $V_t^n\in\R^D$ and a local attractor $L_t^n\in\R^D$,
the best position particle $n$ has visited until time $t$.
Additionally, the swarm shares its global attractor $G_t^n\in\R^D$,
storing the best position \emph{any} particle has visited until the $t$'th step of particle $n$,
i.\,e., since each particle's update of the global attractor is immediately visible for the next particle,
it can have several values in the same iteration.
When necessary, $X_t^{n,d}$ is written for the $d$'th component of $X_t^n$ ($V_t^{n,d}$ etc. analogously).
With 
a given distribution for $(X_0, V_0)$, $(X_{t+1},V_{t+1},L_{t+1},G_{t+1})$ is determined
by the following recursive equations that are called the \emph{movement equations}:
\begin{align*}
L_{0}^{n} :=& X_0^n\text{ for }1\le n \le N,\\
G_t^n     :=& \argmin_{x\in\bigcup_{i=1}^{n-1}\lbrace L_{t+1}^i\rbrace\cup\bigcup_{i=n}^N\lbrace L_{t}^i\rbrace}f(x)\text{ for }t\ge 0\text{, }1\le n\le N,\\
V_{t+1}^n :=& \chi\cdot V_t^n + c_1\cdot r_t^n \odot (L_t^n-X_t^n)
			+ c_2\cdot s_t^n \odot (G_{t}^n-X_t^n)\text{ for }t\ge 0\text{, }1\le n\le N,\\
X_{t+1}^n :=& X_t^n+V_{t+1}^n\text{ for }t \ge 0\text{, }1\le n \le N,\\
L_{t+1}^n :=& \argmin_{x\in\lbrace X_{t+1}^n, L_t^n\rbrace}f(x)\text{ for }t \ge 0\text{, }1\le n \le N.
\end{align*}
Here, $\chi$, $c_1$ and $c_2$ are some positive constants called the
\emph{fixed parameters}
of the swarm, 
and $r_t^n$, $s_t^n$ are uniformly distributed over $[0,1]^D$ and all independent.
$\odot$ is meant as the item-wise product of two vectors.
The result is then a vector as well.

The underlying probability space is called $(\Omega,\mathcal{A},P)$.
The $\sigma$-Algebra $\mathcal{A}$ is equal to the generated $\sigma$-Algebra
of the family of $\sigma$-Algebras $\lbrace\mathcal{A}_t\rbrace_{t\in\N}$,
to which $(X_{t},V_{t},L_{t})_{t\in\N}$ is adapted. 
\end{defi}
One can think of $\mathcal{A}_t$ as the mathematical object carrying the information known at time $t$.
Mainly $\mathcal{A}_t$ refers to the product $\sigma$-algebra for the random variables $r_t^n$ and $s_t^n$,
but we will introduce further random variables,
which will be measurable by some $\mathcal{A}_t$ or $\mathcal{A}$ as well,
to expand our model.
If after the $t$'th step the process is stopped, the \emph{solution} is $G_t^1$.
$G_t^1$ is $\mathcal{A}_t$-measurable because it is the $\argmin$ of the local attractors $L_t^1,\ldots,L_t^N$,
whereas $G_t$ is not $\mathcal{A}_t$-measurable as it partially depends on $L_{t+1}$.
In \cite{bauer1996probability}, one can find a comprehensive introduction into $\sigma$-algebras and filtrations.
Informally, $\mathcal{A}_t$ captures the information of the underlying stochastic process,
that is available at time $t$,
i.\,e., a random variable is $\mathcal{A}_t$-measurable,
iff its value is determined at time $t$ and for a random variable $V$, $\E[V \mid \mathcal{A}_t]$ is the expectation of $V$ under taking the information of every time $t'\le t$ into account.
$\E[V \mid \mathcal{A}_t]$ is an $\mathcal{A}_t$-measurable random variable,
e.\,g., if $V$ is $\mathcal{A}_t$-measurable then $\E[V \mid \mathcal{A}_t]=V$.

The update process of a particle is visualized in Figure \ref{figure:particleUpdate}.
\begin{figure}[htbp]
\centering
{\includegraphics[width=0.5\textwidth]{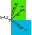}}
\caption{Update process of a particle with $\chi=0.72984$ and $c_1=c_2=1.496172$: Position $X$, velocity $V$, local attractor $L$ and global attractor $G$}
\label{figure:particleUpdate}
\end{figure}
As specified in the movement equations of Definition \ref{def:classicalPSO}, the new velocity $V_{t+1}^n$ of the particle $n$ consists of
a fraction of the old particle without any randomness $\chi\cdot V_{t}^n$,
a randomized part of the vector from the current position to the local attractor $c_1\cdot r_t^n\odot(L_{t}^n-X_t^n)$
and a randomized part of the vector from the current position to the global attractor $c_2\cdot r_t^n\odot(G_{t}^n-X_t^n)$.
With PSO parameters $\chi=0.72984$ and $c_1=c_2=1.496172$,
as proposed in \cite{CK:02}, 
the terminal point of the randomized vector from the current position to the local attractor is sampled uniformly in the blue area of Figure \ref{figure:particleUpdate}
and the terminal point of the randomized vector from the current position to the global attractor is sampled uniformly in the green area of Figure \ref{figure:particleUpdate}.
A possible assignment for the three vectors is visualized by the dashed vectors.
The accumulation of the three parts yields the new velocity $V_{t+1}^n$ and the new position $X_{t+1}^n$ can be reached by moving $X_{t}^n$ by $V_{t+1}^n$.
After the update of the velocity and the position, the local and global attractor could be updated.
If also this step is executed, then it can be proceeded with the next particle.
Algorithm \ref{alg:classicalPSO} determines the respective approach with pseudo code.
\begin{algorithm}[htbp]
\caption{Classical PSO}
\label{alg:classicalPSO}

\SetKwInOut{Input}{input}
\SetKwInOut{InOut}{in/out}
\SetKwInOut{Output}{output}
\SetKwInOut{Data}{data}

\SetKw{KwTo}{to}%

\Input{objective function $f:\R^D\rightarrow\R$}
\Output{an optimized position $G\in\R^D$ for the objective function $f$}
\BlankLine
\tcc{initialize all $N$ particles:}
\tcc{the positions $X\in(\R^D)^N$ and the velocities $V\in(\R^D)^N$}
$(X,V) \ASSIGN$ getInitialPositionsAndVelocities()\;
\tcc{the local attractors $L\in(\R^D)^N$}
$L\ASSIGN X$\;
\tcc{initialize the global attractor $G\in\R^D$}
$G\ASSIGN \argmin_{x\in\lbrace L[1],\ldots,L[N]\rbrace}f(x)$\;
\BlankLine
\While{termination criterion not fulfilled}
{
        \For{$n \ASSIGN 1$ \KwTo $N$}
        {
                \For{$d \ASSIGN 1$ \KwTo $D$}
                {
                        \tcc{update $d$'th velocity coordinate of $n$'th particle}
                        \tcc{rand($a,b$) supplies a uniform random value in $[a,b]$}
                        $V[n][d]\ASSIGN\chi\cdot V[n][d]+$\\
			$+c_1\cdot\text{rand($0.0,1.0$)} \cdot(L[n][d]-X[n][d])$\\
                        $+c_2\cdot\text{rand($0.0,1.0$)} \cdot(G[d]-X[n][d])$
                }
                \tcc{update position of $n$'th particle}
                $X[n]\ASSIGN X[n] + V[n]$\;
                \tcc{update local attractor of $n$'th particle}
                \If{$f(X[n]) < f(L[n])$}
                {
                        $L[n]\ASSIGN X[n]$\;
                }
                \tcc{update global attractor}
                \If{$f(X[n]) < f(G)$}
                {
                        $G\ASSIGN X[n]$\;
                }
        }
}
\BlankLine
\Return $G$ \;
\end{algorithm}

\newpage
\section{Stagnation Analysis of PSO}
\label{sec:theory}
In this chapter some assumptions are made, such that theoretical analysis can be done.
With these assumptions theoretical proofs are provided,
which state that PSO does not reach local optima almost surely if the number of particles is too small.
\subsection{Theoretical Tools}
First of all some theoretical tools are introduced to make later proofs more manageable.
The following technical, well known lemma, which will be used in the proof of Lemma \ref{lem:clt},
gives sufficient conditions for an infinite product of probabilities to have a positive value.
\begin{lem}\label{lem:prod}
If $0\le a_i<1$ for every $i\in\mathbb{N}$ and
$\sum_{i=1}^\infty a_i<\infty$, then
$$
\smash[t]{\prod_{i=1}^\infty (1-a_i)>0}
\enspace.
$$
\end{lem}
\begin{proof}
$\prod_{i=1}^{\infty} (1-a_i)>0$ iff
$\exists n_0\in\N:\allowbreak\prod_{i=n_0}^{\infty} (1-a_i)>0$. 
Choose $n_0$ sufficiently large such that
$\sum_{i=n_0}^{\infty} a_i<1$. It follows that
$\prod_{i=n_0}^{\infty} (1-a_i)\ge 1-\sum_{i=n_0}^{\infty} a_i>0$.
\end{proof}
In the following, 
the Markov's inequality is applied to prove that 
any finite sum of independent, identically distributed random variables,
with negative expectation and finite first six moments each,
stays below zero with probability $>0$.
\begin{lem}
\label{lem:clt}
Let $(I_t)_{t\in\N}$ be independent identically distributed (i.i.d.) random variables with expected value $\mu^*$, let $\mu<0$, $M>0$ and $p_0>0$ be constants.
If $\mu^*\le \mu$, $\E[(I_t-\mu^*)^i]\le M$ for all $i\in\lbrace 1,\ldots,6\rbrace$ and $\P(I_t\le0)\ge p_0$, then:
\begin{itemize}
\item $\P(\sum_{\tilde t=0}^{t-1} I_{\tilde t}> 0)\le\frac{C}{t^{3}}$ for a constant $C>0$, such that $C:=C(\mu,M)$, and for all $t>0$,
\item $\P(\forall t\in\N:\sum_{\tilde t=0}^{t-1} I_{\tilde t}\le 0)\ge p$ for a constant $p>0$, such that $p:=p(\mu,M,p_0)$.
\end{itemize}
\end{lem}
\begin{proof}
As $\E[I_0-\mu^*]=0$ and $I_{\tilde t}$ are i.i.d., we get
\begin{align*}
\E\Big[\Big(\sum_{\tilde t=0}^{t-1}(I_{\tilde t}-\mu^*)\Big)^6\Big]=
&t\cdot \E[(I_0-\mu^*)^6]+10t(t-1)\cdot \E[(I_0-\mu^*)^3]^2 
\\&
+15 t(t-1)\E[(I_0-\mu^*)^4]\E[(I_0-\mu^*)^2]
+15 t(t-1)(t-2)\E[(I_0-\mu^*)^2]^3
\\
\le& \tilde C\cdot t^3,
\end{align*}
where
\begin{align*}
\tilde C:=&\tilde C(M):=M+10\cdot M^2+15\cdot M^2+15\cdot M^3\\
\ge& \E[(I_0-\mu^*)^6]+10\E[(I_0-\mu^*)^3]^2
+15 \E[(I_0-\mu^*)^4]\E[(I_0-\mu^*)^2] +15\E[(I_0-\mu^*)^2]^3,
\end{align*}
because every product which contains $E[(I_{\tilde t}-\mu^*)^1]$ as a factor is zero
and therefore only the exponent six, the pair of exponents two and four,
the pair of exponents three and three and the tuple of exponents two, two and two remain.
With the Markov's inequality on the sixth power of the centralized sum, which is a positive random variable,we get
\begin{equation*}
\P\Big(\sum_{\tilde t=0}^{t-1} I_{\tilde t}> 0\Big)\le\P\Big(\sum_{\tilde t=0}^{t-1}(I_{\tilde t}-\mu^*)\ge -\mu^* t\Big)
\end{equation*}
\begin{equation*}
\overset{\text{Markov}}{\le}\frac{\E[(\sum_{\tilde t=0}^{t-1}(I_{\tilde t}-\mu^*))^6]}{(\mu^* t)^6}\le\frac{\tilde C\cdot t^3}{(\mu^* t)^6}
\overset{\vert\mu^*\vert\ge\vert\mu\vert}{\le}\frac{\tilde C}{\mu^6t^3}\le \frac{C}{t^3}
\end{equation*}
where $C:=C(\mu,M):={\tilde C(M)}/{\mu^6}$.
Those inequalities prove the first part of Lemma \ref{lem:clt}.
To complete the proof further inequalities are needed, which will be received with conditional probabilities.
Let $\tilde t<\tilde T$ be a non negative integer.
\begin{eqnarray*}
&&   \P\Big(\sum_{t=0}^{\tilde T}         I_t\le 0\Bigm\vert \sum_{t=0}^{\tilde t}I_t\le0\wedge \forall t'<\tilde t:\sum_{t=0}^{ t'}I_t\le0\Big)\\
&\ge&\P\Big(\sum_{t=\tilde t+1}^{\tilde T}I_t\le 0\Bigm\vert \sum_{t=0}^{\tilde t}I_t\le0\wedge \forall t'<\tilde t:\sum_{t=0}^{ t'}I_t\le0\Big)\\
&=&  \P\Big(\sum_{t=\tilde t+1}^{\tilde T}I_t\le 0\Bigm\vert \sum_{t=0}^{\tilde t}I_t>0  \wedge \forall t'<\tilde t:\sum_{t=0}^{ t'}I_t\le0\Big)\\
&\ge&\P\Big(\sum_{t=0}^{\tilde T}         I_t\le 0\Bigm\vert \sum_{t=0}^{\tilde t}I_t>0  \wedge \forall t'<\tilde t:\sum_{t=0}^{ t'}I_t\le0\Big).
\end{eqnarray*}
As
\begin{equation*}
\P\Big(\sum_{t=0}^{\tilde T}I_t\le 0\Bigm\vert\forall t'<\tilde t:\sum_{t=0}^{t'}I_t\le0\Big)
\end{equation*}
is a convex combination of the first and the last term of the recent series of inequalities, it is less or equal to the greater value
\begin{equation*}
\P\Big(\sum_{t=0}^{\tilde T}I_t\le 0\Bigm\vert \sum_{t=0}^{\tilde t}I_t\le0\wedge \forall t'<\tilde t:\sum_{t=0}^{ t'}I_t\le0\Big)
\end{equation*}
and by induction
\begin{equation*}
\P\Big(\sum_{t=0}^{\tilde T}I_t\le0\Bigm\vert\forall t'<\tilde T:\sum_{t=0}^{t'}I_t\le0\Big)\ge\P\Big(\sum_{t=0}^{\tilde T}I_t\le0\Big).
\end{equation*}
Summing up all these probabilities leads to a sum, which is bounded:
\begin{equation}
\label{lem:clt:infiniteSum}
\sum\limits_{t=1}^\infty \P\Big(\sum\limits_{\tilde t=0}^{t-1}I_{\tilde t}> 0\Big)\le
\sum\limits_{t=1}^{\infty}\frac{C}{t^3}<\infty
\end{equation}
Since $\P(I_t\le 0)\ge p_0>0$, it follows that
$$
\P\Big(\forall t\le T:\sum_{\tilde t=0}^{t-1} I_{\tilde t}\le 0\Big)\ge \P(\forall t\le T:I_t\le 0)\ge p_0^T>0.
$$
Therefore the probability to reach no positive value is positive for fixed finite times.
As mentioned in equation \ref{lem:clt:infiniteSum} the infinite sum with the probabilities,
which specify the probability to remain non-positive at a specific time,
is bounded and therefore Lemma \ref{lem:prod} is applicable.\\
Altogether we have
\begin{eqnarray*}
\lefteqn{\P\Big(\forall t\in\N:\sum\limits_{\tilde t=0}^{t-1}I_{\tilde t}\le 0\Big)=
\lim\limits_{t \rightarrow \infty}\P\Big(\forall t' \le t:\sum\limits_{\tilde t=0}^{t'-1}I_{\tilde t}\le 0 \Big)}\allowdisplaybreaks\\
&=&\P(I_0\le 0)\prod\limits_{\tilde t=2}^\infty \P\Big(\sum\limits_{t=0}^{\tilde t-1}I_{t}\le 0
\Bigm\vert \forall t'<\tilde t:
\sum\limits_{t=0}^{ t'-1}I_{t}\le 0 \Big)\allowdisplaybreaks\\
&\ge&\prod\limits_{\tilde t=1}^\infty \P\Big(\sum\limits_{t=0}^{\tilde t-1}I_{t}\le 0\Big)\\
&\ge&\underbrace{\prod\limits_{\tilde t=1}^\infty \Big(1-\min\Big(1-p_0^{\tilde t}, \frac{C}{{\tilde t}^3}\Big)\Big)}_{>0\text{ with Lemma \ref{lem:prod}}}=:p>0.
\end{eqnarray*}
Obviously $p$ depends only on $\mu$, $M$ and $p_0$ because in the last expression only $C:=C(\mu,M)$ and $p_0$ appear.
\end{proof}
The next lemma provides some estimation for the moments of two random variables if a bound for the sixth moment of their sum is provided.
\begin{lem}\label{lem:boundedMoments}
Let $A$ and $B$ be independent random variables with finite first six moments and expectation zero.
There exists a fixed function $h:\R\rightarrow\R$ such that for each $i\in\lbrace 1,\ldots,6\rbrace$
\begin{itemize}
\item $\vert \E[(A+B)^i]\vert \le h(M)$
\item $\vert \E[A^i]\vert \le h(M)$
\item $\vert \E[B^i]\vert \le h(M)$
\end{itemize}
if $\E\left[(A+B)^6\right]\le M$.
\end{lem}

\begin{proof}
For each $i\in\lbrace 1,\ldots,6\rbrace$ we define fixed functions $h_{1,i}:\R\rightarrow\R$, and $h_{2,i}:\R\rightarrow\R$ such that
\begin{itemize}
\item $\vert \E[(A+B)^i]\vert \le h_{1,i}(M)$,
\item $\vert \E[A^i]\vert \le h_{2,i}(M)$ and
\item $\vert \E[B^i]\vert \le h_{2,i}(M)$,
\end{itemize}
if $\E\left[(A+B)^6\right]\le M$.

All expected values exist as specified.
Only the inequalities without absolute value need to be shown because all evaluations can also be made with $-A$ instead of $A$ and $-B$ instead of $B$, which guarantees the inequality for odd $i$.
For even $i$ the absolute value has no effect on the equation because the value inside the absolute value is already non-negative.
As random variables $A$ and $B$ are independent,
the expectation of a product can easily be separated.
For $i=1$ all expected values are zero and therefore the following functions are possible:
$$
h_{1,1}(M):=h_{2,1}(M):=0.
$$
$h_{1,i}(M):=M+1$ for all $i\in\lbrace 2,\ldots,6\rbrace$, because
\begin{align}
\E[(A+B)^i]\le& \E[\max(1,(A+B))^i)]\nonumber\\
\overset{i\le 6}{\le}&
\label{equ:lem:boundedMomentsHelper}
\E[\max(1,(A+B)^6)]\\
\le& \E[1 +(A+B)^6)]\nonumber\\
\le& 1+M.\nonumber
\end{align}
\begin{align*}
E[(A+B)^4]=&
\E[A^4]+4\E[A^3]\underbrace{\E[B^1]}_{=0}
+\underbrace{6\E[A^2]\E[B^2]}_{\ge 0}
+4\underbrace{\E[A^1]}_{=0}\E[B^3]+\E[B^4],
\end{align*}
which leads to
$$
\underbrace{\E[A^4]}_{\ge 0}+\underbrace{E[B^4]}_{\ge 0}\le \E[(A+B)^4]\le h_{1,4}(M)
$$
and therefore
$$
h_{2,4}(M):=h_{1,4}(M)=M+1
$$
is possible.
Similar to inequalities \ref{equ:lem:boundedMomentsHelper}
$$
h_{2,i}(M):=h_{2,4}(M)+1=M+2
$$
is received for all $i\in\lbrace 1,2,3\rbrace$.
Similar to the case $i=4$ we receive
\begin{align*}
\underbrace{\E[(A+B)^6]}_{\le h_{1,6}(M)}=&\E[A^6]+\binom{6}{2}\underbrace{\E[A^4]\E[B^2]}_{\ge 0}
+\binom{6}{3}\underbrace{\E[A^3]\E[B^3]}_{\ge -h_{2,3}(M)^2}
+
\binom{6}{4}\underbrace{\E[A^2]\E[B^4]}_{\ge 0}+\E[B^6]
\end{align*}
$$
\Rightarrow \underbrace{\E[A^6]}_{\ge 0}+\underbrace{\E[B^6]}_{\ge 0}\le h_{1,6}(M)+\binom{6}{3}h_{2,3}(M)^2=:h_{2,6}(M)
$$
and finally
$h_{2,5}(M):=1+h_{2,6}(M)$.

For each $M$ the maximal function value of all functions $h_{i,j}$ is a suitable choice for the function value of $h$.
$$
h(M):=\max_{i\in\lbrace 1,2\rbrace,\,j\in\lbrace 1,\ldots,6\rbrace}h_{i,j}(M)
$$
\end{proof}
\subsection{Potential and Stagnation Phases}
Now the idea of a stagnation measure is introduced, which is a multidimensional extension to the potential
used in \cite{SWa:13, SWb:13, SWc:15}.
For every step, a $D$-dimensional vector of potentials is evaluated -- one potential value for each dimension.
It is intended that the greater the value of such a potential for a single dimension is,
the greater is the impact of this dimension on the behavior of the swarm,
i.\,e., the greater is the portion of the change in the function value,
which is due to the movement in that dimension.
This property is not declared in the definition because it cannot be quantified in a strict way.
It will be specified in detail in the Assumptions
\ref{ass:insignificance1} and 
\ref{ass:insignificance2}.
Furthermore, the logarithmic potential is defined, which compares the impact of a specific dimension with the maximal impact along all dimensions.
The dimension which has currently the highest impact on the swarm has a logarithmic potential of zero and
all other dimensions will have no larger logarithmic potential.
Since the convergence analysis in \cite{JLY:07} implies that the general movement of a converging
particle swarm drops exponentially, a logarithmic scale is used and linear decrease is expected.
\begin{defi}[Potential]\label{def:potential}
Let $\tilde \Phi:\R^{3\cdot N\cdot D}\rightarrow\R^D$ be a measurable function.
Let $\triangle t$ be a positive integer constant, which will be called the \emph{step width}.
$\Phi(t,d)$ is defined as
$$
\Phi(t,d):=(\tilde\Phi(X_{t\cdot\triangle t},V_{t\cdot\triangle t},L_{t\cdot\triangle t}))_d.
$$
$\Phi:\N_0\times\lbrace 1,\ldots,D\rbrace\rightarrow (\Omega\rightarrow\R)$ is a function
which evaluates to a random variable for each pair in $\N_0\times\lbrace 1,\ldots,D\rbrace$.
$\Phi$ is called a \emph{potential} if $\Phi(t,d)$ is positive almost surely for all $t$ and $d$.
Additionally
\begin{align*}
\Psi(t,d):=&
\log\Big({\Phi(t,d)}\big/{\max\limits_{\tilde d\in\lbrace 1,\ldots,D\rbrace}\Phi(t,\tilde d)}\Big)
=
\log\Big(\Phi(t,d)\Big)-\log\Big(\max\limits_{\tilde d\in\lbrace 1,\ldots,D\rbrace}\Phi(t,\tilde d)\Big)
\end{align*}
is called a \emph{logarithmic potential} and
$(I_{t, d})_{t\in\N}$,
with $I_{t, d}:=\Psi(t+1, d)-\Psi(t, d)$ are called \emph{increments of dimension $d$}.
$\Phi(t,d)$, $\Psi(t,d)$ and $I_{t-1,d}$ are $\mathcal{A}_{t\cdot\triangle t}$-measurable random variables,
where $\mathcal{A}_t$ is the $\sigma$-algebra, which is specified in Definition \ref{def:classicalPSO}.
\end{defi}
In this paper the expression $\log$ is associated with the logarithm of base $2$.
The step width $\triangle t$ specifies how the time is scaled in respect to the potential,
which means that the potential is only evaluated for PSO configurations $(X_{t},V_{t},L_{t})$ if $t$ is a multiple of $\triangle t$.
The effect of $\triangle t$ will be explained later.
In Section \ref{sec:experiments} and in all figures a potential is used which will be similar to the item-wise
product of the current velocity and the gradient of the function at the current position.
More precisely, the following potential will be used:
\begin{defi}[Experimental potential]\label{def:experimentalPotential}
The potential is defined as
$$
\Phi(t,d):=\max_{n\in\lbrace 1,\ldots,N\rbrace}\vert f(X^n_{\triangle t\cdot t}) -f(\tilde X^{n,d}_{\triangle t\cdot t})\vert
$$
where $\big(\tilde X^{n,d}_t\big)_{\tilde d}:=\begin{cases}
X^{n,\tilde d}_t+V^{n,\tilde d}_t, & \text{if }\tilde d=d,\\
X^{n,\tilde d}_t, & \text{otherwise}
\end{cases}$\\
and $f$ is the objective function which should be optimized.
\end{defi}
$\tilde X^{n,d}_t$ represents the position of a particle if only a step in dimension $d$ is done.
The deviation in one dimension is defined as a limit of a difference quotient.
We get this difference quotient by dividing the potential by the velocity.
If the absolute value of the velocity becomes very small we get something similar to the item-wise product of the
current velocity and the gradient as already mentioned.
If the absolute value of the gradient in some dimension tends to zero then higher deviations influence the potential.
It might be possible to use other potentials but for the objective functions, which are considered, this potential is sufficient.
Mainly there are two prerequisites which need to be fulfilled such that this potential is suitable.
Firstly, the velocities need to be nonzero almost surely.
If not all particles are initialized at the same point,
then the velocities will not encounter a value of zero almost surely.
Even for zero velocity initialization this is almost surely true, if the evaluation of the potential starts at the second step and standard position initialization is used.
Secondly, the objective function needs to fulfill the property that there exists no set of positive measure such that each point evaluates to the same function value.
For most continuous functions which have no plateaus this prerequisite is achieved.
Therefore this potential is positive for such functions almost surely if a standard PSO initialization is used.

By means of the logarithmic potential,
a stagnation phase can be defined.
If the logarithmic potential of a dimension becomes very low,
then there are other dimensions which have currently much more impact on the swarm behavior.
Finally, a dimension which has very low logarithmic potential for a long period of time will be far away from an optimized position in this dimension.
To specify this activity more in detail, the following definition is stated:
\begin{defi}[$(\triangle t, \Phi, N_0, c_0, c_s)$-Stagnation phase] \label{def:stagnationPhase}
Let $\triangle t$ be a constant step width, $\Phi$ a potential with the associated logarithmic potentials $\Psi$
and their increments $I_{t,d}$, $N_0$ a positive constant integer less than $D$ and $c_0\le c_s<0$ negative constants.
The following stopping times are defined:
\begin{align*}
\beta_{-1}:=&0\\
&\text{and inductively for all }i\ge 0:\\
\alpha_{i}:=&\triangle t\cdot\inf\lbrace t\ge\frac{\beta_{i-1}}{\triangle t}:\vert\lbrace d
\in\lbrace 1,\ldots,D\rbrace
:\Psi(t,d)\le c_0\rbrace\vert\ge N_0\rbrace
\\
\beta_i     :=&\triangle t\cdot\inf\lbrace t\ge\frac{\alpha_{i}}{\triangle t}:\vert\lbrace d
\in\lbrace 1,\ldots,D\rbrace
:\Psi(\frac{\alpha_i}{\triangle t},d)\le c_0
\wedge\max_{\frac{\alpha_i}{\triangle t}\le t'\le t}\Psi(t',d)\le c_s\rbrace\vert <  N_0\rbrace\\
\end{align*}
If $\alpha_i$ is finite, then it is said that the $i$'th $(\triangle t,$ $\Phi,$ $N_0,$ $c_0,$ $c_s)$-stagnation phase starts at time $\alpha_i$.
If $\beta_i$ is finite too, then it is said that the $i$'th $(\triangle t,$ $\Phi,$ $N_0,$ $c_0,$ $c_s)$-stagnation phase ends at time $\beta_i$.
If $\alpha_i$ is finite but $\beta_i$ is not finite then the $i$'th $(\triangle t,$ $\Phi,$ $N_0,$ $c_0,$ $c_s)$-stagnation phase does not end.
The event that a $(\triangle t,$ $\Phi,$ $N_0,$ $c_0,$ $c_s)$-stagnation phase starts at time $T_0\triangle t$ is defined as
$\lbrace \exists i\in\N_0:\alpha_i=T_0\triangle t\rbrace$.
\end{defi}
$N_0$ defines some minimal number of dimensions, which stagnate during the complete phase,
i.\,e., the logarithmic potential in these dimensions is at most $c_s$ during the stagnation phase.
$c_0$ defines a starting safety distance to the highest potential and $c_s$ defines a permanent safety distance to the highest potential such that dimensions which stay below that bound remain insignificant.
It is intended that it should not happen that stagnating dimensions, i.\,e., dimensions with low logarithmic potential,  influence the attractors,
because the swarm behaves differently if additional dimensions influence the attractors.
To ensure that stagnating dimensions have almost no influence on the swarm,
the logarithmic potential of stagnating dimensions need to be the safety distance apart of the logarithmic potential of dimension with most influence.
The bigger this safety distance is the more reasonable is the assumption of actual insignificance of stagnating dimensions.
It might happen that $\alpha_i=\beta_{i-1}$.
For example the set $\lbrace 1,\ldots,N_0\rbrace$ can be the set of dimensions, which have a logarithmic potential of at most $c_0$ at the start of a stagnation phase.
During the stagnation phase the logarithmic potential of dimension $(N_0+1)$ can become less than $c_0$.
If the logarithmic potential of dimensions $2,\ldots,N_0$ stay below $c_0$ and the logarithmic potential of dimension $1$ increases to a value larger than $c_s$,
then this stagnation phase ends and the next stagnation phase immediately starts because all dimensions in $\lbrace 2,\ldots, N_0+1\rbrace$ have a logarithmic potential of less than $c_0$.
Therefore it can happen that a stagnation phase starts immediately after another stagnation phase has ended, i.\,e., $\alpha_i=\beta_{i-1}$.
Furthermore, stagnation phases will not end at the same time as they start, i.\,e., $\alpha_i > \beta_i$,
because $c_0\le c_s$ and therefore the size of the set of dimensions with low logarithmic potential is at least $N_0$ at the beginning.

Dimensions with low logarithmic potential have low influence on the change of the value of the objective function during one step compared to the other dimensions.
This is obvious in respect to the experimental potential, which is introduced in Definition \ref{def:experimentalPotential},
because the $d$'th value represents the change of the function value if we do a step in dimension $d$ and leave the other dimensions as they are.
If the logarithmic potential of a dimension $d$ is low for some period of time, this dimension is called stagnating dimension.
Stagnating dimensions are called stagnating, because their small impact on the change in the value of the objective function results in a small impact on the decision whether a new position is a local attractor or a global attractor.
Therefore the decision whether a new position is an attractor does hardly depend on the movement in the stagnating dimensions.
\subsection{Unlimited Stagnation Phases}
\begin{defi} \label{nota:helpingRandomVariables}
Let $\Theta$ be a subset of probability distributions on $\R$.
Let $\triangle t$ be the constant step width.
Let $(B_{t,\Gamma_B})_{t\in\N,\Gamma_B\in\Theta}$ and $(J_{t,d,\Gamma_J})_{t\in\N,d\in\lbrace 1,\ldots,D\rbrace,\Gamma_J\in\Theta}$ be sets of random variables such that
\begin{itemize}
\item for all $t>0$ and $\Gamma_B\in\Theta$, $B_{t-1,\Gamma_B}$ is $\mathcal{A}_{\triangle t\cdot t}$-measurable and has distribution $\Gamma_B$,
\item for all $t>0$, $d\in \lbrace 1,\ldots,D\rbrace$ and $\Gamma_J\in\Theta$, $J_{t-1,d,\Gamma_J}$ is $\mathcal{A}_{\triangle t\cdot t}$-measurable and has distribution $\Gamma_J$,
\item for all $t\ge0$ and $\Gamma_B,\Gamma_J\in\Theta$, the random variables $J_{t,1,\Gamma_J}$, $\ldots$ , $J_{t,D,\Gamma_J}$, $B_{t,\Gamma_B}$ are independent and
\item for all $t\ge0$ and $\Gamma_B,\Gamma_J\in\Theta$, $\sigma(B_{t,\Gamma_B},(J_{t,d,\Gamma_J})_{1\le d\le D})$ is independent from $\mathcal{A}_{\triangle t\cdot t}$.
\end{itemize}
\end{defi}
The $\sigma$ operator provides the smallest $\sigma$-algebra, such that the arguments are measurable.
A detailed introduction into $\sigma$-algebras can be found in \cite{bauer1996probability}.
Only a small portion of the introduced variables are needed,
but if this huge set of random variables is not defined,
random variables would have been used which exist only under certain conditions,
which is formally not possible.

From now on the reality is approximated with the following model.
\begin{ass}[Separation of logarithmic potential]
\label{ass:separation}
It is assumed that for objective functions $f$ there exist $\triangle t$, $\Phi$, $N_0$, $c_0$ and $c_s$ such that
for any fixed time $T_{0}\cdot\triangle t$
there exist $\mathcal{A}_{\triangle t\cdot T_0}$-measurable $\Theta$-valued random variables $\Gamma_{J,T_0}$ and $\Gamma_{B,T_0}$ such that
\begin{itemize}
\item for all $T\ge T_0$ and $d\in\lbrace 1,\ldots,D\rbrace$ the random variables $B_{T,\Gamma_{B,T_0}}$ and $J_{T,d,\Gamma_{J,T_0}}$ are $\mathcal{A}_{(T+1)\triangle t}$-measurable,
\item for fixed times $T_0\triangle t$ and $T\triangle t$ and fixed dimension $d$, $\P(A\setminus B)=0$,
where 
\begin{align*}
A:=\lbrace& (\exists i:\alpha_i=T_0\triangle t\wedge \beta_i>T\triangle t)
\wedge\Psi(T_0,d)\le c_0\wedge \max_{T_0\le t'\le T}\Psi(t',d)\le c_s\rbrace
\end{align*}
$$
\text{and }\hspace{1cm}B:=\lbrace I_{T,d}=B_{T,\Gamma_{B,T_0}}+J_{T,d,\Gamma_{J,T_0}}\rbrace,
$$
\item for all $\omega\in\Omega$ the expectation of a $\Gamma_{J,T_0}(\omega)$-distributed random variable is zero and
\item for all $\omega\in\Omega$ 
the first six moments of a random variable, which has distribution $\Gamma_{B,T_0}(\omega)$ or $\Gamma_{J,T_0}(\omega)$, exist.
\end{itemize}
\end{ass}
If there is no start of a stagnation phase at time $T_0$ then $\Gamma_{J,T_0}$ and
$\Gamma_{B,T_0}$ can have any value which does not contradict with the restrictions in Assumption \ref{ass:separation}.
The value of $\Gamma_{J,T_0}$ and $\Gamma_{B,T_0}$ in these cases is not used in the analysis.

\begin{figure}[htb]
\center{\includegraphics[width=0.7\textwidth]{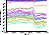}}
\caption{Logarithmic potential (with potential as specified in Definition \ref{def:experimentalPotential}) on sphere function in $25$ dimensions and with $3$ particles; each line represents the logarithmic potential of one dimension; visualized are time steps $200\,000$ to $210\,000$}
\label{figure:potential}
\end{figure}
The event $A$ specifies that there is a stagnation phase
which has started at time $T_0\triangle t$ and has not ended at time $T\triangle t$ or earlier
and the dimension $d$ is one of the dimensions with low logarithmic potential.
This event implies almost surely the event $B$
which states that the increments belonging to $\triangle t$ time steps after time $T\triangle t$ are composed by two components.

Mainly this assumption states that during stagnation phases the change in logarithmic potential of dimensions with low logarithmic potential
is composed of some basic random noise ($B_{T,\Gamma_{B,T_0}}$), which depends on the behavior of dimensions with high logarithmic potential,
and some individual random noise ($J_{T,d,\Gamma_{J,T_0}}$), which depends on the randomness in this specific dimension.
The basic random noise ($B_{T,\Gamma_{B,T_0}}$) will be called \emph{base increments} and the individual random noise ($J_{T,d,\Gamma_{J,T_0}}$) will be called \emph{dimension dependent increments}.
These two parts are independent of each other and they are also independent for different times $t$ as specified in Definition \ref{nota:helpingRandomVariables}.
The reason for this separation is that the dimensions with low logarithmic potential have no effect on local and global attractors of the swarm.
Therefore these dimensions do not interact with each other.
The distributions of the random variables $B_{T,\Gamma_{B,T_0}}$ and $J_{T,d,\Gamma_{J,T_0}}$ are fixed at time $T_0$,
because $\Gamma_{B,T_0}$ and $\Gamma_{J,T_0}$ are $\mathcal{A}_{T_0\triangle t}$-measurable.
The reason for this assumption is that according to the experiments the distributions mainly depend on the set of dimensions with low logarithmic potential,
which is fixed at the beginning of a stagnation phase.
Therefore $\Theta$,
the set of probability distributions on $\R$ specified in Definition \ref{nota:helpingRandomVariables},
need to contain only very few distributions.
The dimensions with high logarithmic potential determine the behavior of the particles,
which partially determines the logarithmic potential of the dimensions with low logarithmic potential.
An additional part of the logarithmic potential in a dimension arises from the random variables in the movement equations of the PSO in that dimension.
These two parts determine the logarithmic potential of dimensions with low logarithmic potential.
In Figure \ref{figure:potential} you can see the similar behavior of dimensions with low logarithmic potential.
Admittedly in any precise choice of a potential these new random variables will not be completely independent of each other,
but if $\triangle t$ is enlarged it can be expected that the dependencies become small.
For example, if $\triangle t$ is set to $1$ then the increments describe the behavior of consecutive steps which depend on each other in a strong manner
but for a larger $\triangle t$ larger intervals depend only partially on each other.
Therefore it is reasonable to accept the property of independence in the variable $t$ for the theoretical model as an approximation of the real-world process.

Furthermore, the dependencies and independencies in Assumption \ref{ass:separation} are not needed completely, but they are used to get some inequalities.
These inequalities will be described and justified experimentally in Section \ref{sec:experiments}.\\
\begin{defi}[$(\triangle t, \Phi, N_0, c_0, c_s, \mu, M, p_0)$-Stagnation phase] \label{def:stagnationPhaseClass}
Let $\mu\in\R$, $M\in\R$, $p_0\in[0,1]$, $T_0\ge 0$ be constants.\\
Let $A$ be the event that a $(\triangle t, \Phi, N_0, c_0, c_s)$-stagnation phase starts at time $T_0\triangle t$, i.\,e.,
$$
A=\lbrace\exists i\in\N_0:\alpha_i=T_0\triangle t\rbrace.
$$
Let $C_1$ be the event that the expectation of a $\Gamma_{B,T_0}$ distributed random variable is less or equal than $\mu$, i.\,e.,
$$
C_1=\lbrace\E[B_{T_0,\Gamma_{B,T_0}}\mid\mathcal{A}_{T_0\triangle t}]\le\mu\rbrace.
$$
Let $C_2$ be the event that the sixth moment of a $\Gamma_{B,T_0}$ distributed random variable and a $\Gamma_{J,T_0}$ distributed random variable is bounded by $M$, i.\,e.,
$H:=\E[B_{T_0,\Gamma_{B,T_0}}\vert \mathcal{A}_{T_0\triangle t}]$
$$
C_2=\Big\lbrace\E\Big[\big(B_{T_0,\Gamma_{B,T_0}}-H+J_{t,1,\Gamma_{J,T_0}}\big)^6\Bigm\vert \mathcal{A}_{T_0\triangle t}\Big]\le M\Big\rbrace.
$$
Let $C_3$ be the event that a $\Gamma_{B,T_0}$-distributed random variable stays below $\frac{\mu}{2}$ with a probability of at least $p_0$, i.\,e.,
\begin{align*}
C_3&=\Big\lbrace \P\Big(B_{T_0,\Gamma_{B,T_0}}-\frac{\mu}{2}\le 0 \Bigm\vert \mathcal{A}_{T_0\triangle t}\Big)\ge p_0\Big\rbrace
=\Big\lbrace \E\Big[\1_{B_{T_0,\Gamma_{B,T_0}}-\frac{\mu}{2}\le 0}\Bigm\vert \mathcal{A}_{T_0\triangle t}\Big]\ge p_0\Big\rbrace.
\end{align*}
Let $C_4$ be the event that a $\Gamma_{J,T_0}$-distributed random variable stays below $-\frac{\mu}{2}$ with a probability of at least $p_0$, i.\,e.,
\begin{align*}
C_4&=\left\lbrace \P\left(\left.J_{T_0,1,\Gamma_{B,T_0}}+\frac{\mu}{2}\le 0 \MIDR \mathcal{A}_{T_0\triangle t}\right)\ge p_0\right\rbrace
=\left\lbrace \E\left[\left.\1_{J_{T_0,1,\Gamma_{B,T_0}}+\frac{\mu}{2}\le 0}\MIDR \mathcal{A}_{T_0\triangle t}\right]\ge p_0\right\rbrace.
\end{align*}
A $(\triangle t,$ $\Phi,$ $N_0,$ $c_0,$ $c_s)$-stagnation phase is called $(\triangle t,$ $\Phi,$ $N_0,$ $c_0,$ $c_s,$ $\mu,$ $M,$ $p_0)$-stagnation phase if it achieves the conditions mentioned in $C_1$, $C_2$, $C_3$ and $C_4$, i.\,e., $A\cap C_1\cap C_2\cap C_3\cap C_4$ is called the event that a $(\triangle t, \Phi, N_0, c_0, c_s, \mu, M, p_0)$-stagnation phase starts at time $T_0\triangle t$.
\end{defi}
By definition the question whether a $(\triangle t,$ $\Phi,$ $N_0,$ $c_0,$ $c_s)$ -stagnation phase,
which starts at some time $T_0\cdot\triangle t$, is a $(\triangle t,$ $\Phi,$ $N_0,$ $c_0,$ $c_s,$ $\mu,$ $M,$ $p_0)$-stagnation phase can be answered at the beginning of that phase,
i.\,e., the indicator function of this event is $\mathcal{A}_{T_0\cdot\triangle t}$-measurable.
\begin{defi}[Good stagnation phase] \label{def:goodStagnationPhase}
A $(\triangle t,$ $\Phi,$ $N_0,$ $c_0,$ $c_s,$ $\mu,$ $M,$ $p_0)$-stagnation phase is called good if $\mu<0$ and $p_0>0$.
\end{defi}
These definitions only classify stagnation phases.
For a single stagnation phase which has negative expectation and bounded sixth moment $\mu$ and $M$ can be chosen as the exact values.
Furthermore, the probabilities which should be bounded below by $p_0$ are positive because the probability that a random variable encounters a value less or equal its expectation or even larger values is always positive.
Such stagnation phases are called good, because it can be proved that those phases do not end with positive probability.

The following theorem shows that, with positive probability, the swarm never recovers from encountering a good stagnation phase.
\begin{sat}
\label{sat:theorem1}
A good $(\triangle t, \Phi, N_0, c_0, c_s, \mu, M, p_0)$-stagnation phase does not end with a probability, which is at least $p:=p(N_0, \mu, M, p_0)>0$.
\end{sat}
\begin{proof}
Let $X_0$, $V_0$ and $L_0$ be initialized such that a good $(\triangle t,$ $\Phi,$ $N_0,$ $c_0,$ $c_s,$ $\mu,$ $M,$ $p_0)$-stagnation phase starts at time 0
and let $\Gamma_J$ and $\Gamma_B$ be the related probability distributions of this stagnation phase specified in Assumption \ref{ass:separation}.
As the PSO-process is a Markov-process, evaluating the process which starts in a good
$(\triangle t,$ $\Phi,$ $N_0,$ $c_0,$ $c_s,$ $\mu,$ $M,$ $p_0)$-stagnation phase is
equivalent to evaluate a good $(\triangle t,$ $\Phi,$ $N_0,$ $c_0,$ $c_s,$ $\mu,$ $M,$ $p_0)$-stagnation phase within a PSO run.
Let $S_0$ be $\lbrace d\in\lbrace 1,\ldots,D\rbrace \MID \Psi(0,d)\le c_0\rbrace$, the starting subset.
$S_0$, $\Gamma_B$ and $\Gamma_J$ are fixed objects here, because they are always known at the beginning of a stagnation phase.
Then the probability that this good stagnation phase does not end is equal to
$$
\P\Big(\exists U\subset S_{0}:\vert U\vert\ge N_0\wedge\sup_{d\in U, t>0} \Psi(t,d)\le c_{s}\Big).
$$
Let $S$ be defined as $\lbrace d\in S_0:\vert\lbrace 1,\ldots,d\rbrace\cap S_0\vert \le N_0\rbrace$, 
the subset of $S_{0}$, which contains the $N_0$ smallest indices.
$S$ is a fixed object too.
The probability that this good stagnation phase does not end is at least
$\P\left(\sup_{d\in S, t>0} \Psi(t,d)\le c_{s}\right)$.
Let $\tilde I_{t,d}:=B_{t,\Gamma_B}+J_{t,d,\Gamma_J}$ (see Assumption \ref{ass:separation}),
then $I_{t,d}\equiv\tilde I_{t,d}$ if the stagnation phase has not ended at time $t$ or earlier
and $d$ is in the set
$\lbrace \tilde d \in S_{0}\vert\max_{0< \tilde t\le t}\Psi(\tilde t,\tilde d)\le c_s\rbrace$.
Therefore $I_{t,d}\equiv\tilde I_{t,d}$ for all $d\in S$ if $\max_{\tilde d\in S,0<\tilde t\le t}\Psi(\tilde t, \tilde d)\le c_s$.
It follows
\begin{eqnarray*}
&&   \P         \Big(\sup_{t> {0}, d\in S}\Psi(t,d)\le c_{s}\Big)\allowdisplaybreaks\\
&=&  \P         \Big(\sup_{t> {0}, d\in S} \sum_{\tilde t={0}}^{t-1}\tilde I_{\tilde t,d}+\Psi({0},d) \le c_s\Big)\allowdisplaybreaks\\
&\ge&\P         \Big(\sup_{t> {0}, d\in S} \sum_{\tilde t={0}}^{t-1}\tilde I_{\tilde t,d}+c_0           \le c_s\Big)\allowdisplaybreaks\\
&=&  \P         \Big(\sup_{t> {0}, d\in S} \sum_{\tilde t={0}}^{t-1}\tilde I_{\tilde t,d} \le c_s-c_0   \Big)\allowdisplaybreaks\\
&\ge&\P         \Big(\sup_{t> {0}, d\in S} \sum_{\tilde t={0}}^{t-1}\tilde I_{\tilde t,d} \le 0         \Big)\allowdisplaybreaks\\
&=&\P         \Big(\sup_{t> {0}, d\in S} \sum_{\tilde t={0}}^{t-1}\Big(\Big(B_{\tilde t,  \Gamma_B}-\frac{\mu}{2}\Big)
+\Big(J_{\tilde t,d,\Gamma_J}+\frac{\mu}{2}\Big)\Big) \le 0\Big)\allowdisplaybreaks\\
&\ge&\P         \Big(\sup_{t> {0}}         \sum_{\tilde t={0}}^{t-1}\left(B_{\tilde t,  \Gamma_B}-\frac{\mu}{2}\right) \le 0
\wedge\sup_{t> {0}, d\in S} \sum_{\tilde t={0}}^{t-1}\left(J_{\tilde t,d,\Gamma_J}+\frac{\mu}{2}\right) \le 0\Big)\allowdisplaybreaks\\
&=&  \P         \Big(\sup_{t> {0}}         \sum_{\tilde t={0}}^{t-1}\left(B_{\tilde t,  \Gamma_B}-\frac{\mu}{2}\right) \le 0\Big)
\prod_{d\in S}\P\Big(\sup_{t> {0}}         \sum_{\tilde t={0}}^{t-1}\left(J_{\tilde t,d,\Gamma_J}+\frac{\mu}{2}\right) \le 0\Big).
\end{eqnarray*}
The expectations of $\left(B_{\tilde t,\Gamma_B}-{\mu}/{2}\right)$ and $\left(J_{\tilde t,d,\Gamma_J}+{\mu}/{2}\right)$ are at most ${\mu}/{2}$,
which is less than zero.
With Lemma \ref{lem:boundedMoments} the first six moments of
$$
\Big(B_{\tilde t,  \Gamma_B}-{\mu}/{2}-\E\big[B_{\tilde t,  \Gamma_B}-{\mu}/{2}\big]\Big)
\text{ and }
\Big(J_{\tilde t,d,\Gamma_J}+{\mu}/{2}-\E\big[J_{\tilde t,d,\Gamma_J}+{\mu}/{2}\big]\Big)
$$
can be bounded by $h(M)$.
Therefore Lemma \ref{lem:clt} tells us that
$$
\P\Big(\sup_{t> {0}}\sum_{\tilde t={0}}^{t-1}\big(B_{\tilde t,\Gamma_B}-{\mu}/{2}\big) \le 0\Big)\ge p({\mu}/{2},h(M),p_0)>0
$$
and
$$
\P\Big(\sup_{t> {0}}\sum_{\tilde t={0}}^{t-1}\big(J_{\tilde t,d,\Gamma_J}+{\mu}/{2}\big) \le 0\Big)\ge p({\mu}/{2},h(M),p_0)>0.
$$
Altogether we have
\begin{eqnarray*}
\lefteqn{\P\big(\exists U\subset S_{0}:\vert U\vert\ge N_0\wedge\sup_{d\in U, t> {0}} \Psi(t,d)\le c_{s}\big)}\allowdisplaybreaks\\
&\ge&\P         \big(\sup_{t> {0}, d\in S}\Psi(t,d)\le c_{s}\big)\allowdisplaybreaks\\
&\ge&\P         \Big(\sup_{t> {0}}         \sum_{\tilde t={0}}^{t-1}\big(B_{\tilde t,  \Gamma_B}-\frac{\mu}{2}\big) \le 0\Big)
\prod_{d\in S}\P\Big(\sup_{t> {0}}         \sum_{\tilde t={0}}^{t-1}\big(J_{\tilde t,d,\Gamma_J}+\frac{\mu}{2}\big) \le 0\Big)\allowdisplaybreaks\\
&\ge&p\left(\frac{\mu}{2},h(M),p_0\right)\cdot p\left(\frac{\mu}{2},h(M),p_0\right)^{N_0}\allowdisplaybreaks\\
&=:&p(N_0, \mu, M, p_0)=:p
\end{eqnarray*}
\end{proof}
According to Donsker's theorem \cite{billing}
incremental processes composed of independent idential distributed increments converge to a Brownian Motion if they are correctly scaled.
As our general model uses independent increments, we also discuss the Brownian Motion as a feasible model for our process.
To acquire Brownian Motion it is additionally assumed in the model that the increments are normally distributed during good stagnation phases.
Then a Brownian Motion,
with some variance $\sigma^2$ and negative drift $\mu$ according to the expectation and variance of the increments,
approximates the logarithmic potential.
I.\,e., the Brownian Motion at discrete times would be equal to the logarithmic potential at these times.
In \cite{Scheike:92}, it is proved that the probability that a standard Brownian Motion crosses some line $y=m\cdot x+t$ is equal to $\exp(-2\cdot m\cdot t)$ for positive values $m$ and $t$.
Under these assumptions, $\left(\sum_{\tilde t=T_0}^{t-1}\left( \tilde I_{\tilde t,d} -\mu\right)\right)/\sigma$ can be approximated by a standard Brownian Motion,
i.\,e., a Brownian Motion with expectation 0 and variance 1 per time step.
The previous bound $c_s$ is now moved to $-\frac{\mu}{\sigma}t+\frac{c_s-c_0}{\sigma}$.
If the Brownian Motion stays below this line,
then it also stays below the line at the discrete time steps referring to time steps of the PSO,
which implies that the logarithmic potential stays below $c_s$ indefinitely.
Hence, the probability that the logarithmic potential stays below $c_s$ forever,
is approximated by $1-\exp(2\frac{(c_s-c_0)\mu}{\sigma^2})$,
where $\mu$ and $\sigma^2$ represent the expectation and variance of the increments per $\triangle t$ time steps.
An approximate lower bound for the probability that at least $N_0$ dimensions in $S_0$ stay below $c_s$ is represented by the $N_0$'th power of the probability in one dimension:
$(1-\exp(-2\frac{(c_s-c_0)\mu}{\sigma^2}))^{N_0}$.
In the previous model we were only able to determine some positive lower bound for the probability that stagnation remains indefinitely.
This model supplies the previously mentioned formula, which approximates the actual probability quite well for at least $N_0=1$.
In Section \ref{sec:brownianMotion} the values obtained from that formula will be compared to experimentally measured values.
\begin{figure}[htb]
\center{\includegraphics[width=0.7\textwidth]{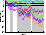}}
\caption{Logarithmic potential (with potential as specified in Definition \ref{def:experimentalPotential}) on sphere function in 10 dimensions and with 2 particles; each line represents the logarithmic potential of one dimension, $c_0=-40$, $c_s=-20$, $N_0=7$}
\label{figure:phases}
\end{figure}
\subsection{Convergence to Non-Optimal Points}
\begin{defi}[Phases]
\label{def:phases}
Let $\triangle t$, $\Phi$, $N_0$, $c_0$, $c_s$, $\mu$, $M$ and $p_0$ be fixed objects.
Let $(\alpha_i)_{i\in\N_0}$ and $(\beta_i)_{i\in\N_0}$ be the stopping times specified in Definition \ref{def:stagnationPhase}.
Furthermore, the following stopping times are defined:
\begin{align*}
\tilde\beta_{-1}=0\phantom{\inf}\\
\intertext{and inductively for all $i\ge 0$:}
\tilde\alpha_{i}=\inf\lbrace& t\ge\tilde \beta_{i-1}:\exists j\in\N_0:\alpha_j=t\wedge\text{the }(\triangle t,\Phi,N_0,c_0,c_s)
\text{-stagnation phase}
\\&
\text{starting at time }t\text{ is a}
(\triangle t,\Phi,N_0,c_0,c_s,\mu,M,p_0)\text{-stagnation phase}\rbrace,\\
\tilde\beta_i=     \inf\lbrace& t  >\tilde\alpha_i:\exists j\in\N_0:\beta_j=t\rbrace.
\end{align*}
$\tilde\alpha_i$ and $\tilde\beta$ determine the start and end of the $i$'th $(\triangle t,$ $\Phi,$ $N_0,$ $c_0,$ $c_s,$ $\mu,$ $M,$ $p_0)$-stagnation phase.
Phases are random time intervals.
Phases of type ${\mathit PH}_X$ are phases before or between $(\triangle t,$ $\Phi,$ $N_0,$ $c_0,$ $c_s,$ $\mu,$ $M,$ $p_0)$-stagnation phases.
The $i$'th phase of type ${\mathit PH}_X$ is defined as
\begin{center}
$\begin{array}{c l}
{[}\tilde\beta_{i-1},\tilde\alpha_i{]}&\text{if }\tilde\alpha_i<\infty\\
{[}\tilde\beta_{i-1},\infty{[}&\text{if }\tilde\alpha_i=\infty\wedge\tilde\beta_{i-1}<\infty\\
\emptyset&\text{otherwise.}
\end{array}$
\end{center}
Phases of type ${\mathit PH}_{Y}$ are $(\triangle t,$ $\Phi,$ $N_0,$ $c_0,$ $c_s,$ $\mu,$ $M,$ $p_0)$-stagnation phases which have an end.
The $i$'th phase of type ${\mathit PH}_{Y}$ is defined as
\begin{center}
$\begin{array}{c l}
{[}\tilde\alpha_i,\tilde\beta_i{]}&\text{if }\tilde\beta_i<\infty\\
\emptyset&\text{otherwise.}
\end{array}$
\end{center}
The phase of type ${\mathit PH}_{F}$ is a $(\triangle t,$ $\Phi,$ $N_0,$ $c_0,$ $c_s,$ $\mu,$ $M,$ $p_0)$-stagnation phase which has no end and it is defined as
\begin{center}
$\begin{array}{c l}
{[}\tilde\alpha_i,\infty{[}&\text{if }\tilde\alpha_i<\infty\wedge\tilde\beta_i=\infty\\
\emptyset&\text{if no such }i\text{ exists.}
\end{array}$
\end{center}
$X_i$ is defined to be the duration of the $i$'th phase of type ${\mathit PH}_X$:
$$
X_i:=\begin{cases}
\tilde\alpha_i - \tilde\beta_{i-1}& \text{if } \tilde\beta_{i-1}<\infty\\
0 & \text{otherwise.}
\end{cases}
$$
$Y_i$ is defined to be the duration of the $i$'th phase of type ${\mathit PH}_Y$:
$$
Y_i:=\begin{cases}
\tilde\beta_i - \tilde\alpha_{i}& \text{if } \tilde\beta_{i}<\infty\\
0 & \text{otherwise.}
\end{cases}
$$
Furthermore, the random variable $T_F$ is defined as the time when the final phase of type ${\mathit PH}_{F}$ starts:
$$
T_F:=\sum_{i=0}^\infty X_i+Y_i.
$$
$Y_i$ is a $\N_0$-valued random variable and $X_i$, $T_F$ are $\N_0\cup\lbrace\infty\rbrace$-valued random variables.
\end{defi}
A phase of type ${\mathit PH}_{X}$ is the interval beginning
at the start of the PSO or at the end of a ($\triangle t$, $\Phi$, $N_0$, $c_0$, $c_s$, $\mu$, $M$, $p_0$)-stagnation phase
and ending at the start of the next ($\triangle t$, $\Phi$, $N_0$, $c_0$, $c_s$, $\mu$, $M$, $p_0$)-stagnation phase.
If there is no further ($\triangle t$, $\Phi$, $N_0$, $c_0$, $c_s$, $\mu$, $M$, $p_0$)-stagnation phase then this phase does not end.
A phase of type ${\mathit PH}_{Y}$ is the time interval of a ($\triangle t$, $\Phi$, $N_0$, $c_0$, $c_s$, $\mu$, $M$, $p_0$)-stagnation phase which has an end
and a phase of type ${\mathit PH}_{F}$ is the time interval of a ($\triangle t$, $\Phi$, $N_0$, $c_0$, $c_s$, $\mu$, $M$, $p_0$)-stagnation phase which has no end.
The duration of a phase equals the difference of the last time step and the first time step contained in the related time interval and is infinity if the time interval has no limit.
If the $i$'th instance of type ${\mathit PH}_X$-phase of a specific run does not exist
then $X_i$ is zero.
If the $i$'th instance of type ${\mathit PH}_Y$-phase of a specific run does not exist
then $Y_i$ is zero.
$T_F$ is infinity if no phase of type ${\mathit PH}_{F}$ starts.
The process starts with a phase of type ${\mathit PH}_{X}$.
Phases of type ${\mathit PH}_X$ and ${\mathit PH}_Y$ are alternating until the final phase, which has type ${\mathit PH}_F$, begins.
Alternatively the last phase could be a phase of type ${\mathit PH}_X$ or phases of type ${\mathit PH}_X$ and ${\mathit PH}_Y$ are alternating infinitely often,
but it will be shown in Theorem \ref{sat:theorem2} that this does not happen almost surely.
It may occur that a phase of type ${\mathit PH}_X$ has duration zero if the set of dimensions with low logarithmic potential has changed.
Figure \ref{figure:phases} shows an example run of a $10$-dimensional PSO with $2$ particles.
$c_0:=-40$, $c_s:=-20$ and $N_0:=7$,
i.\,e., a stagnation phase starts whenever $7$ or more of the $10$ dimensions have logarithmic potential below $-40$.

In the following definition a group of functions is specified, which allows for comprehensible assumptions.
\begin{defi}[Composite functions]
A function is called composite if it can be written as  $f(x)=g(\sum_{i=1}^D f_i(x_i))$,
where $g$ is a strictly monotonically increasing function and $(f_i)_{i\in\lbrace 1,\dots,D\rbrace}$ are functions which have a lower bound each.
\end{defi}
For example all $q$-norms $(\sum_{i=1}^D \vert x_i\vert^q)^\frac{1}{q}$ are composite functions
with $f_i(x_i):=\vert x_i\vert^q$ and\linebreak $g(y):=$sign$(y)\vert y\vert^\frac{1}{q}$,
where sign represents the signum function
sign$(y):=1$ if $y>0$, $0$ if $y=0$ and $-1$ if $y<0$.

Furthermore, the following notation will be used:
\begin{defi}
For objects $W=(W^1,\ldots,W^D)$ and $S\subset\lbrace 1,\ldots,D\rbrace$ we write $W^{d\in S}$ for the object,
where all entries in dimensions contained in $S$ remain,
i.\,e., let $d_1$ to $d_{\vert S\vert}$ be the dimensions contained in $S$ in increasing order
then $W^{d\in S}=(W^{d_1},\ldots,W^{d_{\vert S\vert}})$.\\
Analogously $W^{d\not\in S}$ is written for $W^{d\in\lbrace 1,\ldots,D\rbrace\setminus S}$,
where all entries in dimensions contained in $S$ are removed.
\end{defi}
\begin{ass}[Insignificance of stagnating dimensions (1)]\label{ass:insignificance1}
Let $T_0\in\N$, $T\in\N$ and let $S$ be a subset of $\lbrace 1,\ldots,D\rbrace$, such that $T_0<T$ and $\vert S\vert \ge N_0$.
For fixed objects $\triangle t$, $\Phi$, $N_0$, $c_0$, $c_s$ and a composite objective function $f(x)=g(\sum_{i=1}^D f_i(x_i))$
a modified PSO can be defined.
The modified PSO starts at time $T_0\triangle t$ with initialization
$$
\tilde X_{T_0\triangle t}:=X_{T_0\triangle t}^{d\not\in S}\text{, }
\tilde V_{T_0\triangle t}:=V_{T_0\triangle t}^{d\not\in S}\text{,}
\tilde L_{T_0\triangle t}:=L_{T_0\triangle t}^{d\not\in S}\text{ and }
$$
$$
\text{objective function }
\tilde f:=\sum_{i\not\in S}f_i(x_i).
$$
Furthermore, similar to Definition \ref{def:classicalPSO} the movement equations for the modified PSO for all particles $n$ ($1\le n \le N$) are defined to be
\begin{align*}
\tilde G_t^n     :=& \argmin_{x\in\bigcup_{i=1}^{n-1}\lbrace \tilde L_{t+1}^i\rbrace\cup\bigcup_{i=n}^N\lbrace \tilde L_{t}^i\rbrace}\tilde f(x)\text{ for }t\ge T_0\triangle t,\\
\tilde V_{t+1}^n :=& \chi\cdot \tilde V_t^n + c_1\cdot (r_t^n)^{d\not\in S} \odot (\tilde L_t^n-\tilde X_t^n)
+ c_2\cdot (s_t^n)^{d\not\in S} \odot (\tilde G_{t}^n-\tilde X_t^n)\text{ for }t\ge T_0\triangle t,\\
\tilde X_{t+1}^n :=& \tilde X_t^n+\tilde V_{t+1}^n\text{ for }t \ge T_0\triangle t,\\
\tilde L_{t+1}^n :=& \argmin_{x\in\lbrace \tilde X_{t+1}^n, \tilde L_t^n\rbrace}\tilde f(x)\text{ for }t \ge T_0\triangle t.
\end{align*}
with reused random variables $(r_t^n)^{d\not\in S}$ and $(s_t^n)^{d\not\in S}$.
Let $A$ and $B$ be the events
\begin{align*}
A:=\lbrace &\exists i\in\N_0:\alpha_i=T_0\triangle t
\wedge \max_{d\in S}\Psi(T_0,d)\le c_0
\wedge \max_{d\in S,T_0\le t\le T}\Psi(t,d)\le c_s\rbrace
\end{align*}
and
\begin{align*}
B:=\lbrace &X_t^{d\not\in S}=\tilde X_{t}\wedge V_t^{d\not\in S}=\tilde V_{t}
\wedge L_t^{d\not\in S}=\tilde L_{t}\forall t\in\lbrace T_0\triangle t,\ldots, T\triangle t-1\rbrace\rbrace.
\end{align*}
It is assumed that there exist $\triangle t$, $\Phi$, $N_0$, $c_0$ and $c_s$ such that Assumption \ref{ass:separation} is applicable and $\P(A\setminus B)=0$.
\end{ass}
The modified PSO represents the actual PSO, but the dimensions with low logarithmic potential are completely removed.
The event $A$ appears if a stagnation phase starts at time $T_0\triangle t$ and has not ended at time $T\triangle t$ and $S$ is the set of stagnating dimensions.
The event $B$ appears if the reduced process and the actual process are equal at times $t$ such that $T_0\le t<T$ for all dimensions not in $S$.
The assumption states that almost surely $A$ implies $B$,
which means that this modified PSO behaves similar to the actual PSO
if a stagnation phase starts at time $T_0$ and all dimensions in $S$ are stagnating.
Therefore stagnating dimensions have no effect on the other dimensions, 
but dimensions not in $S$ have effect on the stagnating dimensions
because the times when attractors are updated are determined by dimensions not in $S$.

In the following a short motivation is presented why this assumption is reasonable.
The potential used in this paper represents the impact of a single dimension to the change in the function value in a single step.
The probability that a dimension can change the evaluation whether a new position is a local or global attractor can therefore approximately be bounded by $2^{\Psi(t,d)}$.
For stagnation phases with $c_s=-20$ this value is approximately $0.000\,001$
and therefore the expected waiting time that
a dimension with that low potential has effect on the behavior of the swarm is approximately bounded below by $1\,000\,000$ iterations.
Furthermore, for good stagnation phases the logarithmic potential is expected to decrease linearly on average.
Neglecting the randomness the logarithmic potential at time $t$ can be bounded by $c_0-\mu(t-T_0\triangle t)$,
where $-\mu$ is the negative expectation of the potential per time step, 
and therefore the expected number of time steps such that
this dimension has effect on a decision whether a new position is a local or global attractor after the beginning of an unlimited stagnation phase 
is approximately bounded by
$$
\sum_{i>0}2^{c_0-\mu i}\le 2^{c_0}\int_{0}^{\infty}\exp(-\ln(2)\mu i)di=2^{c_0}\frac{1}{\ln(2)\mu},
$$
which is a finite value.
Furthermore, $c_0$ can be chosen such that this value is even less than one.
Also changes in dimensions with low logarithmic potential can not accumulate over time,
because if updates in local and global attractors happen quite often, then accumulation can not appear
and if the updates occur not quite often,
then changes will not accumulate as well
because the positions of the particles will mainly stay between the local and the global attractors.

The last assumption states that dimensions,
which have low logarithmic potential,
have no effect on the swarm.
This means that those dimensions receive a series of decision results whether a position is a new local or global attractor.
Those decision results influence how the positions in dimensions with low logarithmic potential develop.

In this paper the behavior of the PSO for large time scales is analyzed.
Therefore it is necessary to define in which cases convergence is present.
\begin{defi}[Swarm convergence]\label{def:swarmConvergence}
A swarm is said to \emph{converge} if the coordinates of the global and local attractors
and the positions converge to the same point and the velocities of the particles converge to zero.
The limit of the positions is also called \emph{point of convergence}.
\end{defi}
If the velocity tends exponentially to zero then convergence of attractors and positions also appears.

If there is a stagnation phase which has no end then there is a non empty subset of dimensions
and the dimensions in this set have low logarithmic potential as soon as the last stagnation phase starts.
Assumption \ref{ass:insignificance1} states that there exists an alternative PSO without that dimensions.
Although these dimensions have no effect on the remaining dimensions,
these dimensions are still changing.
The dimensions with low logarithmic potential receive a series of updates of the local and global attractors from the reduced process.
If convergence appears, some additional conclusions can be made.
Depending on the update series and the random variables for the dimensions with low logarithmic potential, some final value for that dimensions is reached.
As it is not plausible that the limit of the coordinates for that dimensions is fixed at the start of the last stagnation phase,
it is assumed that there is no limit value of these coordinates, which has positive probability.
Neither optimal values nor any other value.

Therefore the following assumption is concluded.
\begin{ass}[Insignificance of stagnating dimensions (2)]\label{ass:insignificance2}
For fixed objects $\triangle t$, $\Phi$, $N_0$, $c_0$, $c_s$, $\mu$, $M$ and $p_0$
let $X_0$, $V_0$ and $L_0$ be initialized such that it represents a possible start of a good $(\triangle t, \Phi, N_0, c_0, c_s, \mu, M, p_0)$-stagnation phase,
let $\Gamma_J$ and $\Gamma_B$ be the related probability distributions of this stagnation phase specified in Assumption \ref{ass:separation}
and let
$$
S\subset\lbrace d\in\lbrace 1,\ldots,D\rbrace\MID \Psi(0,d)\le c_0\rbrace
$$
such that $\vert S\vert =N_0$.
Let $\mathcal{E}$ be the event 
$$
\mathcal E:=\lbrace \sup_{t\ge 0,d\in S}\Psi(t,d)\le c_s\wedge \text{the swarm converges}\rbrace.
$$
It is assumed that there exist $\triangle t$, $\Phi$, $N_0$, $c_0$, $c_s$, $\mu$, $M$ and $p_0$ such that
Assumption \ref{ass:separation} is applicable and
$$
\P(\lim_{t\rightarrow\infty}X_t^{n,d}=v\MID \mathcal{E})=0
\text{ for all }
n\in\lbrace 1,\ldots,N\rbrace\text{, }v\in\R\text{ and }d\in S.
$$
\end{ass}
This assumption states that if the PSO reaches a stagnation phase
which does not end and the PSO converges,
then it is almost surely true that specific points will not be reached,
because at least the coordinates of the insignificant dimensions will differ almost surely.
This implies that if the set of local optima is countable
then almost surely none of the local optima will be reached,
because the set of possible coordinate values of optimal points is countable for each dimension and therefore the
dimensions with low logarithmic potential will not reach any of those values almost surely.
Therefore stagnation occurs in those cases.

Assumption \ref{ass:insignificance2} is not restricted to composite functions.
The behavior for functions which are not composite is more complicated, but the effect is the same.
An aspect which will not be proved is that dimensions which become relevant from time to time are optimizing.
If the coordinates in dimensions with low logarithmic potential are converging to a non-optimal value
then the optimal value for the coordinates,
which are optimizing, is probably not the same as in the case where all dimensions are optimizing.
For example the function $f(x,y)=2(x+y)^2+(x-y)^2=3x^2+2xy+3y^2$ is minimal for $x=0$ and $y=0$,
but if $x$ tends to $1$ then $f(1,y)=3y^2+2y+3$ is minimal for $y=-\frac{1}{3}\not=0$.
Nevertheless, the effect is the same for functions, which are not composite,
because the influence of the dimensions with low logarithmic potential is heavily delayed.
The scales of the change in stagnating dimensions and non-stagnating dimensions are very different
and therefore the change in the optimal position is recognized by the optimizing dimensions later in the process,
when the stagnating dimensions do not have the possibility to change the optimal position and their position in a comparable manner.

For example if the logarithmic potential of a dimension $d$ stays smaller than $-100$
then (for the specified experimental potential and tested functions)
also the velocities in dimension $d$ are by a factor of $2^{-100}$ smaller than the velocities in the dimensions with most potential.
If an attractor is updated, then its $d$-th position entry changes a bit.
This change might result in a minor change of the optimal position for all other dimensions
if the position value in dimension $d$ is regarded as a constant value,
i.\,e. let $f_c(x_1,\ldots,x_{d-1},x_{d+1},\ldots,x_D):=f(x_1,\ldots,x_D)$,
let $c_{old}$ be the previous value of the $d$-th position entry
and let $c_{new}$ be the new value of the $d$-th position entry,
then the optimal position of $f_{c_{old}}$ and $f_{c_{new}}$ may vary a bit.
Currently the attractor choice does hardly depend on that change,
but if the swarm finally converges,
then the velocity in dimensions with most logarithmic potential finally reaches the same scale
as the previous change of the optimal position where
the position value in dimension $d$ is regarded as a constant.
The velocity in dimension $d$ has also decreased and is again much smaller.
Therefore it cannot undo the changes done many steps in the past,
where the velocity in dimension $d$ had the same scale as the dimensions with large logarithmic potential now have.
This is meant as the heavily delayed influence of dimensions with low logarithmic potential.
As the movement in dimension $d$ was at no time guided to the
optimal point in this dimension, it is almost impossible that this dimension is optimizing.

With Assumption \ref{ass:insignificance2} the main theorem of this paper can be proved.
\begin{sat}
\label{sat:theorem2}
Let $\Phi$ be a potential, $c_0\le c_s<0$, $\triangle t>0$, $N_0>0$ be constants such that the Assumptions \ref{ass:separation} and \ref{ass:insignificance2} are applicable.
Furthermore, let $T>0$, $\mu<0$, $M<\infty$ and $p_0>0$ be constants.
Let $A_{X,i}:=\lbrace \tilde\beta_{i-1}<\infty\rbrace$ for all $i$.
If the following conditions hold:
\begin{itemize}
\item The swarm converges almost surely,
\item the objective function $f$ has only a countable number of local optima and
\item for all $i$ the expectation $\E[X_i\MID A_{X,i}]\le T$ if $\P(A_{X,i})>0$,
\end{itemize}
then the swarm converges to a point, which is not a local optimum, almost surely.
\end{sat}
\begin{proof}
If the final phase which has type ${\mathit PH}_F$ appears,
then the PSO will not converge to a local optimum for reasons explained below Assumption \ref{ass:insignificance2}.
Now all which is needed to be shown is that a phase of type ${\mathit PH}_F$ appears almost surely.
This is true if it can be shown that the expected value of $T_F$, the starting time of the phase of type ${\mathit PH}_F$, is finite.\\
Let $A_{X,i}:=\lbrace \tilde\beta_{i-1}<\infty\rbrace$, the event that the $i$'th instance of a phase which has type ${\mathit PH}_{X}$ appears,
let $A_{Y,i}:=\lbrace\tilde\beta_{i}<\infty\rbrace$, the event that the $i$'th instance of a phase which has type ${\mathit PH}_{Y}$ appears
and let $A_{S,i}:=\lbrace\tilde\alpha_{i}<\infty\rbrace$,
the event that the $i$'th instance of a good $(\triangle t,$ $ \Phi,$ $N_0,$ $c_0,$ $c_s,$ $\mu,$ $M,$ $p_0)$-stagnation phase appears.
Let $N_X$ be $\sup\lbrace i\in\N:\P(A_{X,i})>0\rbrace$ and let $N_Y$ be $\sup\lbrace i\in\N:\P(A_{Y,i})>0\rbrace$.
Both values can be infinity.
Let $S_{0,i}$ be the random variable, which represents the starting subset of the $i$'th good $(\triangle t, \Phi, N_0, c_0, c_s, \mu, M, p_0)$-stagnation phase
or the empty set if this phase does not occur.
Let $\tilde I_{t,d,i}$ be $B_{t,\tilde\Gamma_{B,i}}+J_{t,d,\tilde\Gamma_{J,i}}$, similar as in the proof of Theorem \ref{sat:theorem1},
such that $\tilde\Gamma_{B,i}$ and $\tilde\Gamma_{J,i}$ are the random probability distributions
of the $i$'th good $(\triangle t, \Phi, N_0, c_0, c_s, \mu, M, p_0)$-stagnation phase.
The random probability distributions are introduced in Assumption \ref{ass:separation}.
If the $i$'th good $(\triangle t, \Phi, N_0, c_0, c_s, \mu, M, p_0)$-stagnation phase
does not exist then $\tilde\Gamma_{B,i}$ is defined to be a Dirac impulse on $\mu$
and $\tilde\Gamma_{J,i}$ is defined to be a Dirac impulse on zero.
Additionally, let $\mathcal{A}$ be the generated $\sigma$-Algebra
of the family of $\sigma$-Algebras $\lbrace\mathcal{A}_t\rbrace_{t\in\N}$,
where $\mathcal{A}_t$ is the natural filtration of the PSO as specified in Definition \ref{def:classicalPSO}.
This $\sigma$-algebra is needed because for any time $t$ it is not possible to decide whether a stagnation phase,
which is active at time $t$ has type ${\mathit PH}_F$ or ${\mathit PH}_Y$.
Therefore $A_{Y,i}$ is not necessarily contained in $\mathcal{A}_t$ for any $t$.
The conditional expectation of $X_i$, the duration of the $i$'th phase of type ${\mathit PH}_X$, is by the requirements of this theorem bounded by
$$
E[X_i\MID A_{X,i}]\le T
$$
if the event $A_{X,i}$ has positive probability.
Let $T_{S,i}$ be the random variable,
such that $T_{S,i}\cdot \triangle t$ represents the time of the start of the $i$'th stagnation phase,
which has either type ${\mathit PH}_Y$ or ${\mathit PH}_F$.
If the $i$'th stagnation phase does not occur then $T_{S,i}$ is set to infinity.
As already specified, $Y_i$ represents the duration of the $i$'th phase of type ${\mathit PH}_Y$ if that phase exists and zero otherwise.
Let $p$ be defined as $p(N_0,\mu,M,p_0)$,
the lower bound of the probability that a good $(\triangle t, \Phi, N_0, c_0, c_s, \mu, M, p_0)$-stagnation phase does not end from Theorem \ref{sat:theorem1}.
\begin{eqnarray*}
\E[T_F]
&=&\sum_{i=0}^{N_X}\P(A_{X,i})\E[X_i\MID A_{X,i}]+\sum_{i=0}^{N_Y}\underbrace{\P(A_{Y,i})\E[Y_i\MID A_{Y,i}]}_{=\E[Y_i]}\allowdisplaybreaks\\
&=&\sum_{i=0}^{N_X}\P(A_{X,i})\E[X_i\MID A_{X,i}]+\sum_{i=0}^{N_Y}\P(A_{S,i})\E[Y_i\MID A_{S,i}]\hspace{6cm}\phantom{a}
\end{eqnarray*}
compare with end of stagnation phase (Definition \ref{def:stagnationPhase})
\begin{eqnarray*}
&=&       \sum_{i=0}^{N_X}\P(A_{X,i})\E[X_i\MID A_{X,i}]+\sum_{i=0}^{N_Y}\P(A_{S,i})\sum_{t=1}^\infty t\triangle t\cdot\\
&&\cdot\P(\vert\lbrace d\in S_{0,i}:\sup_{0\le\tilde t<t}\Psi(T_{S,i}+\tilde t,d)\le c_s\rbrace\vert\ge N_0
\\&&
\wedge \vert\lbrace d\in S_{0,i}:\sup_{0\le\tilde t\le t}\Psi(T_{S,i}+\tilde t,d)\le c_s\rbrace\vert< N_0\vert A_{S,i})\allowdisplaybreaks\\
&\le&\sum_{i=0}^{N_X}(1-p)^i\E[X_i\MID A_{X,i}]+\sum_{i=0}^{N_Y}(1-p)^i\sum_{t=1}^\infty t\triangle t\cdot\\
&&\cdot\P(\vert\lbrace d\in S_{0,i}:\sup_{0\le\tilde t<   t}\sum_{t'=T_{S,i}}^{T_{S,i}+\tilde t}\tilde I_{t',d,i}\le 0\rbrace\vert\ge N_0
\wedge  \vert\lbrace d\in S_{0,i}:\sup_{0\le\tilde t\le t}\sum_{t'=T_{S,i}}^{T_{S,i}+\tilde t}\tilde I_{t',d,i}\le 0\rbrace\vert <  N_0\vert A_{S,i})\allowdisplaybreaks\\
&\le&\sum_{i=0}^{\infty}(1-p)^i T+\sum_{i=0}^{N_Y}(1-p)^i\sum_{t=1}^\infty t\triangle t\cdot
\P(\exists d\in S_{0,i}:
\sup_{0\le\tilde t< t}\sum_{t'=T_{S,i}}^{T_{S,i}+\tilde t}\tilde I_{t',d,i}\le 0
\wedge \sum_{t'=T_{S,i}}^{T_{S,i}+t}\tilde I_{t',d,i}> 0\vert A_{S,i})\allowdisplaybreaks\\
&\le&\frac{T}{p}+\sum_{i=0}^{N_Y}(1-p)^i\sum_{t=1}^\infty t\triangle t
\sum_{d=1}^D
\P(d\in S_{0,i}\MID A_{S,i})\cdot
\P(\sup_{0\le\tilde t<t}\sum_{t'=T_{S,i}}^{T_{S,i}+\tilde t}\tilde I_{t',d,i}\le 0 \wedge \sum_{t'=T_{S,i}}^{T_{S,i}+t}\tilde I_{t',d,i}> 0\vert A_{S,i})\allowdisplaybreaks\\
&\le&\frac{T}{p}+\sum_{i=0}^{N_Y}(1-p)^i\sum_{t=1}^\infty t\triangle t
\sum_{d=1}^D\P(\sum_{t'=T_{S,i}}^{T_{S,i}+t}\tilde I_{t',d,i}> 0\vert A_{S,i})\allowdisplaybreaks\\
&\overset{\text{Lemma \ref{lem:clt}}}{\le}\hspace*{-1cm}&\hspace*{0.8cm}\frac{T}{p}+\sum_{i=0}^{N_Y}(1-p)^i\sum_{t=1}^\infty t\triangle t
\sum_{d=1}^D
\frac{C(\mu,h(M))}{t^3}\allowdisplaybreaks\\
&\le&\underbrace{\frac{T}{p}+\frac{1}{p}D\triangle t\cdot C(\mu,h(M))\sum_{t=1}^\infty t^{-2}}_{=:C(\triangle t, N_0,\mu,M,p_0,T)<\infty}.
\end{eqnarray*}
\end{proof}
This is the final statement of the theoretical part.
The proof of Theorem \ref{sat:theorem2} not only shows that the final phase with type ${\mathit PH}_F$ appears almost surely,
but it also verifies that the expected begin of the final phase is a finite value.
Now evidence is provided in the experimental part,
that those conditions truly occur,
and therefore that PSO does not converge to a local optimum almost surely
for some well known benchmarks and parameters of the PSO,
which are commonly used.

\newpage
\section{Experimental Results}
\label{sec:experiments}
Since we are interested in the behavior of the swarm for $t\rightarrow\infty$, we let the swarm do a very large number of iterations.
During the ongoing process, the absolute value of the velocities and the change in the objective function tend to zero.
Therefore calculations with double precision are not sufficiently precise.
Instead, we used the \emph{mpf\_t} data type of the \emph{GNU Multiple Precision Arithmetic Library}, which supplies arbitrary precision.
Initially we start with a precision of 2000 bits as significant precision of the mantissa and increase the precision on demand,
i.\,e., on every addition and subtraction we perform a check whether the current precision needs to be increased.
The constants of PSO are assigned to values that are commonly used, i.\,e., $\chi=0.72984$ and $c_1=c_2=1.496172$, as proposed in \cite{CK:02}.
The classical PSO is used as specified in Definition \ref{def:classicalPSO}.
Additionally, the PSO is visualized as pseudo code in Algorithm \ref{alg:classicalPSO}.
No bound handling procedure is used.
Particles move through the space without any borders.
We checked our results with some benchmarks of \cite{benchmarkset} and an additional function.
In detail, we investigate the sphere function,
the high conditioned elliptic Function
and Schwefel's problem (see \cite{benchmarkset} for detailed problem description).
Additionally, the diagonal function is added, where its second derivation matrix has a single heavy eigenvalue and the corresponding eigenvector is oriented diagonally.
\begin{defi}[Objective functions]\label{def:objectiveFunctions}
The formal definitions of the used objective functions are
\begin{itemize}
\item $f_{sph}(x):=\sum_{i=1}^D x_i^2$, the sphere function,
\item $f_{hce}(x):=\sum_{i=1}^D(10^6)^\frac{i-1}{D-1}x_i^2$, the high conditioned elliptic function,
\item $f_{sch}(x):=\sum_{i=1}^D(\sum_{j=1}^ix_j)^2$, the Schwefel's problem and
\item $f_{diag}(x):=\sum_{i=1}^D x_i^2+10^6\cdot\left(\sum_{i=1}^D x_i\right)^2$, the diagonal function.
\end{itemize}
\end{defi}
The optimal position for all of these functions is the origin and the optimal value is zero.
\subsection{Used Software}
For arbitrary precision the \emph{mpf\_t} data type of the \emph{GNU Multiple Precision Arithmetic Library} with version 4.3.2 is used.
The program, which simulates PSO is implemented in C++ and uses the \emph{mpf\_t} data type (program code is enclosed).
The program runs on \emph{openSUSE 12.3} and uses the \emph{GNU Compiler Collection} with version 4.3.3.
For creation of diagrams \emph{plot2svg} of J\"urg Schwizer within \emph{Matlab R2013b} and \emph{InkScape} 0.48 is used.

\subsection{Expectation of increments for various scenarios}
This section represents a motivation for stagnation phases with a specific number of stagnating dimensions.
Firstly, some typical runs of the PSO are presented,
which illustrate that for a combination of function and number of particles there is always a fixed number of dimensions, which are not stagnating.
Secondly, the expectation of the increments are visualized to support this impression.
\begin{figure}[htbp]
\centering
\subfloat[\label{figure:separationOfDimensions1A}Logarithm of potential $\log(\Phi)$]
{\includegraphics[width=0.865\textwidth]{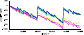}}

\subfloat[\label{figure:separationOfDimensions1B}Logarithmic potential $\Psi$]
{\includegraphics[width=0.865\textwidth]{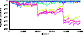}}

\subfloat[\label{figure:separationOfDimensions1C}Logarithm of distance to optimum]
{\includegraphics[width=0.865\textwidth]{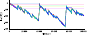}}
\caption{(a) Potential, (b) logarithmic potential and (c) distance to optimum with potential as specified in Definition \ref{def:experimentalPotential} on sphere function in 10 dimensions and with 3 particles; each line represents a single dimension, usual initialization}
\label{figure:separationOfDimensions1}
\end{figure}
\begin{figure}[htbp]
\centering
\subfloat[\label{figure:separationOfDimensions2A}Logarithm of potential $\log(\Phi)$]
{\includegraphics[width=0.865\textwidth]{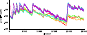}}

\subfloat[\label{figure:separationOfDimensions2B}Logarithmic potential $\Psi$]
{\includegraphics[width=0.865\textwidth]{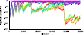}}

\subfloat[\label{figure:separationOfDimensions2C}Logarithm of distance to optimum]
{\includegraphics[width=0.865\textwidth]{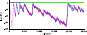}}
\caption{(a) Potential, (b) logarithmic potential and (c) distance to optimum with potential as specified in Definition \ref{def:experimentalPotential} on sphere function in 10 dimensions and with 3 particles; each line represents a single dimension, initialization with Algorithm \ref{alg:specialInit}}
\label{figure:separationOfDimensions2}
\end{figure}

Figure \ref{figure:separationOfDimensions1} shows a typical run of the PSO on the sphere function with three particles and 10 dimensions.
All positions are initialized uniformly at random in the $10$-dimensional cube $[-100.0,\,100.0]^{10}$.
Additionally, zero velocity initialization is used, which means that all entries of all velocity vectors are initialized with the value zero.
Other velocity initializations have also been investigated, but the results are the same.
The initialization of positions uniformly at random in a cube and with zero velocity will be referred to as \emph{usual initialization}, because this is a common initialization for PSO algorithms.
The Figure \ref{figure:separationOfDimensions1} contains the logarithm of the potential, the logarithmic potential and the logarithm of the distance to the origin, which is the only optimal point.
For some periods of time there are some dimensions which have low logarithmic potential and for the same dimensions the logarithm of the distance to the origin stays quite constant.
If one of those dimensions regains potential then all dimensions regain potential,
but dimensions which already have low logarithmic potential do increase much slower than the others.
Therefore their logarithmic potential decreases.
In Figure \ref{figure:separationOfDimensions1} approximately at time steps $40\,000$ and $75\,000$ a dimension,
which had lower logarithmic potential than other dimensions, regains potential.
At the same times the logarithmic potential of the remaining dimensions,
which already had low logarithmic potential,
significantly decreases.
While some dimensions increase their potential in dimensions,
which had large logarithmic potential previously,
the respective positions may become worse.
Finally, there are three dimensions which have a significantly smaller logarithmic potential than the other dimensions.
The other 7 dimensions are optimizing.
There are always short periods of time, where one of the 7 dimensions has a quite constant difference to the origin, as you can see in Figure \ref{figure:separationOfDimensions1}.
These periods can be that large that one of the three dimensions with low logarithmic potential regains potential,
but as the logarithmic potential of the lower three dimensions continuously decreases,
it becomes more and more unlikely that dimensions change from low logarithmic potential to high logarithmic potential.
An interesting fact is that the number of optimizing dimensions in this experiments stay the same,
even if the number of dimensions of the domain changes.
For the sphere function and three particles finally $\max(0,D-7)$ dimensions appear,
which have significantly smaller logarithmic potential than the other dimensions.

Figure \ref{figure:separationOfDimensions2} illustrates another run of the PSO on the sphere function with three particles and 10 dimensions.
For this run the positions are initialized as specified in Algorithm \ref{alg:specialInit}.
\begin{algorithm}
\caption{Special position initialization}
\label{alg:specialInit}

\SetKwInOut{Input}{input}
\SetKwInOut{InOut}{in/out}
\SetKwInOut{Output}{output}
\SetKwInOut{Data}{data}

\SetKw{KwTo}{to}%

\Input{number of particles $N$, number of dimensions $D$, scale $S$, number of dimensions with low logarithmic potential $L$, first dimension with low logarithmic potential $d^*$}
\Output{a vector of initial positions for each particle $X\in\left(\R^D\right)^N$}
\BlankLine
\tcc{rand($a,b$) supplies a uniform random value in $[a,b]$}
\For{$d \ASSIGN 1$ \KwTo $D$}
{
	\For{$n \ASSIGN 1$ \KwTo $N$}
	{
		$X[n][d] \ASSIGN$ rand(-100.0, 100.0)\;
	}
}
\BlankLine
\For{$d \ASSIGN d^*$ \KwTo $d^*+L-1$}
{
	$Y \ASSIGN$ rand(-100.0, 100.0)\;
	\For{$n \ASSIGN 1$ \KwTo $N$}
	{
		$X[n][d] \ASSIGN X[n][d] \cdot 2^{-S}+Y$\;
	}
}
\BlankLine
\Return $X$ \;
\end{algorithm}
This algorithm first initializes all $D$ dimensions as previously described uniformly at random in the interval $[-100.0, 100.0]$.
The dimensions $d^*$ to $d^*+L-1$ are initialized with some random center ($Y\in[-100,0, 100.0]$) and some random noise in the range $[-100.0\cdot 2^{-S}, 100.0\cdot 2^{-S}]$.
This initialization simulates stagnation of $L$ consecutive dimensions with initial logarithmic potential of approximately $-S$.
The chosen parameters for the run visualized in Figure \ref{figure:separationOfDimensions2} are $N=3$ particles,
$D=10$ dimensions, a scale of $S=200$, $L=9$ initial stagnating dimensions and $d^*=1$ the index of the first stagnating dimension.

Another option would have been scaling all position values of the first $L$ dimensions with $2^{-S}$, but this is problematic.
On the one hand the value of the initial potential then is mainly much less than $-S$,
because not only the velocities are reduced,
but also the derivative is smaller in the neighborhood of the optimum.
The reduction of the derivative cannot be controlled with some general formula,
because it depends on the objective function.
On the other hand the logarithmic potential used in this paper on functions which are not composite,
like $f_{sch}$ and $f_{diag}$,
acts not as proposed in the assumptions of Section \ref{sec:theory} for this initialization.
This is due to the fact that the derivative of those objective functions in stagnating dimensions depends also on the non-stagnating dimensions.
If it changes which part overweights the other then also the distribution of the increments changes significantly during a stagnation phase.
This change appears for the functions $f_{sch}$ and $f_{diag}$ when the positions of non-stagnating dimensions reach values in the interval $[-100.0\cdot 2^{-S}, 100.0\cdot 2^{-S}]$.
\begin{exam}
\label{exam:derivativeEvolution}
The function $f(x,y)=x^2+(x+y)^2$, the Schwefel's problem in two dimensions, is analyzed.
The first derivative equals $\left(4 x+2y, 2x+2y\right)$.
If $y= -2$, then the optimal value for $x$ is $1$,
because then the first entry of the first derivative is zero.
If $x= 100 + 1$,
then all entries of the first derivate are mainly determined by the value of $x$.
If the difference of $x$ to the optimal value is decreased by a factor of $100$ to $x= 1 + 1$,
then the first derivative is also decreased approximately by a factor of $100$.
If the difference of $x$ to the optimal value is again decreased by a factor of $100$ to $x= 0.01 + 1$,
then the first entry of the first derivative is also decrease by a factor of $100$,
but the second entry stays quite constant.
Therefore the second entry of the first derivative is now mainly determined by $y$.
The potential recognizes this change, too.
\end{exam}
If the positions of the stagnating dimension would be initialized small compared to the other dimensions,
then the position values of the stagnating dimensions stay constant while the position values of non-stagnating dimensions tend to their optimal value.
Initially the complete first derivative is determined by the dimensions with larger position values.
In the beginning the absolute values of all entries of the first derivative are decreasing, 
but finally the entries for stagnating dimensions of the first derivative stay constant.
Therefore the behavior changes significantly if this border is crossed.
For sure such situations may appear after initialization or during a PSO run,
but it is very unlikely that the particles are initialized in that way if standard initializations are used
and it is very unlikely that stagnating dimensions become much more optimized than non-stagnating dimensions during a PSO run.
The objective functions $f_{sph}$ and $f_{hce}$ are not facing this problem,
because the derivative of a single dimension depends only on the dimension itself.

For the PSO run in Figure \ref{figure:separationOfDimensions2} 9 of 10 dimensions are initialized as stagnating dimensions with logarithmic potential $-200$.
This figure illustrates that if less than seven dimensions are non-stagnating,
then the logarithmic potential of stagnating dimensions is increasing.
The smaller the number of non-stagnating dimensions is,
the larger is the increase of the logarithmic potential.
Finally, one of the stagnating dimensions becomes significant.
This results in a running phase in that dimension.
Running phases are introduced in \cite{SWa:13, SWb:13, SWc:15}.
If the swarm is running in a dimension $d$,
then the coordinate value of dimension $d$ determines the local attractor and the influence of the other dimensions can be neglected.
All velocities in direction $d$ are either all positive or all negative,
the local attractors are updated each step and the global attractor is updated at least once each iteration.
Dimension $d$ is heavily improved during this phase.
The running phase terminates as soon as walking in dimension $d$ does not lead to a further improvement of the function value.
During this running phase, the previously non-stagnating dimensions and the running dimension regain potential.
The other dimensions also experience an increase in potential,
but this increase is much smaller compared to the running dimension and the non-stagnating dimensions.
Therefore the logarithmic potential decreases.
This behavior repeats until seven dimensions are non-stagnating.
The largest running phase in Figure \ref{figure:separationOfDimensions2} occurs approximately at step $19\,000$. 

The PSO runs visualized in Figures \ref{figure:separationOfDimensions1} and \ref{figure:separationOfDimensions2} do not look quite linear as proposed.
\begin{figure}[htbp]
\centering
\subfloat[\label{figure:separationOfDimensions3A}Logarithm of potential $\log(\Phi)$]
{\includegraphics[width=0.865\textwidth]{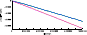}}

\subfloat[\label{figure:separationOfDimensions3B}Logarithmic potential $\Psi$]
{\includegraphics[width=0.865\textwidth]{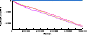}}

\subfloat[\label{figure:separationOfDimensions3C}Logarithm of distance to optimum]
{\includegraphics[width=0.865\textwidth]{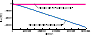}}
\caption{(a) Potential, (b) logarithmic potential and (c) distance to optimum with potential as specified in Definition \ref{def:experimentalPotential} on sphere function in 10 dimensions and with 3 particles; each line represents a single dimension, usual initialization}
\label{figure:separationOfDimensions3}
\end{figure}
This is due to the fact that they do not show continuous stagnation phases.
Figure \ref{figure:separationOfDimensions3} represents the extended run of Figure \ref{figure:separationOfDimensions1}, with hundred times more iterations.
The linear drift of the logarithmic potential can easily be recognized in this figure.
Furthermore, the non-stagnating dimensions can hardly be distinguished because their variance is not larger than in less number of iterations.
As the scale has increased, the variance looks smaller.

\begin{defi}
For an evaluated random variable of the $r$'th test run,
we write the original random variable in single tilted brackets with index $r$.
For example $\langle A\rangle_r$ is written for the evaluated random variable with name $A$ in test run $r$.
\end{defi}
\begin{expe}
\label{exp:expectationOfIncrements}
\label{EXP:EXPECTATIONOFINCREMENTS}
The first statistical analysis measures expected alteration of the logarithm of the potential and the logarithmic potential for various scenarios.
All functions are evaluated with $D=10$ dimensions and $100\,000$ iterations.
The stagnating dimensions, the number of particles $N$ and the objective functions vary.
Algorithm \ref{alg:specialInit} is used for position initialization and all velocities are initially set to zero.
The initial logarithmic potential is adjusted,
such that for no test run the logarithmic potential of any stagnating dimension reaches a value of $-100$ or more within the $100\,000$ iterations.
Each tested configuration is started $R=500$ times with different seeds.
Let $D_S$ be the set of stagnating dimensions, which is fixed for each specific configuration.
Let $T_{m}$ be $50\,000$ and let $T_{e}$ be $100\,000$, the number of half and full time steps of the evaluated process.
\end{expe}
The following estimaters are defined:
\begin{align*}
\overline{\mu_U}:=&\frac{1}{R(T_e-T_m)}\sum_{r=1}^{R}\left[{\log\left(\max_{d\not\in D_S}\left(\left\langle\Phi\left(\frac{T_e}{\triangle t}, d\right)\right\rangle_r\right)\right)
	-\log\left(\max_{d\not\in D_S}\left(\left\langle\Phi\left(\frac{T_m}{\triangle t}, d\right)\right\rangle_r\right)\right)}\right],\allowdisplaybreaks\\
\overline{\mu_M}:=&\frac{1}{R(T_e-T_m)}\sum_{r=1}^{R}\left[{\log\left(\min_{d\not\in D_S}\left(\left\langle\Phi\left(\frac{T_e}{\triangle t}, d\right)\right\rangle_r\right)\right)
	-\log\left(\min_{d\not\in D_S}\left(\left\langle\Phi\left(\frac{T_m}{\triangle t}, d\right)\right\rangle_r\right)\right)}\right],\allowdisplaybreaks\\
\overline{\mu_D}:=&\frac{1}{R\cdot\vert D_S\vert(T_e-T_m)}\sum_{r=1}^{R}\sum_{d\in D_S}\left[{\log\left(\left\langle\Phi\left(\frac{T_e}{\triangle t}, d\right)\right\rangle_r\right)
	-\log\left(\left\langle\Phi\left(\frac{T_m}{\triangle t}, d\right)\right\rangle_r\right)}\right],\allowdisplaybreaks\\
\overline{\mu_L}:=&
\overline{\mu}_D-\overline{\mu}_U.
\end{align*}
As linearity is available only in logarithmic scale, the logarithm of the potential $\Phi$ and the logarithmic potential $\Psi$ is used.
$\overline{\mu_U}$ represents the average decrease or increase of the logarithm of the potential of the most significant dimension per time step,
$\overline{\mu_M}$ represents the average decrease or increase of the logarithm of the potential of the least significant dimension among the initially non-stagnating dimensions per time step,
$\overline{\mu_D}$ represents the average decrease or increase of the logarithm of the potential of stagnating dimensions per time step and
$\overline{\mu_L}$ represents the average decrease or increase of the logarithmic potential of the stagnating dimensions per time step.
It is assumed that the first half of iterations is sufficient for adequate mixing.
Then the second half is used for average calculation.
Additionally the squared standard deviations are estimated for $\overline{\mu_U}$, $\overline{\mu_M}$, $\overline{\mu_D}$ and $\overline{\mu_L}$ by
\begin{align*}
\overline{\VAR_U}:=&\frac{1}{R(T_e-T_m)^2}\sum_{r=1}^{R}\left[{\log\left(\max_{d\not\in D_S}\left(\left\langle\Phi\left(\frac{T_e}{\triangle t}, d\right)\right\rangle_r\right)\right)
}\right.\\&\hspace{3.0cm}\left.{
	-\log\left(\max_{d\not\in D_S}\left(\left\langle\Phi\left(\frac{T_m}{\triangle t}, d\right)\right\rangle_r\right)\right)}-(T_e-T_m)\overline{\mu_U}\right]^2,\\
\overline{\VAR_M}:=&\frac{1}{R(T_e-T_m)^2}\sum_{r=1}^{R}\left[{\log\left(\min_{d\not\in D_S}\left(\left\langle\Phi\left(\frac{T_e}{\triangle t}, d\right)\right\rangle_r\right)\right)
}\right.\\&\hspace{3.0cm}\left.{
	-\log\left(\min_{d\not\in D_S}\left(\left\langle\Phi\left(\frac{T_m}{\triangle t}, d\right)\right\rangle_r\right)\right)}-(T_e-T_m) \overline{\mu_M}\right]^2,\\
\overline{\VAR_D}:=&\frac{1}{R\cdot\vert D_S\vert(T_e-T_m)^2}\sum_{r=1}^{R}\sum_{d\in D_S}\left[{\log\left(\left\langle\Phi\left(\frac{T_e}{\triangle t}, d\right)\right\rangle_r\right)
}\right.\\&\hspace{4.9cm}\left.{
	-\log\left(\left\langle\Phi\left(\frac{T_m}{\triangle t}, d\right)\right\rangle_r\right)}-(T_e-T_m)\overline{\mu_D}\right]^2,\\
\overline{\VAR_L}:=&\frac{1}{R\cdot\vert D_S\vert(T_e-T_m)^2}\sum_{r=1}^{R}\sum_{d\in D_S}\left[{\left\langle\Psi\left(\frac{T_e}{\triangle t}, d\right)\right\rangle_r
	-\left\langle\Psi\left(\frac{T_m}{\triangle t}, d\right)\right\rangle_r}-(T_e-T_m)\overline{\mu_L}\right]^2
\end{align*}
and
$\overline{\sigma_U}:=\sqrt{\overline{\VAR_U}}$, $\overline{\sigma_M}:=\sqrt{\overline{\VAR_M}}$,
$\overline{\sigma_D}:=\sqrt{\overline{\VAR_D}}$, $\overline{\sigma_L}:=\sqrt{\overline{\VAR_L}}$.
The variance will be analyzed in Section \ref{subsec:assumptionOfIndependence} in detail.
In Figure \ref{fig:expectationSphere3Particles} the measured decrease of the logarithm of the potential for non-stagnating dimension $\overline{\mu_U}$,
for stagnating dimensions $\overline{\mu_D}$, the measured change in logarithmic potential $\overline{\mu_L}$ and their measured standard deviations are visualized
for the sphere function $f_{sph}$, $N=3$ particles, $D=10$ dimensions and for different numbers of stagnating dimensions $L$.
\begin{figure}[t]
\center{\includegraphics[width=0.7\textwidth]{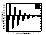}}
\caption{Measured expectations $\overline{\mu_U}$, $\overline{\mu_D}$ and $\overline{\mu_L}$ and their standard deviations for sphere function $f_{sph}$, $N=3$ particles, $D=10$ dimensions, $d^*=1$ and variable number of stagnating dimensions $L$}
\label{fig:expectationSphere3Particles}
\end{figure}
The decrease of the logarithm of the potential for non-stagnating dimensions $\overline{\mu_U}$ becomes smaller as the number of non-stagnating dimensions increases.
This is due to the fact that,
the more dimensions need to be optimized,
the more time is needed for optimization.
The second columns represent the decrease of the logarithm of the potential for stagnating dimensions $\overline{\mu_D}$, which also becomes smaller,
but not that fast as $\overline{\mu_U}$.
The third columns, which represent the change of logarithmic potential $\overline{\mu_L}$, is heavily positive for only one non-stagnating dimension,
but it decreases while the number of non-stagnating dimensions increases until finally the measured change becomes negative.
If less than 7 dimensions are non-stagnating and further stagnating dimensions are available,
then the logarithmic potential is likely to increase until at least one of the stagnating dimensions becomes non-stagnating,
because $\overline{\mu_L}$ is positive in those cases.

Figures \ref{fig:expectations2Particles} and \ref{fig:expectations3Particles} visualize the expectations of the change in logarithmic potential per time step for various scenarios.
\begin{figure}[htb]
\center{\includegraphics[width=0.7\textwidth]{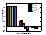}}
\caption{Measured expectations of logarithmic potential $\overline{\mu_L}$ and their standard deviations for various functions, $N=2$ particles, $D=10$ dimensions and variable number of stagnating dimensions $L$}
\label{fig:expectations2Particles}
\end{figure}
\begin{figure}[htb]
\center{\includegraphics[width=0.7\textwidth]{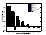}}
\caption{Measured expectations of logarithmic potential $\overline{\mu_L}$ and their standard deviations for various functions, $N=3$ particles, $D=10$ dimensions and variable number of stagnating dimensions $L$}
\label{fig:expectations3Particles}
\end{figure}
An obvious attribute of both graphs is that all functions perform similar if only one dimension is non-stagnating ($L=9$).
All functions mentioned in this paper have a constant second derivative,
which is a positive definite matrix,
and third and higher derivatives are zero.
If only a single dimension is not stagnating then the remaining function is represented by a parabola with some scale.
Remaining function is meant as the function,
which remains after all values of stagnating dimensions are replaced with the constants that represents their current position.
The scale of the parabola does not influence the behavior of the PSO,
because it only evaluates whether a position is better or not,
which is not dependent on the scale.
For sure the scale of a parabola has linear influence on $\Phi$,
but the logarithmic potential is not affected,
because this scale is compensated by the subtraction in Definition \ref{def:potential}.

Furthermore, the sphere function $f_{sph}$ and the high conditioned elliptic function $f_{hce}$ have quite similar measured expectations.
The high conditioned elliptic function actually is a sphere function scaled along the axes with different constant scales.
It seems that the PSO does not care if such scaling is applied to the sphere function.
\\
Other functions can have different behavior.
For the
Schwefel's problem it even depends on which dimensions are stagnating.
To explain this property, examining the matrix representation of the mentioned objective functions is helpful:
\begin{itemize}
\item $f_{sph}(x)=x^t\cdot A_{sph}\cdot x$, with $(A_{sph})_{i,j}=1$ if $i=j$ and $0$ else,
\item $f_{hce}(x)=x^t\cdot A_{hce}\cdot x$, with $(A_{hce})_{i,j}=(10^6)^\frac{i-1}{D-1}$ if $i=j$ and $0$ else,
\item $f_{sch}(x)=x^t\cdot A_{sch}\cdot x$, with $(A_{sch})_{i,j}=D+1-\min(i,j)$ and
\item $f_{diag}(x)=x^t\cdot A_{diag}\cdot x$, with $(A_{diag})_{i,j}=10^6+(A_{sph})_{i,j}$,
\end{itemize}
where $x^t$ represents the transposed vector of $x$.
Important are the directions of the eigenvectors of those matrices.
All eigenvectors have only real values, because the matrices are symmetric and positive semidefinite,
which always is a property of second derivative matrices of objective functions, which are two times continuously differentiable, at local minima.
For sphere function $f_{sph}$ and high conditioned elliptic function $f_{hce}$  there exists an eigenvector basis,
which is parallel to the coordinate axes.
This is the main reason for their similar behavior, because particles prefer to walk along the axes as proposed in \cite{SWb:13}.
In contrast, we have the diagonal function $f_{diag}$ with matrix $A_{diag}$,
which has a single large eigenvalue $(D\cdot 10^6+1)$ and the corresponding eigenvector is oriented diagonally to all axes,
i.\,e., the eigenvector is represented by $(1,1,\ldots,1)^t\in\R^D$.
The PSO has only bad performance on this function.
For $2$, $3$ and $4$ particles only two non-stagnating dimensions are necessary to generate considerable decreasing logarithmic potential,
while for the sphere function and the high conditioned elliptic function there are $3$ non-stagnating dimensions if $2$ particles are available and
there are $7$ non-stagnating dimensions if $3$ particles are available until the measured expectation of logarithmic potential becomes negative.
For $4$ or more particles operating on the sphere function or the high conditioned elliptic function it is even not known whether there exists a finite bound for the number of non-stagnating dimensions, 
such that the expectation of logarithmic potential becomes negative.
As the logarithmic potential of the diagonal function $f_{diag}$ is already negative for two non-stagnating dimensions,
only the values for $L=9$ and $L=8$ are visualized in Figures \ref{fig:expectations2Particles} and \ref{fig:expectations3Particles}.
Negative expectation of logarithmic potential means that stagnating dimension prefer to decrease their logarithmic potential and therefore are more likely to stay stagnating.
\begin{figure}[htb]
\center{\includegraphics[width=0.7\textwidth]{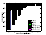}}
\caption{Measured expectations $\overline{\mu_U}$, $\overline{\mu_M}$ and $\overline{\mu_D}$ and their standard deviations for Schwefel's problem $f_{sch}$, $N=3$ particles, $D=10$ dimensions and variable number of stagnating dimensions $L$}
\label{fig:expectationsSchwefelsProblem}
\end{figure}

If some dimensions are temporarily stagnating then the functions appear like 
$$
(x-x^*)^t\cdot A\cdot(x-x^*)+const,
$$
where all rows and columns of stagnating dimensions are removed and $x^*$ is the current optimal point for the remaining dimensions.
For composite functions this optimal point stays at the origin.
For other functions this is not necessarily the case.
For the sphere function $f_{sph}$ and the diagonal function $f_{diag}$ the remaining submatrix does not depend on which dimensions are stagnating,
but only on the number of stagnating dimensions.
The submatrix of the high conditioned elliptic function $f_{hce}$ depends on which dimensions are stagnating,
but the remaining matrix is always a diagonal matrix and therefore the PSO always performs similar to the sphere function.
This implies that for this objective function it also only depends on how many dimensions are currently stagnating.

For the Schwefel's problem it heavily depends on which dimensions are stagnating.
E.\,g., two very different submatrices are the matrices where the first $L$ dimensions are stagnating and where the last $L$ dimensions are stagnating.
The effect on the potential is visualized in the Figures \ref{fig:expectations2Particles}, \ref{fig:expectations3Particles} and \ref{fig:expectationsSchwefelsProblem}.
If the first $L$ dimensions are removed, then the remaining matrix is equal to the respective matrix of the Schwefel's problem in $D-L$ dimensions.
If the last $L$ dimensions are removed, then the remaining matrix is equal to the respective matrix of the Schwefel's problem in $D-L$ dimensions plus the value $L$ for each entry for each entry of the remaining matrix.
This is similar to the matrix of the diagonal function, where a value of $10^6$ is added to each entry of the matrix of the sphere function.
The more dimensions at the end are removed, the more intense becomes a mostly diagonal eigenvector.
Therefore the PSO performs much worse on the Schwefel's problem if the last $L$ dimensions are stagnating than if the first $L$ dimensions are stagnating.
In Figures  \ref{fig:expectations2Particles} and \ref{fig:expectations3Particles} this can be observed,
because the change in logarithmic potential is much lower if the last $L$ dimensions are stagnating.

Figure \ref{fig:expectationsSchwefelsProblem} even shows that the final number of stagnating dimensions can differ if different dimensions are stagnating.
It shows the measured expectations of the change in potential of the largest potential value,
the lowest potential value of the initial non-stagnating dimensions and the average of the initial stagnating dimensions.
If the number of non-stagnating dimensions is small then $\overline{\mu_M}$ is similar to $\overline{\mu_U}$, which means that the $D-L$ initially non-stagnating dimensions stay non-stagnating.
Finally the value of $\overline{\mu_M}$ is similar to $\overline{\mu_D}$, which means that at least one of the initially non-stagnating dimensions becomes stagnating.
Actually this change should happen as soon as $\overline{\mu_U}$ is already larger than $\overline{\mu_D}$ for fewer non-stagnating dimensions,
because then there exist a dimension such that if this dimension becomes also stagnating the logarithmic potential of stagnating dimensions is more likely to decrease.
In Figure \ref{fig:expectationsSchwefelsProblem} $\overline{\mu_U}$ is larger than $\overline{\mu_D}$ for $d^*=1$ and $L=5$.
As expected $\overline{\mu_M}$ is similar to $\overline{\mu_D}$ for $d^*=1$ and $L<5$.
For $d^*=D-L+1$ this is not that clear.
$\overline{\mu_U}$ is slightly larger than $\overline{\mu_D}$ for $L=6$,
but as the resulting decrease of logarithmic potential is very small, it may happen through random fluctuation that all initially non-stagnating dimensions regain potential from time to time.
Therefore $\overline{\mu_M}$ only slightly differs from $\overline{\mu_U}$ for $L=5$.
If the number of iterations is heavily increased, then it will be more likely that one of the initial non-stagnating dimensions becomes stagnating.

Detailed results on the exact values of $\overline{\mu_U}$, $\overline{\mu_M}$, $\overline{\mu_D}$ and $\overline{\mu_L}$ of Experiment \ref{exp:expectationOfIncrements} are listed in Appendix \ref{appendix:dataOfExp:expectationOfIncrements}.
\subsection{Verification of the Assumption of Separation}
In this section evidence is provided for the separation of the increments $I_{t,d}$ into the two parts $B_{t,\Gamma_{B,T_0}}$ and $J_{t,d,\Gamma_{J,T_0}}$,
as proposed in Assumption \ref{ass:separation}.
The following experiment supplies information about the PSO during stagnation phases.
\begin{expe}
\label{exp:increments}
All functions are evaluated with $D=100$ dimensions and $200\,000$ iterations.
The number of particles $N$ and the objective functions vary.
Algorithm \ref{alg:specialInit} is used for position initialization and all velocities are initially set to zero.
$L$, the number of initially stagnating dimensions, and $d^*$, the index of the first stagnating dimension, are chosen,
such that the expectation of the increments of stagnating dimensions are just becoming negative.
$S$, the initial scale of the stagnating dimensions, is set to $500$. 
For no test run the logarithmic potential of any stagnating dimension reaches a value of $-100$ or more within the $200\,000$ iterations.
Each tested configuration is started $R=500$ times with different seeds.
Let $D_S$ be the set of stagnating dimensions, which is fixed for each specific configuration.
Let $T_{m}$ be $100\,000$ and let $T_{e}$ be $200\,000$, the number of half and full time steps of the evaluated process.
$\triangle t$ is set to one for this experiment.
\end{expe}
The number of dimensions is taken that large,
such that it can be assumed that $B_{t,\Gamma_{B,T_0}}$,
the base increment of all stagnating dimensions,
can be approximated by the average of all increments of stagnating dimensions.
The initialization with Algorithm \ref{alg:specialInit} and specific parameters $L$ and $d^*$ leads to a fixed set of stagnating dimensions.
As discussed earlier, the set of stagnating dimensions fixes the expectation of the increments,
but the set of stagnating dimensions also fixes the distributions of the increments,
because the remaining function is always the same and therefore the PSO performs similar.
Therefore it is assumed that $\Gamma_{B,T_0}$ and $\Gamma_{J,T_0}$ are constant distributions for a single scenario.
For reasons of readability, the constant distributions are removed from the index of the estimaters for $B$ and $J$.
The first half of the iterations is again used for sufficient mixing.
With this experiment large sums of increments are investigated.
Therefore $\triangle t$ is set to one for this experiment,
because larger values of $\triangle t$ are implicitly shown with larger numbers of accumulated increments.
Furthermore no scaling of the time axis is existent with this value of $\triangle t$.
The following estimaters for $B$ and $J$ in the $r$'th run and time step $t$ are used:
\begin{align*}
\overline{\langle B_{t}  \rangle_r}:=&\sum_{d\in D_S}\langle I_{t,d}\rangle_r/\vert D_S\vert\\
\overline{\langle J_{t,d}\rangle_r}:=&\langle I_{t,d}\rangle_r-\overline{\langle B_{t} \rangle_r}
\end{align*}
Additionally the following helping variables for sums of $\tau$ values are defined:
\begin{align*}
\langle I'_{\tau,d}\rangle_r:=&\sum_{t=T_m}^{T_m+\tau-1}\langle I_{t,d}\rangle_r\allowdisplaybreaks\\
\overline{\langle B'_{\tau}  \rangle_r}:=&\sum_{t=T_m}^{T_m+\tau-1}\overline{\langle B_{t}  \rangle_r}\allowdisplaybreaks\\
\overline{\langle J'_{\tau,d}\rangle_r}:=&\sum_{t=T_m}^{T_m+\tau-1}\overline{\langle J_{t,d}\rangle_r}
\end{align*}
The estimators for the expectation are calculated over the complete second half to receive a reliable value as in the previous experiment:
\begin{align*}
\overline{\mu_I}:=\overline{\mu_B}:&=\frac{1}{R\cdot(T_e-T_m)}\sum_{r=1}^R\sum_{t=T_m}^{T_e-1}\overline{\langle B_{t} \rangle_r}
\end{align*}
As proposed in Assumption \ref{ass:separation} the measured expectation of $J$ is zero:
\begin{align*}
\overline{\mu_J}:&=\frac{1}{R\cdot\vert D_S\vert\cdot(T_e-T_m)}\sum_{r=1}^R\sum_{d\in D_S}\sum_{t=T_m}^{T_e-1}\overline{\langle J_{t,d} \rangle_r}\allowdisplaybreaks\\
				 &=\frac{1}{R\cdot\vert D_S\vert\cdot(T_e-T_m)}\cdot
				 \sum_{r=1}^R\sum_{t=T_m}^{T_e-1}\sum_{d\in D_S}\left(\langle I_{t,d}\rangle_r-\sum_{\tilde d\in D_S}\langle I_{t,\tilde d}\rangle_r/\vert D_S\vert\right)=0
\end{align*}
A further estimater is introduced to measure the covariance ($\COV$) of the sums $\sum_{t=T_m}^{T_m+\tau-1}B_t$ and $\sum_{t=T_m}^{T_m+\tau-1}J_{t,d}$:
\begin{align*}
\overline{\COV_{B,J,\tau}}:&=\frac{1}{R\cdot\vert D_S\vert}\sum_{r=1}^R\sum_{d\in D_S}\left(\overline{\langle B'_{\tau}\rangle_r}-\tau\cdot\overline{\mu_B}\right)\cdot \overline{\langle J'_{\tau,d}\rangle_r}.
\end{align*}
If $B$ and $J$ are independent then their covariance needs to be zero.
Actually the estimater $\overline{\COV_{B,J,\tau}}$ for this covariance always evaluates to zero, because
\begin{eqnarray*}
\overline{\COV_{B,J,\tau}}
&:=&\frac{1}{R\cdot\vert D_S\vert}\sum_{r=1}^R\sum_{d\in D_S}\left[\left(\overline{\langle B'_{\tau}\rangle_r}-\tau\cdot\overline{\mu_B}\right)\cdot\overline{\langle J'_{\tau,d}\rangle_r}\right]\allowdisplaybreaks\\
&=&\frac{1}{R\cdot\vert D_S\vert}\sum_{r=1}^R\left[\left(\overline{\langle B'_{\tau}\rangle_r}-\tau\cdot\overline{\mu_B}\right)\cdot\right.
\sum_{d\in D_S}\left.\left(\langle I'_{\tau,d}\rangle_r-\overline{\langle B'_{\tau} \rangle_r}\right)\right]\allowdisplaybreaks\\
&=&\frac{1}{R\cdot\vert D_S\vert}\sum_{r=1}^R\left[\left(\overline{\langle B'_{\tau}\rangle_r}-\tau\cdot\overline{\mu_B}\right)\cdot\right.
\Big(\Big(\sum_{d\in D_S}\langle I'_{\tau,d}\rangle_r\Big)-\underbrace{\vert D_S\vert\cdot\overline{\langle B'_{\tau} \rangle_r}}_{=\sum_{d\in D_S}\langle I'_{\tau,d}\rangle_r}\Big)\Big]=0.
\end{eqnarray*}
Furthermore, this idea can be expanded.
We use
$$
\E\bigg[\frac{1}{\vert D_S\vert}\sum_{d\in D_S}I_{t,d}\biggm\vert \mathcal{A}'_t
\bigg]
\text{ with }
\mathcal{A}'_t:=\sigma\left(\mathcal{A}_t,\sigma\left(((s_t^n)^{d\not\in D_S} , (r_t^n)^{d\not\in D_S})_{n\in\lbrace 1,\ldots,N\rbrace}\right)\right)
$$
as approximation on the base increment $B_t$.
This conditional expectation represents the average increment of stagnating dimensions,
after the movement in non-stagnating dimensions is determined,
i.\,e., $(s_t^n)^{d\not\in D_S}$ and $(r_t^n)^{d\not\in D_S}$,
the random variables of the movement equations defined in Definition \ref{def:classicalPSO} for non-stagnating dimensions,
and all information about previous steps are known.
Even for the sum of $\tau$ increments
$$
\E\bigg[\frac{1}{\vert D_S\vert}\sum_{d\in D_S}\sum_{t=T_m}^{T_m+\tau-1}I_{t,d}\biggm\vert \mathcal{A}^\tau_{T_m}
\bigg]
\text{ with }
\mathcal{A}^\tau_{T_m}:=\sigma\Big(\mathcal{A}_{T_m},((s_t^n)^{d\not\in D_S} , (r_t^n)^{d\not\in D_S})_{\substack{n\in\lbrace 1,\ldots,N\rbrace,\\t\in\lbrace T_m,\ldots,T_m+\tau-1\rbrace}}\Big)
$$
is a possible approximation on $\sum_{t=T_m}^{T_m+\tau-1}B_t$ and
$$
\sum_{t=T_m}^{T_m+\tau-1}I_{t,d}
-\E\bigg[\frac{1}{\vert D_S\vert}\sum_{d\in D_S}\sum_{t=T_m}^{T_m+\tau-1}I_{t,d}\biggm\vert \mathcal{A}^\tau_{T_m}
\bigg]
$$
is a possible approximation on $\sum_{t=T_m}^{T_m+\tau-1}J_{t,d}$.
Other choices for $\mathcal{A}'_t$ and $\mathcal{A}^\tau_{T_m}$ are also possible, but the inclusions
$
\mathcal{A}_t\subseteq\mathcal{A}'_t\subseteq\mathcal{A}_{t+1}
$
and
$
\mathcal{A}_{T_m}\subseteq\mathcal{A}^\tau_{T_m}\subseteq\mathcal{A}_{T_m+\tau}
$
need to be satisfied to grant the specified measurability of Assumption \ref{ass:separation}.
With similar calculations as with the estimaters,
the actual expectation of $J_t$ and the covariance of random variables $B_t$ and $J_t$ is zero,
if the proposed approximations on random variables $B$ are accepted.
As presented in Example \ref{exam:derivativeEvolution}
the derivative in stagnating dimensions becomes quite constant
if the coordinate values of stagnating dimensions become constant.
For Figures \ref{figure:separationOfDimensions1}, \ref{figure:separationOfDimensions2} and \ref{figure:separationOfDimensions3} it is significant that even temporarily stagnating dimensions do not change their position values significantly while they are stagnating.
As the logarithm of the potential of all dimensions tends to minus infinity,
as presented in Figure~\ref{figure:separationOfDimensions3A},
the absolute values of the velocities in all dimensions need to tend to zero.
Therefore the potential in stagnating dimensions can be calculated by the multiplication of the absolute value of the velocity in that dimension and the approximately constant derivative in that dimension.
As the increments are calculated by the difference of logarithmic potential values,
only the portion of the velocity and the portion of the dimension with maximal potential remains.
Hence the increments of all stagnating dimensions are only determined by the velocities in that dimension.
The portion of the dimension with maximal potential is completely included in the base increment
and therefore the dimension dependent part $J$ of the increments is only determined by the change of the velocity in the specific dimension.
The update of global and local attractors appear simultaneously for each dimension.
Hence all dimension dependent parts $J$ have quite the same distribution
and are quite independent of each other,
because the velocity change is determined only by the updates of the attractors
and the two random variables of the movement equations,
which belong to that specific dimension.
That is the reason for averaging the expectation of random variables of type $J$ and the covariance of random variables $B$ and $J$ over all stagnating dimensions.
Therefore $\mu_J$, the actual expectation for a random variable $J$,
is
calculated with the previously described approximation by
the following equations:
\begin{align*}
\E&\Big[\sum_{t=T_m}^{T_m+\tau-1}J_{t,d}\Bigm\vert\mathcal{A}_{T_m}\Big]
\end{align*}
The variable $J_{t,d}$ is replaced by $I_{t,d} -B_t$ and then the approximation on $B_t$ is applied.
\begin{align*}
\approx&\frac{1}{\vert D_S\vert}\sum_{d\in D_S}\E\Big[\sum_{t=T_m}^{T_m+\tau-1}I_{t,d}
-\E\bigg[\frac{1}{\vert D_S\vert}\sum_{d'\in D_S}\sum_{t=T_m}^{T_m+\tau-1}I_{t,d'}\biggm\vert
\mathcal{A}^\tau_{T_m}\bigg]\biggm\vert \mathcal{A}_{T_m}\bigg]
\end{align*}
Additivity of the expected value is applied.
In the second part then only $\sigma$-algebra $\mathcal{A}_{T_m}$ remains,
because it is a sub-$\sigma$-algebra of $\mathcal{A}^\tau_{T_m}$,
i.\,e., $\mathcal{A}_{T_m}\subseteq\mathcal{A}^\tau_{T_m}$.
\begin{align*}
\overset{\mathcal{A}_{T_m}\subseteq\mathcal{A}^\tau_{T_m}}{=}&\frac{1}{\vert D_S\vert}\sum_{d\in D_S}\E\Big[\sum_{t=T_m}^{T_m+\tau-1}I_{t,d}\Bigm\vert \mathcal{A}_{T_m}\Big]
-\frac{\vert D_S\vert}{\vert D_S\vert}\E\bigg[\frac{1}{\vert D_S\vert}\sum_{d'\in D_S}\sum_{t=T_m}^{T_m+\tau-1}I_{t,d'}\biggm\vert\mathcal{A}_{T_m}\bigg]=0.
\end{align*}
Additionally $\COV_{B,J,\tau}$, the actual covariance of random variables $B$ and $J$, is calculated by
\begin{align*}
\COV_{B,J,\tau}
:=&\frac{1}{\vert D_S\vert}\sum_{d\in D_S}\E\Big[\Big(\sum_{t=T_m}^{T_m+\tau-1}B_t-\E\Big[\sum_{t=T_m}^{T_m+\tau-1}B_t\Bigm\vert \mathcal{A}_{T_m}\Big]\Big)\cdot
\Big(\sum_{t=T_m}^{T_m+\tau-1}J_{t,d}\Big)\Bigm\vert \mathcal{A}_{T_m}\Big]\approx
\end{align*}
All variables $J_{t,d}$ and $B_t$ are replaced by their approximation.
\begin{align*}
\approx&\frac{1}{\vert D_S\vert}\sum_{d\in D_S}
\E\bigg[\bigg(
	\E\bigg[\frac{1}{\vert D_S\vert}\sum_{d'\in D_S}\sum_{t=T_m}^{T_m+\tau-1}I_{t,d'}\biggm\vert \mathcal{A}^\tau_{T_m}\bigg]
	-\E\bigg[
		\E\bigg[\frac{1}{\vert D_S\vert}\sum_{d'\in D_S}\sum_{t=T_m}^{T_m+\tau-1}I_{t,d'}\biggm\vert \mathcal{A}^\tau_{T_m}\bigg]
		\biggm\vert \mathcal{A}_{T_m}\bigg]
	\bigg)
	\\&\hspace{2cm}
	\cdot\bigg(\sum_{t=T_m}^{T_m+\tau-1}I_{t,d}-\E\bigg[\frac{1}{\vert D_S\vert}\sum_{d'\in D_S}\sum_{t=T_m}^{T_m+\tau-1}I_{t,d'}\biggm\vert \mathcal{A}^\tau_{T_m}\bigg]\bigg)
	\biggm\vert \mathcal{A}_{T_m}\bigg]
\end{align*}
The two nested conditional expectations in the second line can be combined and only $\mathcal{A}_{T_m}$ remains,
because $\mathcal{A}^\tau_{T_m}$ is a sub-$\sigma$-algebra of $\mathcal{A}_{T_m}$.
Furthermore the outmost conditional expectation is replaced by two conditional expectations,
where the inner $\sigma$-algebra $\mathcal{A}^\tau_{T_m}$ is a sub-$\sigma$-algebra of $\mathcal{A}_{T_m}$.
\begin{align*}
=&\frac{1}{\vert D_S\vert}\sum_{d\in D_S}
\E\bigg[\E\bigg[\bigg(
	\E\bigg[\frac{1}{\vert D_S\vert}\sum_{d'\in D_S}\sum_{t=T_m}^{T_m+\tau-1}I_{t,d'}\biggm\vert \mathcal{A}^\tau_{T_m}\bigg]
	-\E\bigg[\frac{1}{\vert D_S\vert}\sum_{d'\in D_S}\sum_{t=T_m}^{T_m+\tau-1}I_{t,d'}\biggm\vert \mathcal{A}_{T_m}\bigg]
	\bigg)
	\\&\hspace{2.5cm}
	\cdot\bigg(\sum_{t=T_m}^{T_m+\tau-1}I_{t,d}-\E\bigg[\frac{1}{\vert D_S\vert}\sum_{d'\in D_S}\sum_{t=T_m}^{T_m+\tau-1}I_{t,d'}\biggm\vert \mathcal{A}^\tau_{T_m}\bigg]\bigg)
	\biggm\vert \mathcal{A}^\tau_{T_m}\bigg]\biggm\vert \mathcal{A}_{T_m}\bigg]
\end{align*}
The second outmost conditional expectation is applied.
The first factor is a constant under this $\sigma$-algebra.
Therefore nothing happens there.
In the second factor the conditional expectation only remains for the sum over $I_{t,d}$.
\begin{align*}
=&\frac{1}{\vert D_S\vert}\sum_{d\in D_S}
\E\bigg[\bigg(
	\E\bigg[\frac{1}{\vert D_S\vert}\sum_{d'\in D_S}\sum_{t=T_m}^{T_m+\tau-1}I_{t,d'}\biggm\vert \mathcal{A}^\tau_{T_m}\bigg]
	-\E\bigg[\frac{1}{\vert D_S\vert}\sum_{d'\in D_S}\sum_{t=T_m}^{T_m+\tau-1}I_{t,d'}\biggm\vert \mathcal{A}_{T_m}\bigg]
	\bigg)
	\\&\hspace{2.0cm}
	\cdot\bigg(\E\bigg[\sum_{t=T_m}^{T_m+\tau-1}I_{t,d}\biggm\vert \mathcal{A}^\tau_{T_m}\bigg]
	-\E\bigg[\frac{1}{\vert D_S\vert}\sum_{d'\in D_S}\sum_{t=T_m}^{T_m+\tau-1}I_{t,d'}\biggm\vert \mathcal{A}^\tau_{T_m}\bigg]\bigg)\biggm\vert \mathcal{A}_{T_m}\bigg]
\end{align*}
The sum $\frac{1}{\vert D_S\vert}\sum_{d\in D_S}$ in front of the complete expression is shifted into the conditional expectation.
The first factor does not depend on $d$.
Therefore it can be applied directly on the second factor.
On the one hand the sum is put inside the first conditional expectation of the second factor.
On the other hand the second conditional expectation does not depend on $d$ and therefore it is summed up $\vert D_S\vert$ times,
which means that this conditional expectation is effectively multiplied by one.
\begin{align*}
=&
\E\bigg[\bigg(
	\E\bigg[\frac{1}{\vert D_S\vert}\sum_{d'\in D_S}\sum_{t=T_m}^{T_m+\tau-1}I_{t,d'}\biggm\vert \mathcal{A}^\tau_{T_m}\bigg]
	-\E\bigg[\frac{1}{\vert D_S\vert}\sum_{d'\in D_S}\sum_{t=T_m}^{T_m+\tau-1}I_{t,d'}\biggm\vert \mathcal{A}_{T_m}\bigg]
	\bigg)
	\\&\hspace{0.5cm}
	\cdot\bigg(\underbrace{\E\bigg[\frac{1}{\vert D_S\vert}\sum_{d\in D_S}\sum_{t=T_m}^{T_m+\tau-1}I_{t,d}\biggm\vert \mathcal{A}^\tau_{T_m}\bigg]
	-\E\bigg[\frac{1}{\vert D_S\vert}\sum_{d'\in D_S}\sum_{t=T_m}^{T_m+\tau-1}I_{t,d'}\biggm\vert \mathcal{A}^\tau_{T_m}\bigg]}_{=0}\bigg)\biggm\vert \mathcal{A}_{T_m}\bigg]=0.
\end{align*}
Admittedly, this calculations can also be done for any other set of random variables,
which does not belong to a stagnation phase at all,
but then the described simplification cannot be justified.
As all stagnating dimensions act similar and can be permuted without any effect,
also the covariances of $J_{t,d}$ for different values of $d$ are zero.
Therefore we have that there definitely exist random variables $B$ and $J$,
which have no linear dependency.
Furthermore with the chosen approximation for variables base increment and dimension dependent increment
we can expect sufficient independence for specific times or time intervals.
\subsection{Experimental Verification of the Assumption of Independence}\label{subsec:assumptionOfIndependence}
The previous section presented reasons for the separation of the increments into independent base increments and dimension dependent increments.
This section experimentally shows that increments, which belong to different periods of time,
fulfill sufficient independence to accept the applicability of the theorems in Chapter \ref{sec:theory}.

First some additional estimaters for the variance and the sixth central moment related to Experiment \ref{exp:increments} are introduced.
\begin{align*}
\overline{\VAR_{I,\tau}}:&=\frac{1}{R\cdot\vert D_S\vert}\sum_{r=1}^R\sum_{d\in D_S}\left(\langle I'_{\tau,d}\rangle_r-\tau\cdot\overline{\mu_I}\right)^2\allowdisplaybreaks\\
\overline{\VAR_{B,\tau}}:&=\frac{1}{R}\sum_{r=1}^R\left(\overline{\langle B'_{\tau}\rangle_r}-\tau\cdot\overline{\mu_B}\right)^2\allowdisplaybreaks\\
\overline{\VAR_{J,\tau}}:&=\frac{1}{R\cdot\vert D_S\vert}\sum_{r=1}^R\sum_{d\in D_S}\overline{\langle J'_{\tau,d}\rangle_r}^2\allowdisplaybreaks\\
\overline{\Msix_{I,\tau}}:&=\frac{1}{R\cdot\vert D_S\vert}\sum_{r=1}^R\sum_{d\in D_S}\left(\langle I'_{\tau,d}\rangle_r-\tau\cdot\overline{\mu_I}\right)^6\allowdisplaybreaks\\
\overline{\Msix_{B,\tau}}:&=\frac{1}{R}\sum_{r=1}^R\left(\overline{\langle B'_{\tau}\rangle_r}-\tau\cdot\overline{\mu_B}\right)^6\allowdisplaybreaks\\
\overline{\Msix_{J,\tau}}:&=\frac{1}{R\cdot\vert D_S\vert}\sum_{r=1}^R\sum_{d\in D_S}\overline{\langle J'_{\tau,d}\rangle_r}^6
\end{align*}
Due to limited resources only the sphere function $f_{sph}$ with $N=2$ and $N=3$ particles,
the high conditioned elliptic function $f_{hce}$ with $N=2$ and $N=3$ particles
and the Schwefel's problem $f_{sch}$ with $N=3$ particles and two different sets of stagnating dimensions are investigated.
As presented in Figures \ref{fig:expectations2Particles} and \ref{fig:expectations3Particles},
the first negative value for the expectation of the increments is available with 
$L=D-3=97$ stagnating dimensions with two particles and sphere or high conditioned elliptic function, 
$L=D-7=93$ stagnating dimensions with three particles and sphere or high conditioned elliptic function and
$L=D-5=95$ stagnating dimensions and $d^*=1$ with three particles and Schwefel's problem.
If $d^*$ is not equal to one then the remaining function of the Schwefel's problem depends heavily on the number of dimensions.
For $D=10$ dimensions the suitable values are $d^*=5$ and $L=6$,
but for $D=100$ dimensions already $d^*=4$ and $L=97$ are the appropriate values.
The measured variances are visualized in Figure \ref{figure:estimaterVar}.
\begin{figure}[htbp]
\centering
\subfloat[\label{figure:estimaterVarI}Estimater for the variance of increments]
{\includegraphics[width=0.865\textwidth]{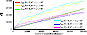}}

\subfloat[\label{figure:estimaterVarB}Estimater for the variance of base increments]
{\includegraphics[width=0.865\textwidth]{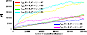}}

\subfloat[\label{figure:estimaterVarJ}Estimater for the variance of dimension dependent increments]
{\includegraphics[width=0.865\textwidth]{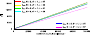}}
\caption{Estimaters for the variance of (a) increments $\overline{\VAR_{I,\tau}}$, (b) base increments $\overline{\VAR_{B,\tau}}$ and (c) dimension dependent increments $\overline{\VAR_{J,\tau}}$ resulting from Experiment \ref{exp:increments}; each line represents a single configuration}
\label{figure:estimaterVar}
\end{figure}
If the increments would be completely independent then the resulting graphs would be just lines.
As already discussed,
the distribution of the increments depends heavily on the number of non-stagnating dimensions.
In Figure \ref{figure:separationOfDimensions1C} at approximately steps $40\,000$ to $75\,000$ there are temporary four dimensions,
which do not improve their position,
but for the sphere function and three particles there generally are seven non-stagnating dimensions.
Therefore with $D=10$ dimensions only three stagnating dimensions normally appear,
because otherwise the expectation of the increments of stagnating dimensions is then positive.
The additional temporary stagnating dimension has a slightly lower logarithmic potential than the other non-stagnating dimensions,
which can be recognized in Figure \ref{figure:separationOfDimensions1B}.
Hence the logarithmic potential of the three obviously stagnating dimensions have a positive drift during this period of time with a temporary additional stagnating dimension.
The temporary stagnating dimension also receives that positive drift,
but as this is a random process,
the logarithmic potential of this dimension can remain small for some time.
At the end of periods of time with temporary additional stagnating dimensions the PSO becomes running in the temporary stagnating dimension and the other stagnating dimensions quickly lose logarithmic potential.
Such periods appear quite frequently during PSO runs and their lengths vary.
The fewer the difference between the expectation of the increments with current number of stagnating dimensions and with one dimension fewer is,
the longer those periods can become.
The heavier increase of the measured variances in Figure \ref{figure:estimaterVar} results from those periods,
because during or at the end of those periods the behavior differs.
Over larger periods this effect become less important.
Also Figure \ref{figure:estimaterVar} supports that impression.
After the heavy increase at the beginning,
the variances grow quite linearly.
For the sphere and high conditioned elliptic function with two particles,
there is not even a recognizable increase at the beginning,
because the difference of the expectation of the increments differ strongly for $L=D-3$ and $L=D-2$ dimensions (compare with Figure \ref{fig:expectations2Particles})
and therefore the lengths of periods with temporary additional stagnating dimensions are very small.
Altogether increments may be independent enough in time when large time scales are used.

Taking a deeper look into the proof of Lemma \ref{lem:clt} results in the awareness that the independence in $t$ is used three times.
It is used to guarantee that,
\begin{itemize}
\item $\E[(\sum_{\tilde t=0}^{t-1}(I_{\tilde t}-\mu^*))^6]\le C \cdot t^3$, where $C$ is some constant and $I$ is either the base increment plus drift or the dimension dependent increment plus drift,
\item there is a positive probability for finite times that the sum $\sum_{\tilde t=0}^{t-1}I_{\tilde t}$ stays negative and
\item the probability to receive a negative sum $\sum_{\tilde t=0}^{t-1}I_{\tilde t}$ is larger or equal to the case when we already know that the sum was negative for smaller values of $t$.
\end{itemize}
First the sixth central moment is investigated.
For this purpose the previously defined estimaters for the sixth central moment are visualized in Figure \ref{figure:estimaterM6}.
The estimaters are divided by $\tau^3$, because the relation to $t^3$ or $\tau^3$ respectively should be analyzed.
\begin{figure}[htbp]
\centering
\subfloat[\label{figure:estimaterM6I}Estimater for the sixth central moment of increments]
{\includegraphics[width=0.865\textwidth]{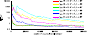}}

\subfloat[\label{figure:estimaterM6B}Estimater for the sixth central moment of base increments]
{\includegraphics[width=0.865\textwidth]{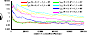}}

\subfloat[\label{figure:estimaterM6J}Estimater for the sixth central moment of dimension dependent increments]
{\includegraphics[width=0.865\textwidth]{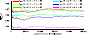}}
\caption{Estimaters for the sixth central moment of (a) increments $\overline{M^6_{I,\tau}}$, (b) base increments $\overline{M^6_{B,\tau}}$ and (c) dimension dependent increments $\overline{M^6_{J,\tau}}$ resulting from Experiment \ref{exp:increments} divided by $\tau^3$; each line represents a single configuration}
\label{figure:estimaterM6}
\end{figure}
With complete independence Figures \ref{figure:estimaterM6I} and \ref{figure:estimaterM6B} should appear similar to Figure \ref{figure:estimaterM6J},
which illustrates that for all scenarios the sixth central moment of the dimension dependent increment divided by $\tau^3$ oscillates around some constant value, as it is intended to.
But similar to the variance also the sixth central moment faces a heavy increase in the beginning,
which results from periods with temporary additional stagnating dimensions.
Nevertheless, also in Figures \ref{figure:estimaterM6I} and \ref{figure:estimaterM6B}
it seems that the central sixth moment can be bounded with some polynomial of degree three,
because the estimaters divided by $\tau^3$ are more or less decreasing or constant, but not increasing.
Note that all theorems in this paper could be adjusted if only constants $C>0$ and $\varepsilon>0$ exist,
such that $\E[(\sum_{\tilde t=0}^{t-1}(I_{\tilde t}-\mu^*))^6]\le C \cdot t^{4-\varepsilon}$ can be fulfilled.

The other two properties look quite reasonable.

For finite times the example runs show that there is a possibility that the increments and also the base increments and the dimension dependent increments sum up to negative values.
Animated by the experimental runs, a precise PSO run can be constructed such that the increments, the base increments and the dimension dependent increments sum up to negative values.
This precise run can be extended such that for each position some noise is possible,
i.\,e., it is permitted that each reached position is at most $\varepsilon$ apart from the intended position.
$\varepsilon$ can be chosen that small so that decisions whether a new position is a local or global optimum do not change.
The probability to receive one of those runs is then positive for any finite time.

Also the property,
that the probability to receive a negative sum is greater or equal,
if we already know that the sum was negative before, is reasonable for the PSO.
That is because mostly the previous sums are even much smaller than zero
and therefore the probability to stay negative,
when the next increment is added, is almost one,
while it is not that likely to be negative if the previous sum is already positive.

Therefore the assumption of independence was introduced, to generate provable theorems.
\subsection{Brownian Motion as Approximation for Accumulated Sums}
\label{sec:brownianMotion}
After Theorem \ref{sat:theorem1} it is proposed that Brownian Motions can be used as approximations on partial sums of increments.
In this section measured probabilities are compared with probabilities received from the approximation with Brownian Motion.

The probability that a stagnation phase does not end
is determined by the probability that sums of increments of initial stagnating dimension stay below some bound.
Actual maximal values cannot be measured,
because we cannot run experiments for an infinite number of iterations.
Therefore Experiment \ref{exp:increments} is used for this purpose.
For each initial stagnating dimension the maximal sum of the increments is calculated by
$$
I_{\rm max,d}:=\max_{T_m\le t \le T_e}\sum_{t'=T_m}^{t-1}I_{t',d}
$$
and $\langle I_{\rm max,d}\rangle_r$ is written for the evaluation on the $r$'th run.
For sure $100\,000$ iterations are far from infinite iterations,
but on the one hand an infinite number of iterations cannot be simulated
and on the other hand, as we have negative expectation of the increments,
the largest values of the sums most likely appear in the beginning of the process.
For the approximation with Brownian Motions the expectation and the variance of the increments are needed.
For this purpose the already defined estimaters $\overline{\mu_L}$ and $\overline{\VAR_{I,\tau}}$ are used.
As the variance does not grow linearly from the beginning, the following bounds for the variance are used:
$$
\overline{\VAR_{\rm max}}:=\max_{1\le\tau\le 1\,000}{\overline{\VAR_{I,100\cdot\tau}}}/{(100\cdot\tau)}
$$
$$
\overline{\VAR_{\rm min}}:=\left(\overline{\VAR_{I,100\,000}}-\overline{\VAR_{I,10\,000}}\right)/{90\,000}
$$
The upper bound $\overline{\VAR_{\rm max}}$ is the maximal measured variance per iteration.
The variances are calculated only each $100$'th step.
The lower bound $\overline{\VAR_{\rm min}}$ represents the approximately linear increase,
neglecting the heavy increase in the beginning.
After the proof of Theorem \ref{sat:theorem1}, it is proposed that
$1-\exp(2(c_s-c_0)\mu/\sigma^2)$
is an approximation on the probability that the maximal sum of a single dimension stays below $c_0-c_S$.
Accordingly, $1-\exp(-2\cdot x\cdot\mu/\sigma^2)$ can be assumed to be an approximation on the cumulative distribution function for $I_{\rm max,d}$.
In Figure \ref{figure:cumulativeDistributionFunctions} the cumulative distribution functions,
generated by the measured values $\overline{\mu_L}$, $\overline{\VAR_{\rm max}}$ and $\overline{\VAR_{\rm min}}$,
i.\,e., 
$$
F_{\VAR_{\rm max}}(x):=1-\exp(2\cdot x\cdot\overline{\mu_L}/\overline{\VAR_{\rm max}})
$$
and 
$$F_{\VAR_{\rm min}}(x):=1-\exp(2\cdot x\cdot\overline{\mu_L}/\overline{\VAR_{\rm min}}),
$$
and the empirical cumulative distribution function, i.\,e.,
$$
F_{\rm emp}(x):=\frac{1}{R\cdot\vert D_S\vert}\sum_{r=1}^R\sum_{d\in D_S}\1_{\langle I_{\rm max,d}\rangle_r\le x},
$$
are visualized for two different scenarios.
\begin{figure}[htbp]
\centering
\subfloat[\label{figure:cumulativeDistributionFunctionsA}Cumulative distribution function for the sphere function with two particles]
{\includegraphics[width=0.47\textwidth]{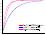}}
\hspace{0.04\textwidth}
\subfloat[\label{figure:cumulativeDistributionFunctionsB}Cumulative distribution function for the sphere function with three particles]
{\includegraphics[width=0.47\textwidth]{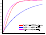}}
\caption{Cumulative distribution functions,
generated by the measured values $\overline{\mu_L}$, $\overline{\VAR_{\rm max}}$ and $\overline{\VAR_{\rm min}}$,
and the empirical cumulative distribution function for (a) the sphere function with two particles and (b) the sphere function with three particles.}
\label{figure:cumulativeDistributionFunctions}
\end{figure}
The two scenarios are visualized exemplarily,
because for other scenarios, the respective graphs look very similar.
The empirical cumulative distribution functions mainly proceed between the two cumulative distribution functions generated by the measured expectations and variances.
There is some positive probability, that $I_{\rm max,d}$ is $0$.
For the visualized scenarios this probability is quite small,
as it is hardly recognizable that the empirical cumulative distribution function starts with some positive value.
For small positive values it may appear that the empirical cumulative distribution functions are evaluated to lower values than the lower cumulative distribution function $F_{\VAR_{\rm max}}$.
Furthermore, the more the value of the cumulative distribution functions tend to $1$,
the more likely the empirical cumulative distribution functions are evaluated to larger values than the larger cumulative function $F_{\VAR_{\rm min}}$.
Both events occur in Figure \ref{figure:cumulativeDistributionFunctionsB}.
Nevertheless, Theorem \ref{sat:theorem1} only supplies positive values for the probability that $I_{\rm max,d}$ stays small enough so that stagnation phases do not end,
but this approximations even supply probabilities, which tend to $1$ if the value for the maximal sum is enlarged.
The allowed value for the maximal sum in the model used in this paper is $c_s-c_0$,
because sums after the beginning of a stagnation phase need to stay below this value.
The larger this difference is chosen,
the larger is the probability that a stagnation phase never ends,
at least for stagnation phases with one stagnating dimension.
As more stagnating dimensions are connected with the base increment,
it is more likely that the value for the maximal sum stays below some bound,
if it is already known that for other dimensions this maximal sum already stays below the same bound.
Therefore it is assumed that $(1-\exp(2(c_s-c_0)\mu/\sigma^2))^{N_0}$
is a lower bound for the probability that the maximal sum of any stagnating dimension stays below $c_s-c_0$,
which implies that a stagnation phase does not end.
The more stagnation phases occur during a single run,
the more likely some of the stagnating dimensions, which do not cause the end of stagnation phases,
have very small logarithmic potential.
Those dimensions mainly can be neglected for calculations on the probability whether a stagnation phase ends or not.
Therefore the more stagnation phases occur, the more likely those phases remain through infinity.
The probability increases to values of the case with a single stagnating dimension if $N_0-1$ dimensions already are that insignificant.
\subsection{Duration of Phases of Type $\mathit{PH}_X$}
In this section it is discussed,
why the requirement of Theorem \ref{sat:theorem2},
that phases of type $\mathit{PH}_X$ have finite expectation,
is acceptable.
Afterwards distributions of durations of phases of type $\mathit{PH}_X$ are visualized.

In general the event that no stagnation phase is active means that at least the largest $D-N_0$ dimensions have quite similar potential,
i.\,e., the absolute difference of the logarithm of their potentials is at most $\vert c_0\vert=-c_0$.
Through the movement equations
it is almost always possible,
that the absolute value of the velocity in most dimensions can be decreased significantly,
which leads to a significantly decreasing potential.
This is even possible for $D-1$ dimensions.
Therefore it is reasonable,
that there exist some positive probability $p$ and a finite number of iterations $T$,
such that the probability that the logarithmic potentials of at least $N_0$ dimensions decrease to values less than $c_0$ within $T$ iterations is at least $p$.
Then the waiting time for the event that the logarithmic potential of at least $N_0$ dimensions decrease to values less than $c_0$
is bounded above by a geometrical distribution,
which has the finite expectation $T/p$.
Therefore the expectation of the waiting time for the next stagnation phase is also bounded by $T/p$,
but this stagnation phase do not necessarily achieve the limits $\mu$, $M$ or $p_0$ specified in Definition \ref{def:stagnationPhaseClass}.
The choices of the values $N_0$, $\mu$, $M$, and $p_0$ are heavily dependent on the scenario.

For instance the sphere function $f_{sph}$ or the high conditioned elliptic function $f_{hce}$ with two particles can be analyzed.
As already discussed,
it only matters how many dimensions are stagnating.
Therefore all stagnation phases are good stagnation phases with equal parameters $\mu$, $M$, and $p_0$ if $N_0=D-3$ dimensions are stagnating.
If more dimensions than $N_0$ are stagnating,
then the expectation of the increments is positive,
which finally lead to fewer number of stagnating dimensions.
This leads to the appraisal that $T/p$ is also a bound for the expected waiting time to the next good stagnation phase for this scenario.

In contrast, the Schwefel's problem $f_{sch}$ is not that easy to handle,
because the behavior of the increments depends on the set of stagnating dimensions.
To guarantee suitable values for $N_0$, $\mu$, $M$ and $p_0$, most subsets of the set of dimensions need to be analyzed,
whether this set may appear as stagnating dimensions or not.
As there are $2^D$ many subsets,
this cannot be analyzed with acceptable time required for computing if $D$ becomes large.
Nevertheless $N_0$ can be chosen a little bit lower so that additional, not only temporary, stagnating dimensions are allowed.
Without knowing the exact values, $\mu$ and $M$ can be chosen as the maximal value over all good stagnation phases
and $p_0$ can be chosen as the minimal value over all good stagnation phases.
Also the experiments show that stagnation appears for Schwefel's problem as well.

For all analyzed scenarios the parameters $N_0$, $c_0$, $c_s$, $\mu$, $M$ and $p_0$ can be chosen,
such that all observed stagnation phases are good stagnation phases,
because the smaller $c_0$ is,
i.\,e., the larger the absolute value of $c_0$ is,
the more unlikely are stagnation phases with positive expectation of increments.
In particular,
this implies $\alpha_i=\tilde\alpha_i$ and $\beta_i=\tilde\beta_i$ for every $i$,
which is assumed from now on.

The following experiment supplies data for actual stagnation phases, because a usual initialization is used.
\begin{expe}
\label{exp:longTimeRun}
\label{EXP:LONGTIMERUN}
All functions are evaluated with $500\,000$ iterations.
The number of particles $N$, the number of dimensions $D$ and the objective functions vary.
Usual initialization is used,
i.\,e., Algorithm \ref{alg:specialInit} can be used for position initialization, but $L$, the number of initially stagnating dimensions, is set to zero,
and all velocities are initially set to zero.
Each tested configuration is started $R=500$ times with different seeds.
Let $T_{e}$ be $500\,000$, the number of time steps of the evaluated process.
$\triangle t$ is set to $100$.
\end{expe}
The following abbreviations are defined:
$$
D_{X,i}:=\lbrace r\in\lbrace 1,\ldots,R\rbrace:i=0\lor\langle\beta_{i-1}\rangle_r\le T_e\rbrace
$$
represents the set of runs such that the $i$'th phase of type $\mathit{PH}_X$ has started within the tested number of iterations.
$$
D_{Y,i}:=\lbrace r\in\lbrace 1,\ldots,R\rbrace:\langle\alpha_{i}\rangle_r\le T_e\rbrace
$$
represents the set of runs such that the $i$'th phase of type $\mathit{PH}_Y$ or $\mathit{PH}_F$ has started within the tested number of iterations.
For all tested scenarios $D_{X,i}$ equals $D_{Y,i}$,
which means that no test run of any tested scenario ended in a phase of type $\mathit{PH}_X$.
$\vert D_{X,0}\vert=R=500$,
because each test run starts with a phase of type $\mathit{PH}_X$.
Furthermore,
$$
F_{X,i}(x):=\sum_{r\in D_{X,i}}\frac{\1_{\langle X_i\rangle_r\le x}}{\vert D_{X,i}\vert}
$$
represents the empirical distribution of the phase lengths $X_i$.
\begin{figure}[htb]
\centering
{\includegraphics[width=0.6\textwidth]{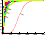}}
\caption{Empirical distribution of phase lengths $X_i$ with sphere function, $N=3$ particles, $D=100$ dimensions, $N_0=93$, $c_0=-40$ and $c_s=-20$}
\label{fig:empiricalPhaseLengths}
\end{figure}
Figure \ref{fig:empiricalPhaseLengths} shows this empirical distribution,
received from Experiment \ref{exp:longTimeRun},
for the durations of a phase of type $\mathit{PH}_X$
with the sphere function $f_{sph}$, $N=3$ particles, $D=100$ dimensions, $N_0=D-7=93$, $c_0=-40$ and $c_s=-20$.
This configuration is chosen exemplarily.
For all configurations the durations of a phase of type $\mathit{PH}_X$ are largest for the initial phase of type $\mathit{PH}_X$, which starts at iteration $0$.
At the beginning, all dimensions have quite the same potential.
Therefore the logarithmic potential of all dimensions are not that far from zero.
This implies that $N_0$ dimensions need to decrease their potential to $c_0=-40$ or lower.
After the end of a stagnation phase it is only known,
that from the set of initially stagnating dimensions
less than $N_0$ dimensions remain,
which have logarithmic potential less than $c_s=-20$ since the start of the stagnation phase.
Nevertheless there can be either other dimensions, which now have already low logarithmic potential,
i.\,e., less than $c_0$ or $c_s$.
Also there might be dimensions, which did not cause the end of the stagnation phase and still have low logarithmic potential.
Additionally, if the logarithmic potential of a dimension is larger than $c_s$, 
then that logarithmic potential is not almost zero.
Instead it most likely is just a little bit larger than $c_s$.
Altogether, after stagnation phases
there are still dimensions with much lower logarithmic potential
than they most likely have immediately after the initialization of the process.
Therefore the data is conform with this theoretical thoughts.
For this large number of stagnating dimensions $N_0=97$ it is also reasonable,
that the graphs of the empirical distributions move even further to the left,
because it becomes more and more likely that some of the stagnating dimensions are much smaller than $c_0$
and therefore do not delay the start of the next stagnation phase.
Further data on the existence of phases of type $\mathit{PH}_X$ and some statistical values can be found in Appendix \ref{appendix:dataOfExp:longTimeRun}.
\subsection{Swarm Converges}
In this section it is discussed why the swarm converges almost surely for the tested scenarios.

On the one hand the chosen parameters $\chi=0.72984$ and $c_1=c_2=1.496172$, as proposed in \cite{CK:02},
are known to be good, because rough theoretical analysis shows,
that convergence can be expected.

On the other hand the experiments supply also information, which lead to the assumption of convergence.
As assumed, the logarithmic potential of stagnating dimensions tends to minus infinity with linear drift.
Also the logarithm of the potential of all dimensions tend to minus infinity with linear drift for all recently mentioned scenarios.
The derivative for a stagnating dimension $d_1$ finally stays quite constant.
Therefore the potential in dimension $d_1$ is approximately the absolute value of the velocity in that dimension multiplied by that constant:
$$
\Phi(t,d_1)\approx \max_{n\in\lbrace 1,\ldots,N\rbrace}\vert V_t^{n,d_1}\vert\cdot const.
$$
The logarithm of the potential is then the logarithm of the absolute value of the velocity plus the logarithm of a constant value:
$$
\log(\Phi(t,d_1))\approx\log\big( \max_{n\in\lbrace 1,\ldots,N\rbrace}\vert V_t^{n,d_1}\vert\big)+\log(const)
$$
This implies that also the logarithm of the absolute values of the velocities of stagnating dimensions tends to minus infinity with linear drift.
For a non-stagnating dimension $d_2$ the first derivative is approximately zero,
but the second derivatives of all used functions are positive definite matrices.
Therefore it can be assumed that the potential of non-stagnating dimensions can be approximated by
the square of the velocity multiplied by a constant:
$$
\Phi(t,d_2)\approx \max_{n\in\lbrace 1,\ldots,N\rbrace}(V_t^{n,d_2})^2\cdot const.
$$
The logarithm of the potential is then two times the logarithm of the absolute value of the velocity plus the logarithm of a constant value:
$$
\log(\Phi(t,d_2))\approx 2\cdot\log\big( \max_{n\in\lbrace 1,\ldots,N\rbrace}\vert V_t^{n,d_2}\vert\big)+\log(const)
$$
This implies that also the logarithm of the absolute value of the velocities of non-stagnating dimensions tend to minus infinity with linear drift.

Therefore the absolute values of the velocities $V_t^{d,n}$ can be approximately bounded by $v_0\cdot c^t$ with positive constants $v_0\in\R$ and $c\in]0,1[$.
The maximal distance,
which can be covered by a particle beginning at some time $T$,
is then bounded by
$$
\sum_{t=T}^\infty v_0\cdot c^t=v_0\cdot c^T\sum_{t=0}^\infty c^t=\frac{v_0\cdot c^T}{1-c}.
$$
This bound is a finite value which tends to zero if $T$ tends to infinity.
Therefore the positions of the particles will tend to constant positions.
This position is indeed the limit of the global attractor,
because otherwise the velocities of a particle will stay approximately as large as the difference from current position to global attractor
and will not tend to zero.

Another point of view is that all investigated functions are strictly convex,
i.\,e.,
$$f(a\cdot x+(1-a)y)<a\cdot f(x)+(1-a)f(y)$$ for all $a\in]0,1[$ and $x,y\in\R^D,\,x\not=y$.
Usually the positions are initialized in a closed set.
Therefore the maximal function value for the initial global attractor $f(G_1^1)$ is bounded.
As all investigated functions can be written as $x^t\cdot A\cdot x$, with positive definite matrix $A$,
the set of points with lower function values than the worst possible initial global attractor is bounded too.
With the theorem of Bolzano-Weierstra{\ss} we receive that $(G_t^1)_{t\in\N_0}$ contains convergent subsequences.
If there is more than one accumulation point,
then the function value of all accumulation points needs to be equal,
because the functions are continuous
and therefore points with a larger function value than the lowest function value of any accumulation point cannot remain as accumulation points.
But if there are different accumulation points,
then the particles almost surely visit the area between the accumulation points.
As the investigated functions are strictly convex,
the points between two points with equal function value have strictly lower function values than the function values of the accumulation points.
After one of those points in between is visited,
the global attractor will change to that point and will never return to the surroundings of the intended accumulation points,
because the function values in this area are larger than the encountered value of the current point.
Therefore the sequence of the global attractor can have at most one accumulation point almost surely.
The theorem of Bolzano-Weierstra{\ss} guarantees,
that there is at least one accumulation point.
This implies that there is exactly one accumulation point $G_{lim}$ almost surely.
If the global attractor does not converge to that accumulation point,
then there exists an $\varepsilon>0$ such that the set $I:=\lbrace t\in \N_0: \Vert G_t^1-G_{lim}\Vert\rbrace$ does not have a finite number of elements.
With the theorem of Bolzano-Weierstra{\ss} also a subsequence of the sequence only containing indexes available in $I$ has at least one accumulation point.
This accumulation point cannot be equal to $G_{lim}$,
but this is a contradiction to the statement that there is only a single accumulation point.
Therefore the complete sequence of the global attractors converges almost surely to the limit point $G_{lim}$.
Indeed that limit point is not deterministic.
Also the local attractors need to converge to the same limit point almost surely,
because otherwise the same contradictions appear as with multiple accumulation points of the global attractor.
With usual parameter choices then also the positions will converge to the limit point of the global attractor and the velocities will tend to zero.

Nevertheless the property,
that a swarm will converge almost surely,
is dependent on the objective function, the number of particles, the number of dimensions and the parameters of the PSO
and needs to be checked if Theorem \ref{sat:theorem2} should be applied.
\subsection{Empirical Distributions of Stopping Times $\alpha_i$ and $\beta_i$ and of Logarithmic Potential}
This section provides additional data concerning Experiment \ref{exp:longTimeRun}.
First, empirical distributions of the stopping times $\alpha_i$ and $\beta_i$ are provided for two different scenarios.
Afterwards, empirical distributions of the logarithmic potential at the end of the test runs are visualized.

The empirical distributions of the stopping times $\alpha_i$ and $\beta_i$ are defined by the following expressions:
$$
F_{\alpha,i}(x):=\frac{1}{R}\sum_{r=1}^R\1_{\langle\alpha_i\rangle_r\le x}
\text{, }
F_{\beta,i}(x):=\frac{1}{R}\sum_{r=1}^R\1_{\langle\beta_i\rangle_r\le x}
$$
Figure \ref{figure:empiricalDistributionAlphaBeta1} shows the empirical distributions for the configuration with sphere function, $N=3$ particles, $D=100$ dimensions, $N_0=93$, $c_0=-40$ and $c_s=-20$ and
Figure \ref{figure:empiricalDistributionAlphaBeta2} shows the empirical distributions for the configuration with sphere function, $N=3$ particles, $D=8$ dimensions, $N_0=1$, $c_0=-40$ and $c_s=-20$.
\begin{figure}[htbp]
\centering
\subfloat[\label{figure:empiricalDistributionAlphaBeta1A}$F_{\alpha,i},\,F_{\beta,i}$ up to iteration $100\,000$]
{\includegraphics[width=0.47\textwidth]{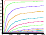}}
\hspace{0.04\textwidth}
\subfloat[\label{figure:empiricalDistributionAlphaBeta1B}$F_{\alpha,i},\,F_{\beta,i}$ up to iteration $500\,000$]
{\includegraphics[width=0.47\textwidth]{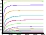}}
\caption{Empirical distribution of the stopping times $\alpha_i$ and $\beta_i$ ($F_{\alpha,i},\,F_{\beta,i}$) with sphere function, $N=3$ particles, $D=8$ dimensions, $N_0=1$, $c_0=-40$ and $c_s=-20$}
\label{figure:empiricalDistributionAlphaBeta1}
\end{figure}
\begin{figure}[htbp]
\centering
\subfloat[\label{figure:empiricalDistributionAlphaBeta2A}$F_{\alpha,i},\,F_{\beta,i}$ up to iteration $100\,000$]
{\includegraphics[width=0.47\textwidth]{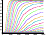}}
\hspace{0.04\textwidth}
\subfloat[\label{figure:empiricalDistributionAlphaBeta2B}$F_{\alpha,i},\,F_{\beta,i}$ up to iteration $500\,000$]
{\includegraphics[width=0.47\textwidth]{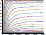}}
\caption{Empirical distribution of the stopping times $\alpha_i$ and $\beta_i$ ($F_{\alpha,i},\,F_{\beta,i}$) with sphere function, $N=3$ particles, $D=100$ dimensions, $N_0=93$, $c_0=-40$ and $c_s=-20$}
\label{figure:empiricalDistributionAlphaBeta2}
\end{figure}
The duration of the first phase of type $\mathit{PH}_X$ takes longer with more dimensions.
Also it is more unlikely that stagnation phases remain.
Especially in case of $D=100$ dimensions in the beginning most stagnation phases are quite short,
because in $93$ dimensions the logarithmic potential needs to stay small.
For $D=8$ only the logarithmic potential of a single dimension needs to stay small
and therefore the probability that a stagnation phase remains is much higher.
For $D=100$ it becomes also more and more likely that stagnation phases remain,
because there are more and more dimensions, which have a logarithmic potential which remains far below $c_0$.
This can also be observed in the Figures \ref{figure:empiricalDistributionAlphaBeta1} and \ref{figure:empiricalDistributionAlphaBeta2}.
The leftmost line represents the empirical distribution of $\alpha_0$, which equals the length of the first phase of type $\mathit{PH}_X$.
For $D=100$ the first empirical distributions of the stopping times $\alpha_i$ and $\beta_i$ reach the value $1$ very early.
For $D=8$ the empirical distribution of $\beta_0$,
which is represented by the second left line,
indicates the end of the first stagnation phase.
In contrast to the previous scenario,
this empirical distribution does not reach the value $1$.
That means that the first stagnation phase does not end within the tested iterations.
A further observation is that phases of type $\mathit{PH}_X$ have quite short duration compared with phases of type $\mathit{PH}_Y$.
Already in Figures \ref{figure:empiricalDistributionAlphaBeta1A} and \ref{figure:empiricalDistributionAlphaBeta2A} the empirical distributions of $F_{\beta,i}$ and $F_{\alpha,{i+1}}$ have only a small distance.
In Figures \ref{figure:empiricalDistributionAlphaBeta1B} and \ref{figure:empiricalDistributionAlphaBeta2B} they are merging to a single line already.

Another indication for the existence of unlimited stagnation phases is the empirical distribution of the final value of the logarithmic potential.
$$
F_{\Psi,d}(x):=\frac{1}{R}\sum_{r=1}^R\1_{(\sum_{d'=1}^D\1_{\langle\Psi(T_e,d')\rangle_r\le x})\ge d}
$$
To receive more meaningful figures,
not the empirical distributions for single dimensions are visualized,
but the empirical distributions for the $d$'th largest dimension is visualized.
$\sum_{d'=1}^D\1_{\langle\Psi(T_e,d')\rangle_r\le x}$ counts how many dimensions of run $r$ have a final logarithmic potential of less or equal than $x$.
A single run has then only an effect if at least $d$ dimensions with final logarithmic potential of less or equal than $x$ appear.
Figures \ref{Figure:empiricalDistributionPotential1} and \ref{Figure:empiricalDistributionPotential2} show the respective empirical distributions for two scenarios.
\begin{figure}[htb]
\centering
{\includegraphics[width=0.6\textwidth]{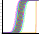}}
\caption{Empirical distribution of the final value of the logarithmic potential $F_{\Psi,d}$ with sphere function, $N=2$ particles and $D=100$ dimensions}
\label{Figure:empiricalDistributionPotential1}
\end{figure}

As already proposed,
the scenario with the sphere function and two particles,
which is displayed in Figure \ref{Figure:empiricalDistributionPotential1},
experiences always $D-3$ stagnating dimensions.
The three dimensions with largest logarithmic values at the end significantly differ from the other dimensions.
Their logarithmic potential is not that far from zero.
In contrast the logarithmic potential of the remaining $97$ dimensions is always smaller than $-500$ and in most cases much smaller than this value.
This small value means that the potential of the fourth largest dimension is $2^{-500}\approx 10^{-151}$ times smaller than the potential of the largest dimension.
The minimal logarithmic potential value of the third largest dimension is approximately $-90$ and the minimal logarithmic potential value of the second largest dimension is approximately $-30$.
Admittedly, these values are also very small compared to the largest dimension,
but, as already discussed, dimensions can become temporarily stagnating.
Therefore small logarithmic potential values can appear even for the dimension with second largest potential.
Nevertheless, the difference from the third to the fourth largest logarithmic potential value is huge.
In contrast the empirical distributions of the dimensions with lower logarithmic potential are quite similar.
They only move a little bit to the left.

\begin{figure}[htb]
\centering
{\includegraphics[width=0.6\textwidth]{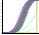}}
\caption{Empirical distribution of the final value of the logarithmic potential $F_{\Psi,d}$ with Schwefel's problem, $N=3$ particles and $D=100$ dimensions}
\label{Figure:empiricalDistributionPotential2}
\end{figure}
For the Schwefel's problem the number of actual stagnating dimensions depends on which dimensions are stagnating and which are not.
There are sets with three non-stagnating dimensions such that the expectation of the increments becomes negative,
but there are also sets with five non-stagnating dimension such that,
if any of the dimensions is removed from the set of non-stagnating dimensions,
then the expectation of the increments becomes positive.
Figure \ref{Figure:empiricalDistributionPotential2} represents the empirical distributions of the Schwefel's problem with three particles and $100$ dimensions.
There it can be observed that the empirical distribution of the fourth and fifth largest logarithmic potential value
do neither always belong to the set of stagnating dimensions nor they always belong to the set of non-stagnating dimensions.
Instead there is some positive probability that the dimension,
which refers to the fourth or the fifth largest logarithmic potential value,
belongs to the set of non-stagnating dimension.
This is illustrated by the heavy increase of the empirical distribution in the near surroundings of the potential value zero.
For the remaining cases the fourth or the fifth largest logarithmic potential value behave similar to the other stagnating dimensions.
It is also reasonable that the fourth or the fifth largest logarithmic potential value behave not identical,
because the distribution of the increments depends on the set of stagnating dimensions.
For sixth and further dimensions all stagnation phases are merged,
while the fifth largest potential only counts as stagnating dimension if four or fewer dimensions are non-stagnating.
The cases with exactly five non-stagnating dimensions is removed for this dimension.
\subsection{Long Time Experiments}
Further more time consuming experiments are done with larger number of particles.
Beware, if the number of steps is doubled, then the runtime sometimes increases by a factor of four,
because the longer the experiment takes the more increases the precision.
This is due to the event that the longer an experiment takes the more diverse become the values.
Also the number of dimensions influences the runtime more than linearly.
To grant no irregularities the precision needs to be increased,
which is executed automatically on demand by the used software.
As this experiments are analyzed for $100\,000\,000$ iterations only single runs are executed.

For example the configuration with four particles and the sphere function was under investigation.
With $D=40$ dimensions and $100\,000\,000$ iterations it appears that after a long period of time the smallest occurring logarithmic potential among the dimensions stays quite constant.
Indeed the dimension, which is responsible for the lowest value of the logarithmic potential changes over time.
The value for the lowest logarithmic potential $\Psi(t,d)$ stays approximately in the range $-1\,000$ to $-2\,000$.
Periods of time with temporary stagnating dimensions can even last up to $10\,000\,000$ iterations.
However the largest overall potential $\Phi(100\,000\,000,d)$ has decreased to approximately $2^{-8\,000}$.
It even appears that each dimension becomes significant from time to time,
which leads to the impression that four particles finally will succeed to optimize the sphere function with $40$ dimensions.
The absolute values of the global attractor for all dimensions at the end are in the range from $2^{-4\,200}$ to $2^{-3\,600}$,
which strengthens this impression.

With $D=400$ dimensions and $100\,000\,000$ iterations it cannot yet be determined, whether there are actual stagnating dimensions or not.
In fact the value for the lowest logarithmic potential $\Psi(t,d)$ reaches approximately $-10\,000$,
but no significant separation of logarithmic potential of different dimensions can be observed,
which would indicate stagnating and non-stagnating dimensions.
Some dimensions are stagnating almost since the beginning of that test run and have very small logarithmic potential,
but other dimensions regained logarithmic potential after very large number of iterations.
The maximal potential value of the last iterations has not decreased that extremely.
Potential values $\Phi(t,d)$ of up to approximately $2^{-30}$  are reached in the last $10\%$ of the measured iterations.
The absolute values of the global attractor for all dimensions at the end are in the range from $2^{-250}$ to $2^{7}$.
This means that the position of the global attractor in some dimensions at the end is still in the range of the initialized positions.
Surely this looks like that there are permanent stagnating dimensions,
but, as periods of time with temporary stagnating dimensions can take even more iterations than in the case with $40$ dimensions,
no prediction can be made,
whether there are stagnating dimensions in this case or not.
On the one hand a larger number of iterations needs to be tested
and on the other hand more than a single test run needs to be executed.
This is not quite easy because already this single test run took some months of computing time,
because the final needed precision is more than $10\,000$ bits,
which cause very time consuming calculations.

Therefore it is not quite clear, whether there is a limit of dimensions such that the PSO with chosen parameters and four particles finally stagnates at non-optimal points, or not.

\section{Conclusion}
Altogether a theoretical base was introduced to formalize stagnation during runs of the PSO.
For some functions evidence was provided, that the chosen theoretical model is applicable.
Especially for the sphere function and its scaled version, the high conditioned elliptic function,
it was shown that with two particles there are at most three dimensions, which do not stagnate permanently.
All other dimensions are finally stagnating.
For three particles the number of dimensions, which do not stagnate permanently,
is increased to seven.
For three particles, and PSO parameters as used in this paper, the sphere function and eight dimensions there is always a dimension,
which finally stops optimizing.
For other functions the number of stagnating dimensions can be significantly different from the configuration with the sphere function.
An example is the diagonal function.
Even with four particles there are only two dimensions, which do not stagnate permanently.

Furthermore, the presented framework is not only applicable to the presented version of the PSO.
Everything can work also with different parameters for the movement equations
or even with changed rules for movement or updates.
For example, the global attractor could be updated only each iteration and not after each particle.
Also other iteration based optimization algorithms can be analyzed if a suitable potential can be defined.

On the one hand it can be analyzed whether optimization algorithms can find local optima for specified functions,
or it may appear that there are always some dimensions, such that the coordinates in this dimension are not optimized.
On the other hand if convergence is present and if comparable definitions for a potential can be made,
for example Definition \ref{def:experimentalPotential} for PSO algorithms,
then the speed of convergence can be compared by comparison of the increments of logarithmic potentials.
Also the number of stagnating dimensions can be a comparison criterion.

In this paper only functions are analyzed, which do not change their shape, no matter how the coordinates are scaled.
For example if there are two vectors $x,y\in\R^D$ then $f(a\cdot x)\le f(a\cdot y)\Leftrightarrow a^2f(x)\le a^2f(y)\Leftrightarrow f(x)\le f(y)$ for any $a\in \R\setminus\lbrace 0\rbrace$.
The established framework can also handle more complicated functions,
for instance functions with bounded third derivatives.
For each local optimum those functions can be approximated by their Taylor approximation of degree two.
If the second derivative matrix is not only positive semidefinite but also positive definite,
then this approximation is sufficient to determine whether the PSO finally stops to tend to that local optimum or not.
If the PSO running on the approximation almost surely stagnates and does not end up in the single optimum,
then the PSO will stagnate if it tends to that local optimum in the original function,
because changes in higher derivatives become less important the nearer the particles are positioned around the local optimum.
In contrast if for all local optima stagnation is impossible almost surely for the Taylor approximation,
then it can also stagnate before, when higher derivatives are still significant.

Additionally if stagnation is not present then convergence to local optima is not granted.
At least not with the recent analysis.
Nevertheless it is expected that if stagnation is not available,
then each dimension reaches large logarithmic potential from time to time.
If dimensions have large logarithmic potential, then it is assumed that these dimensions optimize their positions.
In this paper stagnation is defined for some minimal number of stagnating dimensions $N_0$,
which is appropriate to grant not reaching a local optimum.
To grant convergence to local optima for at least $D-N_0$ dimensions,
stagnation phases need to be defined differently with a maximal (not minimal) number of stagnating dimensions.
This positive result can not be extended by only analyzing the derivatives of local optima, because earlier stagnation can be available.

Nevertheless if there is no positive value of $N_0$, the minimal number of stagnating dimensions,
such that stagnation is observed, then it is expected that convergence to local optima is present.

Final continuing questions in this context, which could be analyzed with the given framework, are the following:
Is there a number of particles, such that it can be granted that at least the sphere function,
continuous functions or differentiable functions will be optimized perfectly almost surely?
Is there a fixed number of particles such that it can be granted that at least the sphere function,
continuous functions or differentiable functions will be optimized perfectly almost surely for some bounded number of dimensions?
Is there a minimal or maximal number of stagnating dimensions for fixed number of particles, fixed PSO parameters and a set of objective functions?

\newpage
\bibliographystyle{abbrv}
\bibliography{literature}
\newpage
\begin{appendix}
\section{Data of Experiment \ref{exp:expectationOfIncrements}}
\label{appendix:dataOfExp:expectationOfIncrements}
\begin{table}[htb]
\center{
\caption{Values of estimaters specified in Experiment \ref{exp:expectationOfIncrements} for the sphere function $f_{sph}$}
{
$\begin{array}{| c | c | c | c | c | c | c | c | }
\hline
f & N & L & d^* & \overline{\mu}_U &  \overline{\mu}_M & \overline{\mu}_D & \overline{\mu}_L\\
\hline
\hline
f_{sph} & 2 & 9 & 1 & -0.280370 & -0.280370 & -0.067542 & 0.212828\\
f_{sph} & 2 & 8 & 1 & -0.133841 & -0.133843 & -0.0742897 & 0.0595512\\
f_{sph} & 2 & 7 & 1 & -0.0265214 & -0.0265178 & -0.0307508 & -0.0042293\\
\hline
f_{sph} & 3 & 9 & 1 & -0.281658   & -0.281658   & -0.0653937  &  0.216264  \\
f_{sph} & 3 & 8 & 1 & -0.232753   & -0.232748   & -0.0736407  &  0.159112  \\
f_{sph} & 3 & 7 & 1 & -0.175692   & -0.175695   & -0.0677883  &  0.107903  \\
f_{sph} & 3 & 6 & 1 & -0.0944085  & -0.094405   & -0.0458445  &  0.048564  \\
f_{sph} & 3 & 5 & 1 & -0.0436236  & -0.043621   & -0.026277   &  0.0173466 \\
f_{sph} & 3 & 4 & 1 & -0.0202586  & -0.020237   & -0.0155185  &  0.0047401 \\
f_{sph} & 3 & 3 & 1 & -0.0074307  & -0.0074679  & -0.0091868  & -0.0017561 \\
\hline
\end{array}$}}
\end{table}
\begin{table}[htb]
\center{
\caption{Values of estimaters specified in Experiment \ref{exp:expectationOfIncrements} for the high conditioned elliptic function $f_{hce}$}
{
$\begin{array}{| c | c | c | c | c | c | c | c | }
\hline
f & N & L & d^* & \overline{\mu}_U &  \overline{\mu}_M & \overline{\mu}_D & \overline{\mu}_L\\
\hline
\hline
f_{hce} & 2 & 9 & 1 & -0.280370 & -0.280370 & -0.067542  &  0.212828\\
f_{hce} & 2 & 8 & 1 & -0.133605 & -0.133599 & -0.0742256 &  0.0593796\\
f_{hce} & 2 & 7 & 1 & -0.026488 & -0.026485 & -0.0307987 & -0.0043110\\
\hline
f_{hce} & 3 & 9 & 1 & -0.281658  & -0.281658  & -0.0653937 &  0.216264\\
f_{hce} & 3 & 8 & 1 & -0.232727  & -0.232721  & -0.073618  &  0.159109\\
f_{hce} & 3 & 7 & 1 & -0.175414  & -0.175412  & -0.0676028 &  0.107811\\
f_{hce} & 3 & 6 & 1 & -0.0941275 & -0.0941295 & -0.0457537 &  0.0483738\\
f_{hce} & 3 & 5 & 1 & -0.0436067 & -0.0436075 & -0.0263199 &  0.0172867\\
f_{hce} & 3 & 4 & 1 & -0.0201859 & -0.0201904 & -0.0154480 &  0.0047379\\
f_{hce} & 3 & 3 & 1 & -0.0074169 & -0.0074541 & -0.0091259 & -0.0017089\\
\hline
\end{array}$}}
\end{table}
\begin{table}[htb]
\center{
\caption{Values of estimaters specified in Experiment \ref{exp:expectationOfIncrements} for the Schwefel's problem $f_{sch}$}
{
$\begin{array}{| c | c | c | c | c | c | c | c | }
\hline
f & N & L & d^* & \overline{\mu}_U &  \overline{\mu}_M & \overline{\mu}_D & \overline{\mu}_L\\
\hline
\hline
f_{sch} & 2 & 9 & 1 & -0.280217 & -0.280217 & -0.0675834 & 0.212634\\
f_{sch} & 2 & 9 & 2 & -0.280472 & -0.280472 & -0.0676385 & 0.212833\\
f_{sch} & 2 & 8 & 1 & -0.0920762 & -0.0920796 & -0.0622003 &  0.0298759\\
f_{sch} & 2 & 8 & 3 & -0.0258683 & -0.0258645 & -0.0449234 & -0.0190551\\
f_{sch} & 2 & 7 & 1 & -0.0052527 & -0.0052403 & -0.0215513 & -0.0162987\\
f_{sch} & 2 & 7 & 4 & -0.0275117 & -0.0446342 & -0.0450705 & -0.0175588\\
\hline
f_{sch} & 3 & 9 & 1 & -0.281479  & -0.281479 & -0.0654046 & 0.216074\\
f_{sch} & 3 & 8 & 1 & -0.220181  & -0.220177 & -0.0691141 & 0.151067\\
f_{sch} & 3 & 8 & 3 & -0.166589  & -0.166591 & -0.0553963 & 0.111193\\
f_{sch} & 3 & 7 & 1 & -0.120368  & -0.120371 & -0.0522450 & 0.0681226\\
f_{sch} & 3 & 7 & 4 & -0.0585785 & -0.058586 & -0.0338823 & 0.0246963\\
f_{sch} & 3 & 6 & 1 & -0.0363988 & -0.036386 & -0.0251476 & 0.0112512\\
f_{sch} & 3 & 6 & 5 & -0.0169291 & -0.0169333 & -0.0173492 & -0.0004201\\
f_{sch} & 3 & 5 & 1 & -0.0083174 & -0.0083236 & -0.0120229 & -0.0037055\\
f_{sch} & 3 & 5 & 6 & -0.0080515 & -0.0086733 & -0.0125910 & -0.0045394\\
f_{sch} & 3 & 4 & 1 & -0.0076700 & -0.0116463 & -0.0117998 & -0.0041299\\
f_{sch} & 3 & 4 & 7 & -0.0065029 & -0.0118378 & -0.0116454 & -0.0051425\\
f_{sch} & 3 & 3 & 1 & -0.0071813 & -0.0118641 & -0.0115253 & -0.0043441\\
f_{sch} & 3 & 3 & 8 & -0.0055640 & -0.0114905 & -0.0110172 & -0.0054532\\
\hline
\end{array}$}
}
\end{table}
\begin{table}[htb]
\center{
\caption{Values of estimaters specified in Experiment \ref{exp:expectationOfIncrements} for the diagonal function $f_{diag}$}
{
$\begin{array}{| c | c | c | c | c | c | c | c | }
\hline
f & N & L & d^* & \overline{\mu}_U &  \overline{\mu}_M & \overline{\mu}_D & \overline{\mu}_L\\
\hline
\hline
f_{diag} & 2 & 9 & 1 & -0.280374 & -0.280374 & -0.067633 & 0.212741\\
f_{diag} & 2 & 8 & 1 & -6.88\cdot 10^{-6} & -6.33\cdot 10^{-6} & -0.037121 & -0.037114\\
\hline
f_{diag} & 3 & 9 & 1 & -0.281716 & -0.281716 & -0.065535 & 0.216181\\
f_{diag} & 3 & 8 & 1 & -1.09\cdot 10^{-5} & -1.08\cdot 10^{-5} & -0.027709 & -0.027699\\
\hline
f_{diag} & 4 & 9 & 1 & -0.281934 & -0.281934 & -0.064976 & 0.216958\\
f_{diag} & 4 & 8 & 1 & -1.58\cdot 10^{-5} & -1.83\cdot 10^{-5} & -0.006380 & -0.006364\\
\hline
\end{array}$}}
\end{table}
\FloatBarrier
\section{Data of Experiment \ref{exp:longTimeRun}}
\label{appendix:dataOfExp:longTimeRun}
\begin{small}
The following variables are defined,
which indicate minimal, maximal and average values and variance observed with Experiment \ref{exp:longTimeRun}:\\
$\min\limits_{X,i}:=\inf_{r\in D_{X,i}}\langle X_i\rangle_r,$
$\max\limits_{X,i}:=\sup_{r\in D_{X,i}}\langle X_i\rangle_r,$\\
$\mu_{X,i}:=\sum_{r\in D_{X,i}} \frac{\langle X_i\rangle_r}{\vert D_{X_i}\vert},$
and $\VAR_{X,i}:=\sum_{r\in D_{X,i}}\frac{\left(\langle X_i\rangle_r- \mu_{X,i}\right)^2}{\vert D_{X,i}\vert}.$
\end{small}
\begin{table}[htb]
\caption{\small{Evaluation of data from Experiment \ref{exp:longTimeRun} with $c_0=-40$, $c_s=-20$ and function $f_{sph}$}}
\begin{center}{{
$
\begin{array}{| c c c | c c c c c | c | }
\hline
N & D & N_0 & \vert D_{X,0} \vert &\min\limits_{X,0} & \max\limits_{X,0} & \mu_{X,0} & \VAR_{X,0} & \vert D_{X,1} \vert\\
\hline\hline
2 &   4 &  1 & 500 &    400 &  9\,000 &  2\,017.6 &  1\,244\,090 & 351\\
2 &  10 &  7 & 500 & 1\,000 &  9\,900 &  3\,239.6 &  1\,766\,472 & 460\\
2 & 100 & 97 & 500 & 1\,500 & 13\,100 &  4\,840.2 &  2\,913\,124 & 493\\
\hline
3 &   8 &  1 & 500 &    700 &  7\,300 &  2\,597.6 &  1\,482\,594 & 488\\
3 &  20 & 13 & 500 & 2\,600 & 17\,900 &  6\,369.0 &  5\,246\,859 & 500\\
3 & 100 & 93 & 500 & 2\,800 & 22\,300 & 10\,319.4 &  9\,648\,924 & 500\\
\hline
\end{array}
$
}\\
{
$
\begin{array}{| c c c | c c c c c | c | }
\hline
N & D & N_0 & \vert D_{X,1} \vert & \min\limits_{X,1} & \max\limits_{X,1} & \mu_{X,1} & \VAR_{X,1} & \vert D_{X,2} \vert\\
\hline\hline
2 &   4 &  1 & 351 &      0 &  3\,600 &    862.7 &    518\,636 & 225\\
2 &  10 &  7 & 460 &      0 & 10\,500 & 1\,054.3 & 1\,030\,785 & 359\\
2 & 100 & 97 & 493 &      0 &  6\,900 & 1\,080.5 &    773\,576 & 457\\
\hline
3 &   8 &  1 & 488 &      0 &  6\,200 &    748.0 &    715\,611 & 430\\
3 &  20 & 13 & 500 &      0 &  9\,700 & 1\,554.8 & 2\,139\,917 & 500\\
3 & 100 & 93 & 500 &      0 & 12\,800 & 2\,087.6 & 4\,274\,366 & 500\\
\hline
\end{array}
$
}\\
{
$
\begin{array}{| c c c | c c c c c | c | }
\hline
N & D & N_0 & \vert D_{X,2} \vert & \min\limits_{X,2} & \max\limits_{X,2} & \mu_{X,2} & \VAR_{X,2} & \vert D_{X,3} \vert \\
\hline\hline
2 &   4 &  1 & 225 &      0 &  5\,100 &    831.1 &     752\,277 & 138\\
2 &  10 &  7 & 359 &      0 &  4\,100 &    763.0 &     424\,394 & 256\\
2 & 100 & 97 & 457 &      0 &  6\,000 &    903.5 &     669\,835 & 365\\
\hline
3 &   8 &  1 & 430 &      0 &  4\,200 &    604.7 &     488\,537 & 348\\
3 &  20 & 13 & 500 &      0 & 13\,800 & 1\,231.0 &  1\,885\,419 & 500\\
3 & 100 & 93 & 500 &      0 & 12\,600 & 1\,577.4 &  2\,411\,149 & 500\\
\hline
\end{array}
$
}}
\end{center}
\end{table}
\begin{table}[htb]
\caption{\small{Evaluation of data from Experiment \ref{exp:longTimeRun} with $c_0=-40$, $c_s=-20$ and function $f_{hce}$}}
\begin{center}{{
$
\begin{array}{| c c c | c c c c c | c | }
\hline
N & D & N_0 & \vert D_{X,0} \vert & \min\limits_{X,0} & \max\limits_{X,0} & \mu_{X,0} & \VAR_{X,0} & \vert D_{X,1} \vert \\
\hline\hline
2 &   4 &  1 & 500 &    400 &  6\,400 &  1\,948.8 &  1\,016\,179 & 368\\
2 &  10 &  7 & 500 & 1\,000 &  7\,800 &  3\,191.4 &  1\,570\,546 & 458\\
2 & 100 & 97 & 500 & 1\,700 &  9\,500 &  4\,481.4 &  2\,077\,034 & 487\\
\hline
3 &   8 &  1 & 500 &    400 &  7\,700 &  2\,520.6 &  1\,554\,436 & 483\\
3 &  20 & 13 & 500 & 1\,700 & 15\,100 &  6\,523.2 &  5\,161\,462 & 500\\
3 & 100 & 93 & 500 & 3\,400 & 20\,800 &  9\,861.0 &  8\,493\,979 & 500\\
\hline
\end{array}
$
}\\
{
$
\begin{array}{| c c c | c c c c c | c | }
\hline
N & D & N_0 & \vert D_{X,1} \vert & \min\limits_{X,1} & \max\limits_{X,1} & \mu_{X,1} & \VAR_{X,1} & \vert D_{X,2} \vert \\
\hline\hline
2 &   4 &  1 & 368 &      0 &  4\,900 &    892.1 &    696\,813 & 230\\
2 &  10 &  7 & 458 &      0 &  5\,900 & 1\,018.3 &    784\,991 & 371\\
2 & 100 & 97 & 487 &      0 &  7\,700 & 1\,092.6 &    932\,923 & 448\\
\hline
3 &   8 &  1 & 483 &      0 &  5\,400 &    767.9 &    665\,658 & 424\\
3 &  20 & 13 & 500 &      0 & 11\,900 & 1\,563.6 & 2\,322\,035 & 500\\
3 & 100 & 93 & 500 &      0 & 11\,600 & 1\,950.8 & 4\,057\,459 & 500\\
\hline
\end{array}
$
}\\
{
$
\begin{array}{| c c c | c c c c c | c | }
\hline
N & D & N_0 & \vert D_{X,2} \vert & \min\limits_{X,2} & \max\limits_{X,2} & \mu_{X,2} & \VAR_{X,2} & \vert D_{X,3} \vert \\
\hline\hline
2 &   4 &  1 & 230 &      0 &  5\,800 &    764.8 &     612\,455 & 132\\
2 &  10 &  7 & 371 &      0 &  5\,700 &    827.2 &     585\,970 & 272\\
2 & 100 & 97 & 448 &      0 &  5\,900 &    989.3 &     720\,108 & 360\\
\hline
3 &   8 &  1 & 424 &      0 &  5\,300 &    612.5 &     509\,301 & 368\\
3 &  20 & 13 & 500 &      0 & 10\,200 & 1\,159.6 &  1\,652\,248 & 500\\
3 & 100 & 93 & 500 &      0 & 13\,600 & 1\,478.8 &  2\,400\,471 & 500\\
\hline
\end{array}
$
}}
\end{center}
\end{table}
\begin{table}[htb]
\caption{\small{Evaluation of data from Experiment \ref{exp:longTimeRun} with $c_0=-40$, $c_s=-20$ and function $f_{sch}$}}
\begin{center}{{
$\begin{array}{| c c c | c c c c c | c | }
\hline
N & D & N_0 & \vert D_{X,0} \vert & \min\limits_{X,0} & \max\limits_{X,0} & \mu_{X,0} & \VAR_{X,0} & \vert D_{X,1} \vert \\
\hline\hline
3 &   8 &  3 & 500 & 1\,200 & 12\,900 &  4\,664.6 &  3\,911\,887 & 474\\
3 &  20 & 15 & 500 & 2\,300 & 16\,000 &  6\,931.2 &  5\,668\,627 & 492\\
3 & 100 & 95 & 500 & 2\,900 & 17\,500 &  8\,996.0 &  8\,277\,224 & 491\\
\hline
\end{array}
$}\\
{
$\begin{array}{| c c c | c c c c c | c | }
\hline
N & D & N_0 & \vert D_{X,1} \vert & \min\limits_{X,1} & \max\limits_{X,1} & \mu_{X,1} & \VAR_{X,1} & \vert D_{X,2} \vert \\
\hline\hline
3 &   8 &  3 & 474 &      0 & 12\,600 & 1\,171.7 & 2\,251\,437 & 371\\
3 &  20 & 15 & 492 &      0 & 15\,000 & 1\,270.7 & 2\,516\,623 & 452\\
3 & 100 & 95 & 491 &      0 & 14\,100 & 1\,365.4 & 2\,824\,870 & 446\\
\hline
\end{array}
$}\\
{
$
\begin{array}{| c c c | c c c c c | c | }
\hline
N & D & N_0 & \vert D_{X,2} \vert & \min\limits_{X,2} & \max\limits_{X,2} & \mu_{X,2} & \VAR_{X,2} & \vert D_{X,3} \vert \\
\hline\hline
3 &   8 &  3 & 371 &      0 &  5\,800 &    990.0 &  1\,037\,286 & 255\\
3 &  20 & 15 & 452 &      0 & 11\,300 &    990.7 &  1\,588\,719 & 367\\
3 & 100 & 95 & 446 &      0 &  6\,000 &    942.2 &  1\,108\,582 & 380\\
\hline
\end{array}
$}}
\end{center}
\end{table}
\begin{table}[htb]
\caption{\small{Evaluation of data from Experiment \ref{exp:longTimeRun} with $c_0=-40$, $c_s=-20$ and function $f_{dia}$}}
\begin{center}{{
$\begin{array}{| c c c | c c c c c | c | }
\hline
N & D & N_0 & \vert D_{X,0} \vert & \min\limits_{X,0} & \max\limits_{X,0} & \mu_{X,0} & \VAR_{X,0} & \vert D_{X,1} \vert \\
\hline\hline
3 &   3 &  1 & 500 &    700 & 13\,300 &  2\,557.0 &  1\,845\,011 &   1\\
3 &  10 &  8 & 500 & 2\,300 & 26\,200 &  5\,722.8 &  9\,167\,760 &   1\\
\hline
4 &   3 &  1 & 500 & 1\,500 & 28\,500 &  6\,533.4 & 13\,461\,824 &  46\\
4 &  10 &  8 & 500 & 6\,200 & 36\,900 & 14\,990.2 & 26\,613\,084 &  75\\
\hline
\end{array}
$}\\
{
$\begin{array}{| c c c | c c c c c | c | }
\hline
N & D & N_0 & \vert D_{X,1} \vert & \min\limits_{X,1} & \max\limits_{X,1} & \mu_{X,1} & \VAR_{X,1} & \vert D_{X,2} \vert \\
\hline\hline
3 &   3 &  1 &  1 & 1\,000 &  1\,000 & 1\,000.0 &           0 &   0\\
3 &  10 &  8 &  1 &    500 &     500 &    500.0 &           0 &   0\\
\hline
4 &   3 &  1 & 46 &    300 & 13\,500 & 3\,021.7 & 9\,384\,310 &   4\\
4 &  10 &  8 & 75 &    400 &  8\,000 & 2\,809.3 & 3\,977\,646 &  11\\
\hline
\end{array}
$}\\
{
$
\begin{array}{| c c c | c c c c c | c | }
\hline
N & D & N_0 & \vert D_{X,2} \vert & \min\limits_{X,2} & \max\limits_{X,2} & \mu_{X,2} & \VAR_{X,2} & \vert D_{X,3} \vert \\
\hline\hline
3 &   3 &  1 &    0 & \infty & -\infty &      0.0 &            0 &   0\\
3 &  10 &  8 &    0 & \infty & -\infty &      0.0 &            0 &   0\\
\hline
4 &   3 &  1 &    4 & 1\,400 & 16\,600 & 6\,825.0 & 33\,761\,875 &   0\\
4 &  10 &  8 &   11 &    900 & 13\,100 & 3\,800.0 & 11\,474\,545 &   0\\
\hline
\end{array}
$}}
\end{center}
\end{table}

\end{appendix}

\end{document}